\documentclass[10pt,journal,compsoc]{IEEEtran}
\usepackage{hyperref}       
\usepackage{url}            
\usepackage{booktabs}       
\usepackage{amsfonts}       
\usepackage{nicefrac}       
\usepackage{microtype}      
\usepackage{enumerate}
\usepackage{graphicx}
\usepackage{subfigure}
\usepackage{booktabs}
\usepackage{multirow}
\usepackage{amsmath}
\usepackage{amssymb}
\usepackage{graphicx}
\usepackage{subfigure}
\usepackage{wrapfig}
\usepackage{hyperref}
\usepackage{pifont}
\usepackage{color}
\usepackage[ruled]{algorithm2e}
\usepackage{algorithmic}
\usepackage{amsmath}
\usepackage{amsthm}
\usepackage[switch]{lineno}

\newtheorem{definition}{Definition}[section]

\newtheorem{theorem}{Theorem}[section]

\newtheorem{prop}[theorem]{Proposition}

\usepackage{setspace}
\newenvironment{sproof}
 {\vspace{-8pt}\begin{spacing}{0.3}\begin{proof}}
 {\end{proof}\end{spacing}\vspace{-3pt}}
\usepackage{ragged2e}
\usepackage{floatrow}
\newfloatcommand{capbtabbox}{table}[][\FBwidth]
\usepackage{tikz}
\newcommand*\circled[1]{\tikz[baseline=(char.base)]{
            \node[shape=circle,draw,inner sep=0.5pt] (char) {\small #1};}}
\newcommand{\term}[1]{\protect\circled{#1}}
\usepackage{array}
\newcolumntype{L}[1]{>{\raggedright\let\newline\\\arraybackslash\hspace{0pt}}m{#1}}
\usepackage{paralist}
\let\itemize\compactitem
\let\enditemize\endcompactitem
\let\enumerate\compactenum
\let\endenumerate\endcompactenum

\ifCLASSOPTIONcompsoc
  \usepackage[nocompress]{cite}
\else
  \usepackage{cite}
\fi
\ifCLASSINFOpdf
\else
\fi
\begin{document}
\title{Trusted Multi-View Classification  with \\ Dynamic Evidential Fusion}
\author{Zongbo~Han,
        Changqing~Zhang$^*$,
        Huazhu Fu, and
        Joey Tianyi Zhou
\IEEEcompsocitemizethanks{\IEEEcompsocthanksitem Z. Han and C. Zhang are with the College of Intelligence and Computing, Tianjin University, Tianjin 300072, China (e-mail: \{zongbo, zhangchangqing\}@tju.edu.cn). 
\IEEEcompsocthanksitem H. Fu is with the Institute of High Performance Computing (IHPC), Agency for Science, Technology and Research (A*STAR), Singapore 138632. (E-mail: hzfu@ieee.org)
\IEEEcompsocthanksitem J. T. Zhou is with the A*STAR Centre for Frontier AI Research (CFAR), Singapore 138632 (e-mail: zhouty@ihpc.a-star.edu.sg).
\IEEEcompsocthanksitem Corresponding author: Changqing Zhang.
}
}

\markboth{IEEE Transactions on Pattern Analysis and Machine Intelligence}%
{Shell \MakeLowercase{\textit{et al.}}: Bare Demo of IEEEtran.cls for Computer Society Journals}

\IEEEtitleabstractindextext{%
\begin{abstract}
\justifying{Existing multi-view classification algorithms focus on promoting accuracy by exploiting different views, typically integrating them into common representations for follow-up tasks. Although effective, it is also crucial to ensure the reliability of both the multi-view integration and the final decision, especially for noisy, corrupted and out-of-distribution data. Dynamically assessing the trustworthiness of each view for different samples could provide reliable integration. This can be achieved through uncertainty estimation. With this in mind, we propose a novel multi-view classification algorithm, termed  trusted multi-view classification (TMC), providing a new paradigm for multi-view learning by dynamically integrating different views at an evidence level. The proposed TMC can promote classification reliability by considering evidence from each view. Specifically, we introduce the variational Dirichlet to characterize the distribution of the class probabilities, parameterized with evidence from different views and integrated with the Dempster-Shafer theory. The unified learning framework induces accurate uncertainty and accordingly endows the model with both reliability and robustness against possible noise or corruption. Both theoretical and experimental results validate the effectiveness of the proposed model in accuracy, robustness and trustworthiness.}
\end{abstract}
\begin{IEEEkeywords}
Multi-view learning, Evidential deep learning, Varitional Dirichlet.
\end{IEEEkeywords}}
\maketitle
\IEEEdisplaynontitleabstractindextext
\IEEEpeerreviewmaketitle
\IEEEraisesectionheading{\section{Introduction}\label{sec:introduction}}
\IEEEPARstart{M}{}ulti-view data is common in real-world scenarios, and state-of-the-art multi-view learning methods have achieved significant success across a wide spectrum of applications. These methods are typically built on complex models \cite{liu2021self,lin2021completer,zhang2019cpm,xu2014large,liu2021one,peng2019comic,li2017discriminative} (usually integrating different views with deep neural networks (DNNs) ) and rely on large datasets with semantic annotations. Although existing multi-view learning algorithms provide promising classification performance, they tend to underestimate the uncertainty, and are thus vulnerable to yielding unreliable predictions, particularly when presented with noisy or corrupted views (\emph{e.g.}, information from abnormal sensors). Consequently, their deployment in safety-critical applications (\emph{e.g.}, machine-learning-aided medical 
diagnosis~\cite{esteva2017dermatologist} or autonomous driving~\cite{bojarski2016end}) is limited. This has inspired us to propose a new paradigm for multi-view classification with trusted decisions.

For multi-view learning, traditional methods usually assume an equal value for different views \cite{foresti2002distributed,simonyan2014two, shutova2016black, tsai2019multimodal} or learn a fixed weight factor for each \cite{natarajan2012multimodal,perez2019mfas,yan2004learning, atrey2010multimodal,poria2015deep}. In either case, the underlying assumption is that the qualities or importance of these views are essentially stable for all samples. However, in practice, the quality of a view may vary for different samples. Thus, it is crucial for designed models to be aware of this for adaption. For example, in machine-learning-aided multimodal medical diagnosis \cite{perrin2009multimodal,sui2018multimodal}, the algorithm may need to output a confident diagnosis based on a magnetic resonance (MR) image for one subject, while for another subject, a positron emission tomography (PET) image may be sufficient. Meanwhile, in image-text data classification, the text modality may be informative for one sample but noisy for another. Perceptual decisions are basically based on multiple sensory inputs whose reliabilities also vary over time ~\cite{hou2019neural}. Therefore, classification should be adaptive to dynamically changing multi-view correlations, and the decisions should be well explained according to multi-view inputs. Typically, we not only need to know the classification result, but also be able to answer ``How confident is the classification?" and ``Why is the confidence so high/low ?". As such, it is sensible to simultaneously model uncertainty for each sample which reflects the confidence of classification and uncertainty for each view of each sample which reflects the confidence from individual view.

Quantifying the prediction uncertainty is practical for enhancing the reliability of complex machine learning~\cite{charpentier2020posterior}. Uncertainty-based methods can be roughly divided into two main categories, \emph{i.e.}, Bayesian and non-Bayesian. Bayesian approaches characterize uncertainty by learning a distribution over the weights \cite{mackay1992bayesian, bernardo2009bayesian,neal2012bayesian}. Numerous Bayesian methods have been proposed, including Laplacian approximation \cite{mackay1992practical}, Markov Chain Monte Carlo (MCMC) \cite{neal2012bayesian} and variational techniques \cite{graves2011practical, ranganath2014black, blundell2015weight}. Unfortunately, compared with standard neural networks, due to the challenges of modeling the distribution of weights and convergence, Bayesian methods are computationally expensive. \cite{gal2016dropout} estimates the uncertainty by introducing dropout~\cite{srivastava2014dropout} as a Bayesian approximation in the testing phase, thereby reducing the computational cost. To avoid explicitly learning the distribution of parameters, recently a variety of non-Bayesian methods have been developed, including deep ensemble \cite{lakshminarayanan2017simple}, evidential deep learning \cite{sensoy2018evidential} and deterministic uncertainty estimation \cite{van2020uncertainty}. All of these methods are designed for estimating the uncertainty on single-view data, despite the fact that fusing multiple views guided by uncertainty could potentially improve both classification performance and reliability.

To promote trusted learning, in this work, we propose a novel multi-view classification method to elegantly integrate multi-view information for trusted decision making (shown in Fig.~\ref{fig:framework1}). Our model exploits different views at an evidence level instead of feature or output level as done previously, which produces a stable and reasonable uncertainty estimation and thus promotes both classification reliability and robustness. The Variational Dirichlet distribution is employed to model the distribution of the class probabilities, parameterized with evidence from different views and then the evidences are integrated with the interpretable Dempster-Shafer theory.  Compared with the conference version \cite{han2021trusted}, we significantly improve our work in the following aspects: (1) theoretical analysis to clarify the advantages of our model (e.g., proposition 3.1-3.4 with their proof in Sec. 3.3); (2) a more principled (variational) perspective to explain the desired Dirichlet distribution; (3) model improvement to enhance the interaction between different views (e.g., pseudo-view enhanced TMC in Sec. 3.4); (4) extensive experiments and discussions to further validate our model in potential applications (e.g., experiments for RGB-D scene recognition in Sec. 4.2). Overall, the proposed methodology supported by theoretical guarantees could provide trusted multi-view classification which is validated by sufficient empirical results.  In summary, the contributions of this paper are:
\begin{itemize}
\item [(1)] We first propose the  trusted multi-view classification (TMC) model, providing a new multi-view classification paradigm with reliable integration and decision explainability (benefiting from the uncertainty of each view).
\item [(2)] The proposed model is an effective and efficient framework (without any additional computations or neural network changes) for sample-adaptive multi-view integration. Further it integrates multi-view information at an evidence level with the Dempster-Shafer theory in an optimizable (learnable) way.
\item [(3)] Based on the subjective uncertainty explicitly learned for each view, the integration is conducted in a promising and theoretical guaranteed manner. This enables our model to improve both classification accuracy and trustworthiness.
\item [(4)] We conduct extensive experiments on datasets with multiple views, which clearly validate the superiority of our algorithm in accuracy, robustness, and reliability. This can be attributed to the benefits of our uncertainty estimation and multi-view integration strategy. The code is publicly available \footnote{\url{https://github.com/hanmenghan/TMC}}.
\end{itemize}

\section{Related Work}
\textbf{Multi-view Learning.}  Learning on data with multiple views has proven effective in a variety of tasks. CCA-based multi-view models \cite{hotelling1992relations,akaho2006kernel,wang2007variational,andrew2013deep,wang2015deep,wang2016deep} are representative ones that have been widely used in multi-view representation learning. These models essentially seek a common representation by maximizing the correlation between different views. Considering common and exclusive information, hierarchical multi-modal metric learning (HM3L) \cite{zhang2017hierarchical} explicitly learns shared multi-view and view-specific metrics, while $AE^2$-Nets \cite{zhang2019cpm} implicitly learn a complete (view-specific and shared multi-view) representation for classification. Recently, the methods \cite{tian2019contrastive,bachman2019learning,chen2020simple,hassani2020contrastive} based on contrastive learning have also achieved good performance. Due to its effectiveness, multi-view learning has been widely used in various applications \cite{kiela2018efficient,bian2017revisiting,kiela2019supervised,wang2020makes}. 

\textcolor{black}{\textbf{Late fusion models in multi-view, multi-modal and multi-kernel learning.} Late fusion is a widely used fusion strategy, which can be categorized as follows: naive fusion strategy, learnable classifier fusion, and confidence-based fusion. The naive fusion strategy is the most widely used method, which fuses decision results from different sources through decision averaging \cite{wang2019multi,shutova2016black,liu2018late}, decision voting \cite{liong2020amvnet,morvant2014majority}, or weighting multiple decisions \cite{potamianos2003recent, evangelopoulos2013multimodal, wang2021late,zhang2021late}. Besides, there are studies that integrate decisions from different sources through learnable models \cite{wei2019surface,wang2021mogonet,wang2019generative,ding2021cooperative}. Recently, confidence-based fusion has attracted attention \cite{subedar2019uncertainty,tian2020uno}. The significant difference between our method and them is that ours is an end-to-end fusion framework with theoretical guarantees.}

\textbf{Uncertainty-Based Learning.} DNNs have achieved great success in various tasks. However, since most deep models are essentially deterministic functions, their uncertainty cannot be obtained. 
Bayesian neural networks (BNNs) \cite{denker1991transforming,
mackay1992practical,neal2012bayesian} endow deep models with uncertainty by replacing the deterministic weight parameters with distributions. To avoid the prohibitive computational cost of DNNs, MC-dropout \cite{gal2016dropout} performs dropout sampling from the weight during training and testing. Ensemble-based methods \cite{lakshminarayanan2017simple} train and integrate multiple deep networks and achieve promising performance. To further reduce the parameters of ensemble models, different enhancement strategies have been developed, including independent subnetworks \cite{havasi2021training} and subnetworks of different depths \cite{antoran2020depth}.  
Instead of indirectly modeling uncertainty through network weights, the algorithm \cite{sensoy2018evidential} introduces the subjective logic theory to directly model it without ensemble or Monte Carlo sampling. Building upon RBF networks, the distance between test samples and prototypes can be used as the agency for deterministic uncertainty \cite{van2020uncertainty}. \textcolor{black}{Recently, uncertainty estimation algorithms based on Dirichlet distribution\cite{malinin2018predictive,Malinin2020Ensemble,sensoy2018evidential,charpentier2020posterior,kopetzki2021evaluating} have been proposed. PriorNet \cite{sensoy2018evidential} enables the model to be aware of out-of-distribution data by jointly exploiting out-of-distribution and in-distribution data. Evidential network \cite{sensoy2018evidential} models the Dirichlet distribution by introducing subjective logic. Ensemble distribution distillation \cite{Malinin2020Ensemble} obtains the Dirichlet distribution by distillation from the predictions of multiple models. PostNet \cite{charpentier2020posterior} employs normalizing flow and Bayesian loss to estimate uncertainty and obtain Dirichlet distribution during training. Median smoothing is applied to the Dirichlet model and significantly improves its ability to deal with adversarial examples \cite{kopetzki2021evaluating}.} 


\textbf{Dempster-Shafer Evidence Theory (DST).}  DST, a theory on belief functions, was first proposed by Dempster \cite{dempster1967upper} as a generalization of the Bayesian theory to subjective probabilities \cite{dempster1968generalization}. Later, it was developed into a general framework to model epistemic uncertainty \cite{shafer1976mathematical}. In contrast to BNNs, which obtain uncertainty indirectly (e.g., multiple stochastic samplings from weight parameters or variational inferences), DST models it directly. DST allows beliefs from different sources to be combined by various fusion operators to obtain a new belief that considers all available evidence \cite{sentz2002combination, josang2012interpretation}. When faced with beliefs from different sources, Dempster's rule of combination tries to fuse the shared parts, and ignores the conflicting beliefs through normalization factors. A more specific implementation will be discussed later. \textcolor{black}{DS evidence theory has also been extensively studied in multiview learning, which mainly focuses on improving the Dempster’s combination rule \cite{atrey2010multimodal,liu2017weighted,bloch1996some,kantardzic2010click,basir2007engine,le1997application}. Different from previous works, we estimate uncertainty through a variational Dirichlet distribution and perform trusted decision fusion through a reduced version of the Dempster’s combination rule with theoretical guarantees.}
\begin{figure*}[!t]
\centering
\subfigure[Overview of the  trusted multi-view classification]{
\begin{minipage}[t]{0.7\linewidth}
\centering
\includegraphics[width=1\linewidth,height=0.38\linewidth]{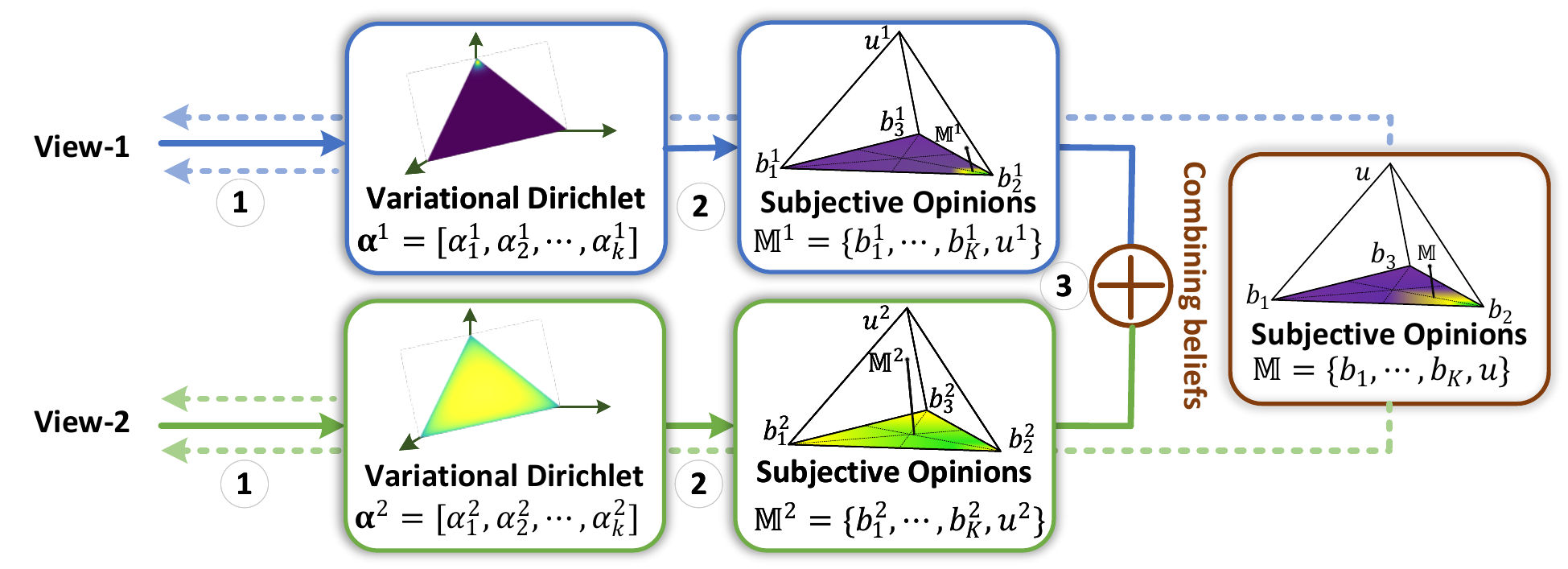}
\centering
\end{minipage}
\label{fig:framework1}}
\subfigure[Combining beliefs]{
\begin{minipage}[t]{0.250\linewidth}
\centering
\includegraphics[width=1\linewidth,height=1\linewidth]{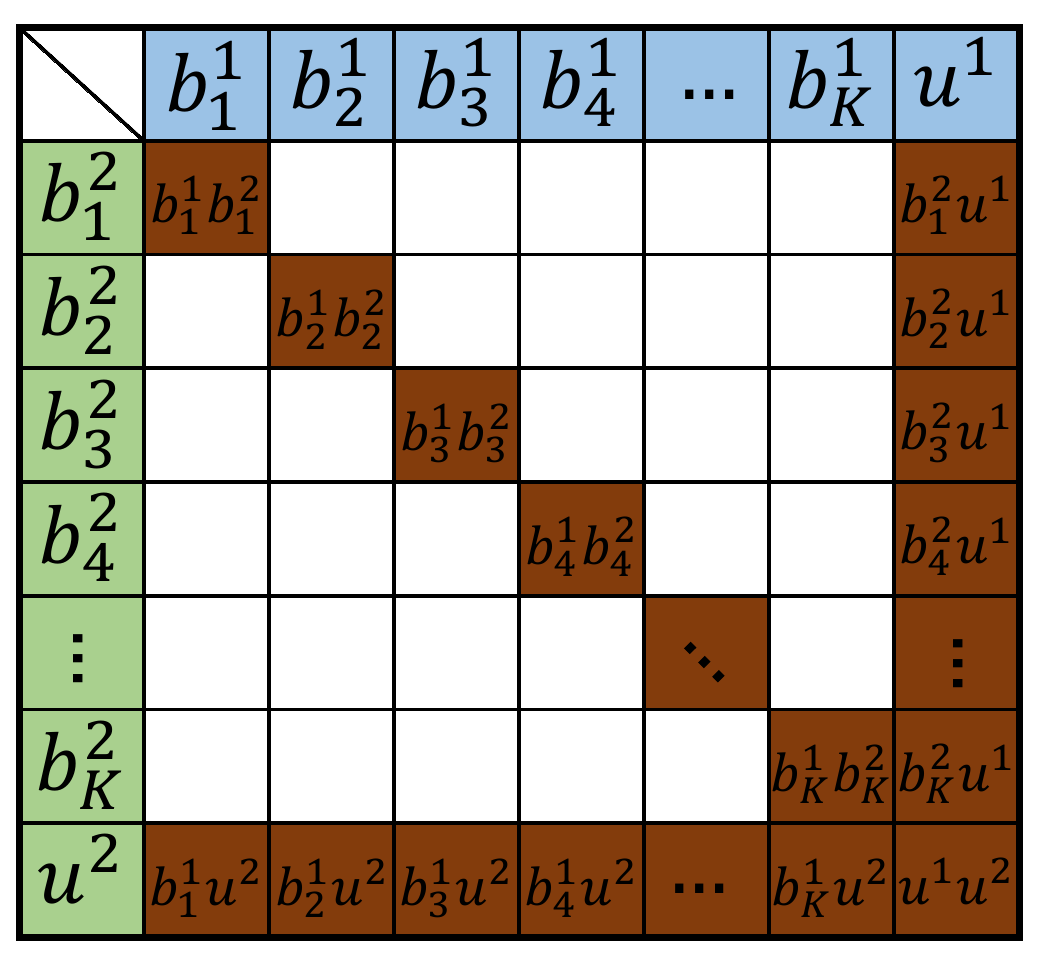}
\end{minipage}
\label{fig:framework2}}
\label{fig:framework}
\caption{Illustration of our algorithm. For clarity, we only use two views in the Fig.~\ref{fig:framework1}, where view-1 produces a confident observation and view-2 produces an uncertain observation. In this situation, the final obtained subjective opinions will be mainly determined by the confident view-1. The overall process of the model is composed of the following steps. Firstly, the Dirichlet distribution of each view is obtained using variational approximation (\term{1}). Then, the obtained Dirichlet distribution induces the subjective opinions including the belief and uncertainty masses (\term{2}). \textcolor{black}{Unlike frequency probability, subjective opinion can express the unknown explicitly through uncertainty. The subjective opinions $\mathbb{M}$ can be considered as a point in a higher-dimensional standard simplex space.} To integrate the subjective opinions from different 
views, DST-based combination rule is adopted (\term{3}). In training, the parameters of the model are updated along the dotted lines. The DST combination rule and an example are shown in Definition~\ref{def2} and Fig.~\ref{fig:framework2}, respectively. Specifically, given two sets of beliefs (blue and green blocks), we recombine the compatible parts of the two sets (brown blocks) and ignore the mutually exclusive parts (white blocks) to obtain the combined beliefs.}
\end{figure*}

\section{Trusted Multi-View Classification}
To obtain trusted multi-view fusion based on uncertainties of different views, reliable single-view uncertainty is required. The softmax output is usually considered as the confidence of prediction. However, this can lead to over confidence, even for erroneous prediction since the largest softmax output is usually used \cite{moon2020confidence,van2020uncertainty}. To capture the predictive uncertainty for single view, the Dirichlet distribution is considered to provide more trusted predictions \cite{sensoy2018evidential, NEURIPS2018_3ea2db50, Malinin2020Ensemble, NEURIPS2020_0eac690d}. It is intractable to fuse multiple Dirichlet distributions of different views for a closed-form solution. Thus, to flexibly exploit the uncertainty from multiple Dirichlet distributions, the subjective logic is introduced to estimate the uncertainty of different views. Accordingly, we propose an evidence based combination rule for trusted fusion. Systematic theoretical analysis is conducted to prove the effectiveness of the proposed method.

\subsection{Variational Dirichlet for Class Distribution}
\label{sec:varitional}
For a $K$-class classification problem, softmax operator is widely used in deep neural networks to convert the continuous output to class probabilities $\boldsymbol{\mu}=[\mu_1,\cdots,\mu_K]$, where $\sum_{k=1}^{K}\mu_k=1$. Mathematically, the class probabilities $\boldsymbol{\mu}$ can be regarded as a multinomial distribution parameters which describes the probability of $K$ mutually exclusive events \cite{bishop2006pattern}. Although effective, softmax operator usually leads to over confidence \cite{moon2020confidence,van2020uncertainty}. For this issue, we try to obtain a Dirichlet distribution which can be considered as the conjugate prior of the multinomial distribution \cite{bishop2006pattern} and provide a predictive distribution. Accordingly, uncertainty could be obtained according to the Dirichlet distribution to alleviate overconfidence problem \cite{sensoy2018evidential, NEURIPS2018_3ea2db50, Malinin2020Ensemble}. The definition of Dirichlet distribution can be found in Def.~\ref{def:Dir}.

\begin{definition} (\textbf{Dirichlet Distribution})
\label{def:Dir}
The Dirichlet distribution is parameterized by its $K$ concentration parameters $\mathbf{\boldsymbol{\alpha}}=[\alpha_1, \ldots, \alpha_K]$. The probability density function of the Dirichlet distribution is given by:
\begin{equation}Dir(\mathbf{\boldsymbol{\mu}} \mid \boldsymbol{\alpha})=\left\{\begin{array}{ll}
\frac{1}{B(\boldsymbol{\alpha})} \prod_{i=1}^{K} \mu_{i}^{\alpha_{i}-1} & \text { for } \mathbf{\boldsymbol{\mu}} \in \mathcal{S}_{K} ,\\
0 & \text { otherwise ,}
\end{array}\right.\end{equation}
where $\mathcal{S}_{K}$ is the $K-1$ dimensional unit simplex, defined as:
\begin{equation}
\mathcal{S}_{K}=\left\{\mathbf{\boldsymbol{\mu}} \mid \sum_{i=1}^{K} \mu_{i}=1 \text { and } 0 \leq \mu_{1}, \ldots, \mu_{K} \leq 1\right\},
\end{equation}
and $B(\boldsymbol{\alpha})$ is the K-dimensional multinomial beta function.
\label{def:dirichlet}
\end{definition}

\textbf{Variational Lower Bound}. Given a multi-view observation $\mathbb{X}=\{\mathbf{x}^m\}_{m=1}^{M}, \mathbb{X}\in \mathcal{X}$ and the corresponding label $\mathbf{y}\in \mathcal{Y}$, \textcolor{black}{where $\mathcal{X}$ and $\mathcal{Y}$ are sets of all observations and corresponding labels, respectively.}  For view $m$, we introduce the following generative process:
\begin{equation}
\begin{aligned}
\label{eq:gen}
\boldsymbol{\mu}^m &\sim p_{\theta^m}(\boldsymbol{\mu}^m\mid \mathbf{x}^m)=Dir(\boldsymbol{\mu}^m\mid \boldsymbol{\alpha}^m) \\
& \text{and} \quad \mathbf{y} \sim p(\mathbf{y} \mid \boldsymbol{\mu}^m) = Mult(\mathbf{y} \mid \boldsymbol{\mu}^m).
\end{aligned}
\end{equation}
$\boldsymbol{\mu}^m=[\mu_1^m, \cdots, \mu_K^m]$ are the parameters of the multinomial distribution, which are also regarded as the probabilities of multiple classes. $p_{\theta}(\boldsymbol{\mu}^m\mid \mathbf{x}^m)$ is a probabilistic encoder inducing the Dirichlet distribution $Dir(\boldsymbol{\mu}^m\mid \boldsymbol{\alpha}^m)$ which is inpsired by variational autoencoder \cite{kingma2014auto}. $Mult(\mathbf{y}\mid \boldsymbol{\mu}^m)$ is a multinomial distribution with parameters $\boldsymbol{\mu}^m$. Different from most deep neural networks that directly obtain multinomial distribution via softmax operator, Eq.~\ref{eq:gen} considers the conjugate prior of the multinomial distribution. Then the marginal likelihood of $p_{\theta^m}{(\mathbf{y}\mid \mathbf{x}^m)}$ can be decomposed as (the details are in the appendix):
\begin{equation}
\begin{aligned}
\label{eq:likelihood}
& \log p_{\theta^m}(\mathbf{y}\mid \mathbf{x}^m) = \\
& D_{KL}\left[q_{\theta^m}{(\boldsymbol{\mu}^m\mid \mathbf{x}^m)\| p_{\theta^m}(\boldsymbol{\mu}^m \mid \mathbf{x}^m, \mathbf{y})}\right] + \mathcal{L}(\mathbf{x}^m, \mathbf{y}),
\end{aligned}
\end{equation}
where  $q_{\theta^m}(\boldsymbol{\mu}^m \mid \mathbf{x}^m)$ is a probabilistic encoder with parameters $\theta^m$, and $\mathcal{L}(\mathbf{x}^m, \mathbf{y})$ can be decomposed as:
\begin{equation}
\begin{aligned}
\label{eq:loss_kl}
\mathcal{L}(\mathbf{x}^m, \mathbf{y}) = & \mathbb{E}_{q_{\theta^m}(\boldsymbol{\mu}^m\mid \mathbf{x}^m)}\left[\log p(\mathbf{y}\mid \boldsymbol{\mu}^m)\right]\\
& - D_{KL}\left[q_{\theta^m}(\boldsymbol{\mu}^m \mid \mathbf{x}^m)\| p_{\theta^m}(\boldsymbol{\mu}^m \mid \mathbf{x}^m))\right].
\end{aligned}
\end{equation}

Since the Kullback–Leibler divergence in Eq.~\ref{eq:likelihood} is always non-negative, and then $\mathcal{L}(\mathbf{x}^m, \mathbf{y})$ is the lower bound of the marginal likelihood $\log p_{\theta^m}(\mathbf{y}\mid \mathbf{x}^m)$. Therefore, we could maximize the marginal likelihood by optimizing $\mathcal{L}(\mathbf{x}^m, \mathbf{y})$.

\textbf{Decomposition of $\mathcal{L}(\mathbf{x}^m, \mathbf{y})$}. The first term in Eq.~\ref{eq:loss_kl} is essentially the integral of the traditional cross-entropy loss on the simplex determined by $Dir(\boldsymbol{\mu}^m \mid \boldsymbol{\alpha}^m)$, which can be formalized as: 
\begin{equation}
\begin{aligned}
& \mathbb{E}_{q_{\theta^m}(\boldsymbol{\mu}^m\mid \mathbf{x}^m)} \left[\log p(\mathbf{y}\mid \boldsymbol{\mu}^m)\right]\\ 
& = \mathbb{E}_{Dir(\boldsymbol{\mu}^m\mid \boldsymbol{\alpha}^m)}\left[\log p(\mathbf{y}\mid \boldsymbol{\mu}^m)\right]
= \sum_{k=1}^{K} y_{k}\left(\psi\left(\alpha_{k}^{m}\right) - \psi\left(S^{m}\right)\right),
\end{aligned}
\end{equation}
where $y_k$ is the $k$-th element of $\mathbf{y}$ represented by a onehot vector, $\psi(\cdot)$ is the \emph{digamma} function, and $S^m=\sum_{k=1}^K\alpha^m_k$ is the Dirichlet strength corresponding to the Dirichlet distribution $Dir(\boldsymbol{\mu}^m\mid \boldsymbol{\alpha}^m)$.

The second term in Eq.~\ref{eq:loss_kl} is essentially a prior constraint to ensure a more reasonable Dirichlet distribution. For multi-classification, reasonable $\boldsymbol{\mu}$ should satisfy $\mu_k \gg \mu_j$, where $k$ corresponds to the groundtruth label and $k\neq j$. Correspondingly, the reasonable Dirichlet distribution should be concentrated on the vertex corresponding to the groundtruth label on a simplex. To obtain concentrated Dirichlet distribution, we replace the softmax operator with a non-negative activation, and add $1$ to the outputs. Formally, assuming the neural network output is $\mathbf{o}$, then the obtained Dirichlet distribution parameters $\boldsymbol{\alpha} = \sigma(\mathbf{o}) + \mathbf{1}$, where $\sigma$ is non-negative activation function (e.g., Softplus, ReLU). In this manner, the obtained parameters $\boldsymbol{\alpha} = [\alpha_1, \cdots, \alpha_K]$ of Dirichlet distribution satisfy $\alpha_k\geq1$ where $ K\geq k\geq 1$. For multi-classification, given label $\mathbf{y}$, a reasonable Dirichlet distribution should be concentrated on the vertex corresponding to the label $\mathbf{y}$ on a simplex, which implies that the parameters of the Dirichlet distribution should be as close as possible to $\mathbf{1}$ except for that of the correct label. Formally, we can get the following equation to add the prior constraint:
\begin{equation}
\begin{aligned}
\label{eq:dkl}
 D_{KL}&\left[q_{\gamma}(\boldsymbol{\mu}^m \mid \mathbf{x}^m) \| p_{\gamma}(\boldsymbol{\mu}^m \mid \mathbf{x}^m))\right] \\
& \simeq D_{KL}\left[Dir(\boldsymbol{\mu}^m \mid \tilde{\boldsymbol{\alpha}}^m) \|Dir(\boldsymbol{\mu}^m \mid [1, \cdots, 1]) \right],
\end{aligned}
\end{equation}
where $\tilde{\boldsymbol{\alpha}}^m = \mathbf{y} + (1-\mathbf{y})\odot\boldsymbol{\alpha}^m$ is the Dirichlet distribution after replacing the $\alpha_k$ corresponding to the groundtruth label with $1$ which can avoid penalizing the Dirichlet parameter of the groundtruth class to $1$. Then Eq.~\ref{eq:dkl} can be derived as: 
\begin{equation}
\label{eq:kl_d}
\begin{array}{l}
D_{KL}\left[Dir(\boldsymbol{\mu}^m \mid \tilde{\boldsymbol{\alpha}}^m) \|Dir(\boldsymbol{\mu}^m \mid [1, \cdots, 1]) \right] \\
\quad=\log \left(\frac{\Gamma\left(\sum_{k=1}^{K} \tilde{\alpha}_{k}\right)}{\Gamma(K) \prod_{k=1}^{K} \Gamma\left(\tilde{\alpha}_{k}\right)}\right) \\
\quad\quad+\sum_{k=1}^{K}\left(\tilde{\alpha}_{k}-1\right)\left[\psi\left(\tilde{\alpha}_{k}\right)-\psi\left(\sum_{j=1}^{K} \tilde{\alpha}_{j}\right)\right],
\end{array}
\end{equation}
where $\Gamma(\cdot)$ is the \emph{gamma} function.

In practice, the overall variational objective function for view $m$ can be written as: 
\begin{equation}
\begin{aligned}
\mathcal{L}&(\mathbf{x}^m, \mathbf{y})=\mathbb{E}_{q_{\theta}(\boldsymbol{\mu}^m\mid \mathbf{x}^m)} \left[\log p(\mathbf{y}\mid \boldsymbol{\mu}^m)\right] \\ 
& - \lambda_t D_{KL}\left[Dir(\boldsymbol{\mu}^m \mid \tilde{\boldsymbol{\alpha}}^m) \|Dir(\boldsymbol{\mu}^m \mid [1, \cdots, 1]) \right],
\end{aligned}
\end{equation}
where $\lambda_t$ balances the expected classification error and KL-regularization \cite{higgins2017beta}. We gradually increase $\lambda_{t}$ to prevent the network from overemphasizing the KL divergence in the beginning of training, which may result in insufficient exploration of the parameter space and output a nearly flat uniform distribution. 

\subsection{Uncertainty and Evidence Theory}
\label{sec:evidence}
After obtaining the Dirichlet distributions from different views, we now focus on quantifying the uncertainty of the Dirichlet distribution to flexibly integrate multiple Dirichlet distributions. The fusion of multiple Dirichlet distributions is intractable in practice, since there is no closed-form solution for integrating multiple Dirichlet distributions. Therefore, the subjective logic \cite{jsang2018subjective} is introduced. In the context of multi-class classification, subjective logic provides a theoretical framework to associate the parameters of the Dirichlet distribution with the belief and uncertainty. In this way, we could quantify the uncertainty for classification, jointly modeling the probability of each class and the overall uncertainty of the current classification. For the $K$-class classification problem, subjective logic assigns belief mass $\{b_k\}_{k=1}^{K}$ to each class and an overall uncertainty mass $u$ to the whole classes frame based on the Dirichlet distribution. For the $m^{th}$ view, the $K+1$ mass values are all non-negative and their sum is one:
\begin{equation}
u^m+\sum_{k=1}^{K}{b_k^m} = 1, 
\label{eq:0}
\end{equation}
where $u^m \geq 0$ and $b_k^m \geq 0$ denote the overall uncertainty for view $m$ and the belief mass of the $k^{th}$ class, respectively.

Given a Dirichlet distribution $Dir(\boldsymbol{\mu}^m\mid \boldsymbol{\alpha}^m)$, subjective logic \cite{jsang2018subjective} and evidential classification \cite{sensoy2018evidential,bao2021evidential} associate the \emph{evidence} $\boldsymbol{e}^m=[e_1^m, \cdots, e_K^m]$ with the concentration parameters of the Dirichlet distribution $\boldsymbol{\alpha}^m=[\alpha_1^m, \cdots, \alpha_K^m]$. The \emph{evidence} refers to the metrics acquired from the input to support classification, which is closely related to the expected concentration parameters of the Dirichlet distribution. Specifically, the 
relationship between $e_k^m$ and $\alpha_k^m$ is $e_k^m=\alpha_k^m-1$.  Therefore, the belief mass $b_k^m$ and the uncertainty $u^m$ are obtained by:
\begin{equation}
b_{k}^m=\frac{e_k^m}{S^m}=\frac{\alpha_{k}^m-1}{S^m}\quad \text {and} \quad u^m = \frac{K}{S^m},
\label{eq:sl}
\end{equation}
where $S^m=\sum_{i=1}^{K}{(e^m_i+1)}=\sum_{i=1}^{K}{\alpha^m_i}$ is known as the Dirichlet strength. From Eq.~\ref{eq:sl}, it is easy to infer that the more evidence obtained for the $k^{th}$ class, the higher the assigned belief mass. Correspondingly, the less the total evidence obtained, the higher the overall uncertainty for classification. The belief assignment can be considered as subjective opinion. Given a subjective opinion, the mean of the corresponding Dirichlet distribution $\hat{\mathbf{\mu}}^m$ for the class probability $\hat{\mu}_{k}^m$ is computed as $\hat{\mu}_{k}^m=\frac{\alpha_{k}^m}{S^m}$ \cite{frigyik2010introduction}.

\begin{figure*}[!t]
\centering
\subfigure[Confident prediction]{
\centering
\begin{minipage}[t]{0.23\linewidth}
\centering
\includegraphics[width=0.9\linewidth,height=0.75\linewidth]{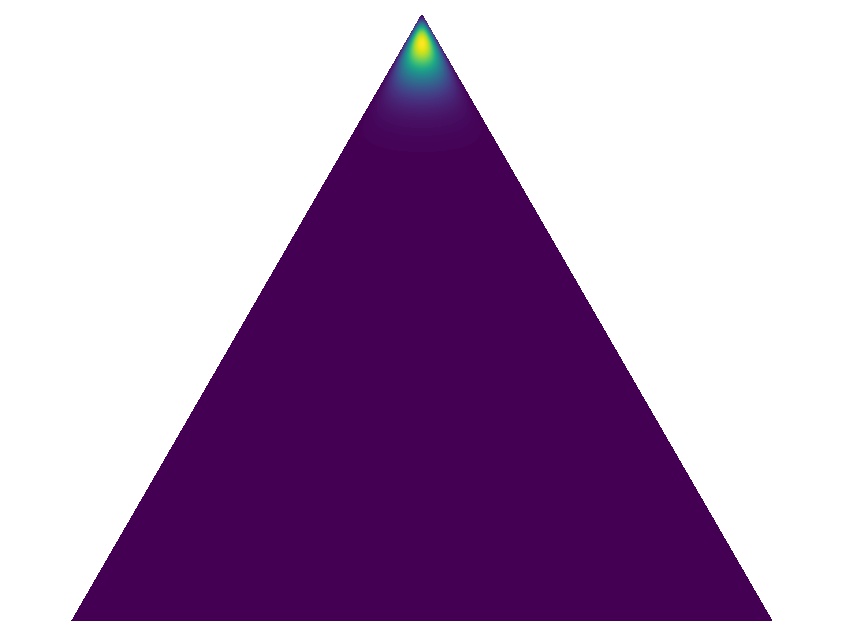}
\centering
\end{minipage}
}
\subfigure[{None evidence}]{
\begin{minipage}[t]{0.23\linewidth}
\centering
\includegraphics[width=0.9\linewidth,height=0.75\linewidth]{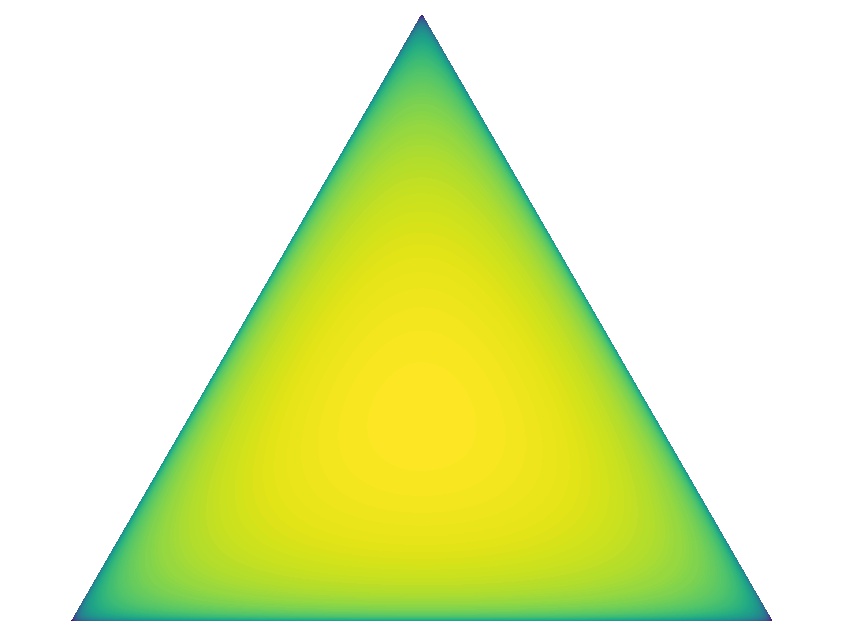}
\end{minipage}}
\centering
\subfigure[{Conflicting evidence}]{
\begin{minipage}[t]{0.23\linewidth}
\centering
\includegraphics[width=0.9\linewidth,height=0.75\linewidth]{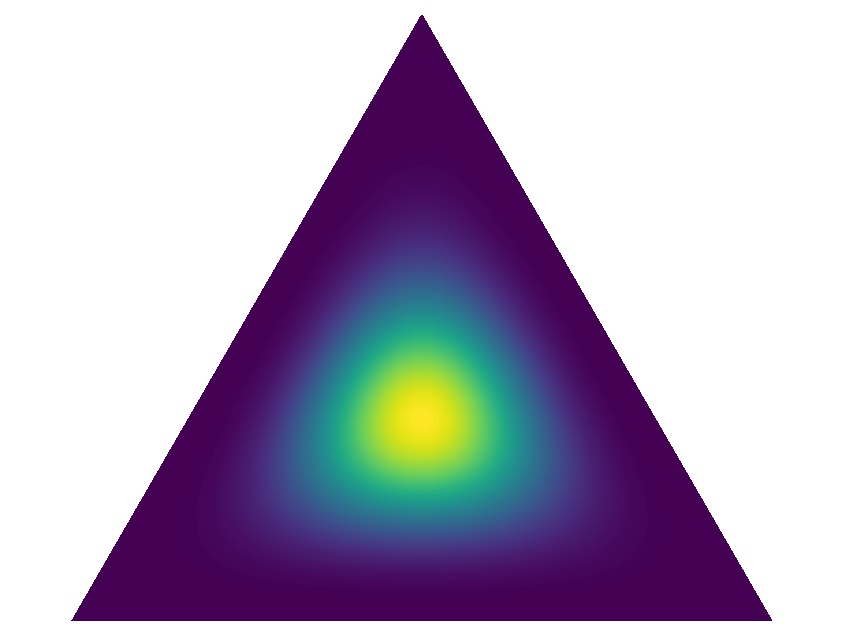}
\end{minipage}}
\subfigure[{Subjective opinion $(\mathbb{M})$}]{
\centering
\begin{minipage}[t]{0.23\linewidth}
\centering
\includegraphics[width=0.9\linewidth,height=0.8\linewidth]{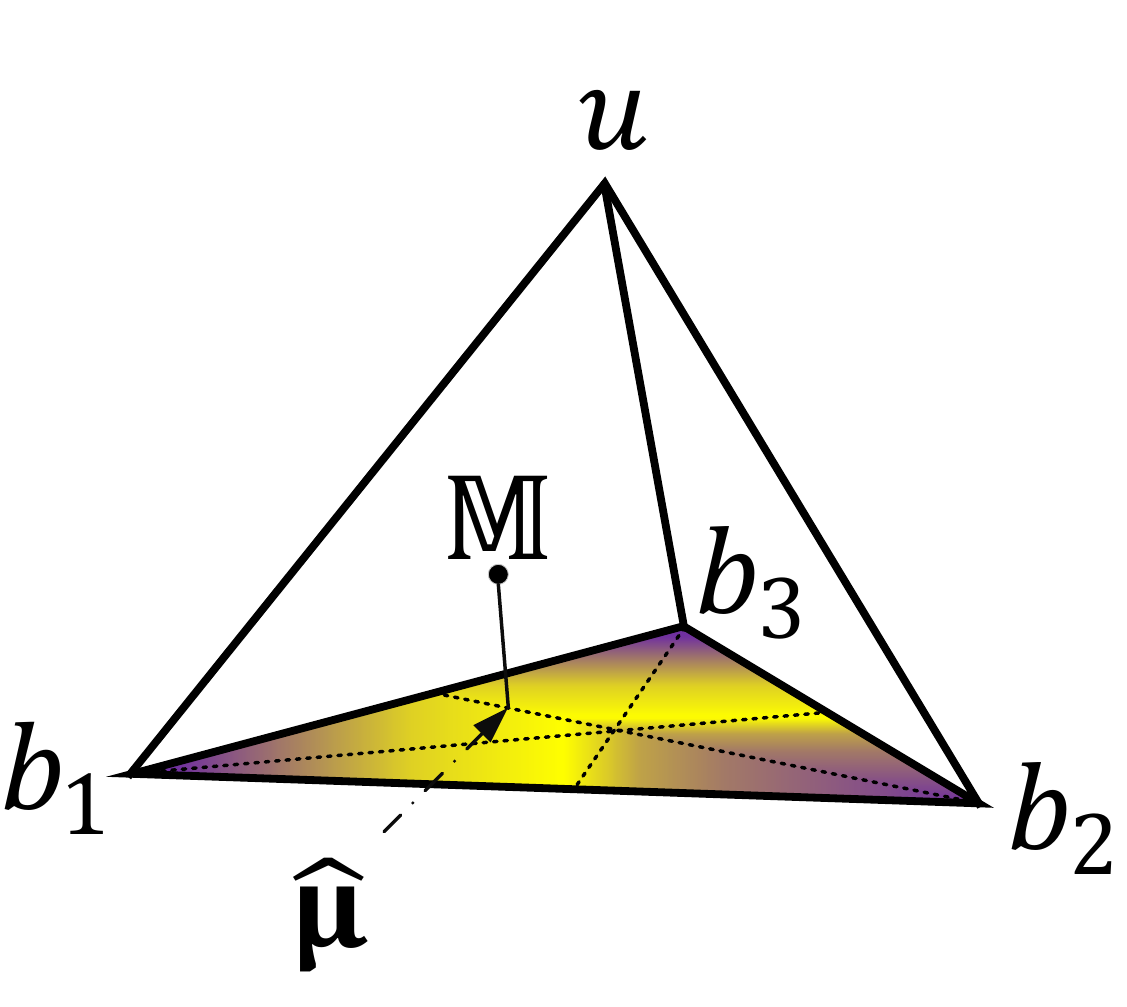}
\centering
\end{minipage}
}
\centering
\caption{Typical examples of Dirichlet distribution \cite{NEURIPS2018_3ea2db50, Malinin2020Ensemble} and subjective opinion. Refer to the text for details.}
\label{fig:dirichlet}
\end{figure*}

\textbf{Differences from Existing Classifiers.} First, in traditional neural network the output can be regarded as one point on a simplex, while the Dirichlet distribution characterizes the probability of each point on a simplex. Therefore, the proposed method models the second-order probability and overall uncertainty of the output with a Dirichlet distribution. Second, in traditional DNN-based classifiers, the softmax function is used in the last layer, which usually leads to over-confidence \cite{sensoy2018evidential,guo2017calibration}. In our model, we introduce variational Dirichlet distribution to avoid the issue by adding overall uncertainty mass. Moreover, existing uncertainty-based algorithms~\cite{gal2016dropout,lakshminarayanan2017simple} generally require additional computation to quantify uncertainty. The uncertainty is estimated during inference for these models~\cite{gal2016dropout,lakshminarayanan2017simple}, preventing them seamlessly promoting accuracy, robustness and uncertainty estimation in a unified framework. These limitations of existing algorithms (\emph{e.g.}, inability to directly obtain uncertainty) also impede their extension for trusted multi-view classification.

To intuitively illustrate above difference, we provide examples for a triple classification task. Assuming $\mathbf{\boldsymbol{\alpha}}=[41, 2, 2]$, we have $\mathbf{e}=[40, 1, 1]$. The corresponding Dirichlet distribution in Fig.~\ref{fig:dirichlet}(a) yields a sharp distribution where the density is centered at the top of a standard 2-simplex. This implies that sufficient evidence is collected for accurate classification. In contrast, $\mathbf{\boldsymbol{\alpha}}=[1.0001, 1.0001, 1.0001]$, provides weak evidence $\mathbf{e}=[0.0001, 0.0001, 0.0001]$ for classification. Then, the uncertainty mass is $u\approx 1$. In this case, the evidence induces a flat distribution over the simplex, as shown in Fig.~\ref{fig:dirichlet}(b). Finally, when $\mathbf{\boldsymbol{\alpha}}=
[6, 6, 6]$, since each class has stronger evidence and thus produces a relatively low uncertainty (Fig.~\ref{fig:dirichlet}(c)). However, the classification is still uncertain. We avoid this by exploiting the Kullback-Leibler (KL) divergence scheme which adds a prior constraint to ensure a more reasonable Dirichlet distribution, elaborated in Sec.~\ref{sec:varitional}.
As shown in Fig.~\ref{fig:dirichlet}(d), to describe the uncertainty, it is easy to transform a Dirichlet distribution with a standard 3-simplex, i.e., a regular tetrahedron with vertices (1,0,0,0), (0,1,0,0), (0,0,1,0) and (0,0,0,1) in $\mathbf{R}^4$, based on the subjective logic theory (Eq.~\ref{eq:0} and Eq.~\ref{eq:sl}). The point ($\mathbb{M}$) inside the simplex indicates an opinion $\big\{\{b_k\}_{k=1}^{3}, u\big\}$. The expectation of the Dirichlet distribution ($\hat{\boldsymbol{\mu}}$) is the projection of $\mathbb{M}$ on the bottom.

\subsection{DS-Combination Rule for Multi-View Classification}
\label{sec:ds}
After estimating the uncertainty of each view, we now focus on the adaptation to data with multiple views. We note that the Dempster-Shafer theory allows opinions (belief and uncertainty masses) from different sources to be integrated, producing a comprehensive opinion that takes all the available evidence into account. Unfortunately, for $K$-class classification, the number of input and output masses in the original Dempster's combination rule is very large (at most $2^K$), making it overly complex in application. To adapt Dempster's combination rule for our model, we propose the reduced Dempster's combination rule (Definition~\ref{def2}), which has the following advantages. First, the input masses are trimmed according to the $K$-class classification task, and thus the complexity of the fusion process is also reduced and efficient. Second, the reduced Dempster's combination rule allows us to conduct a theoretical analysis (Sec.~\ref{sec:analysis}). Finally, the reduced Dempster's combination rule inherits the essence (fusion based on uncertainty) of DST, enabling it to provide more  trusted classification.
Specifically, we need to combine $M$ independent sets of probability mass assignments $\{\mathbb{M}^m\}_{m=1}^M$, where $\mathbb{M}^m=\big\{\{b_k^m\}_{k=1}^{K}, u^m\big\}$, to obtain a joint mass $\mathbb{M}=\big\{\{b_k\}_{k=1}^{K}, u\big\}$ (step \term{4} in Fig.~\ref{fig:framework1}).

\begin{definition} (\textbf{Reduced Dempster's Combination Rule for $K$-Class Classification})
\label{def2} The combination (termed joint mass) $\mathbb{M}=\big\{\{b_k\}_{k=1}^{K}, u\big\}$ is calculated from the two sets of masses $\mathbb{M}^1= \big\{\{b_k^1\}_{k=1}^{K}, u^1\big\}$ and $\mathbb{M}^2=\big\{\{b_k^2\}_{k=1}^{K}, u^2\big\}$ with the following rule:
\begin{equation}
\mathbb{M}=\mathbb{M}^1\oplus \mathbb{M}^2.
\end{equation}
The more specific calculation rule can be formulated as follows:
\begin{equation}
\begin{aligned}
& b_{k}=\frac{1}{1-C}( b^1_kb^2_k + b^1_ku^2 + b^2_ku^1), u = \frac{1}{1-C}u^1u^2,
\label{eq:fusion1}
\end{aligned}
\end{equation}
where $C = \sum_{i\neq j}{b^1_i b^2_j}$ is a measure of the amount of conflict between the two mass sets (the white blocks in Fig.~\ref{fig:framework2}), and the scale factor $\frac{1}{1-C}$ is used for normalization. 
\end{definition}

The opinion $\mathbb{M}$ is formed by combining the opinions $\mathbb{M}^1$ and $\mathbb{M}^2$. Intuitively, the combined belief mass ($b_k$) and overall uncertainty ($u$) correspond to the brown blocks in Fig.~\ref{fig:framework2}. More specifically, the combination rule ensures: (1) When both views are of high uncertainty (large $u^1$ and $u^2$), the final classification must have low confidence (small $b_k$); (2) When both views are of low uncertainty (small $u^1$ and $u^2$), the final classification may have high confidence (large $b_k$); (3) When only one view is of low uncertainty (only $u^1$ or $u^2$ is large), the final classification only depends on the confident view; (4) When the opinions from the two views are in conflict, both $C$ and $u$ will accordingly increase.

Given data with $M$ views, we can obtain the above-mentioned mass for each. Then, we can combine these beliefs from different views with our combination rule. Specifically, we fuse the belief and uncertainty masses from different views by using the following rule:
\begin{equation}
\begin{aligned}
\mathbb{M}=\mathbb{M}^1\oplus \mathbb{M}^2 \oplus \cdots \mathbb{M}^V.
\label{eq:fusion2}
\end{aligned}
\end{equation}
After obtaining the joint mass $\mathbb{M}=\big\{\{b_k\}_{k=1}^{K}, u\big\}$, we can induce the corresponding joint evidence from multiple views according to Eq.~\ref{eq:sl}, and then the parameters of the Dirichlet distribution are obtained as:
\begin{equation}
\begin{aligned}
S = \frac{K}{u} , e_k =  b_k \times S \text{ and } \alpha_k = e_k+1.
\label{eq:sl2}
\end{aligned}
\end{equation}
With the combination rule, we can obtain the estimated multi-view joint evidence $\mathbf{\boldsymbol{e}}$ and the corresponding parameters of the joint Dirichlet distribution $Dir(\boldsymbol{\mu}\mid \mathbf{\boldsymbol{\alpha}})$, producing the final probability of each class and the overall uncertainty. Compared with softmax, subjective uncertainty is more suitable for the fusion of multiple decisions. Subjective logic provides an additional mass function ($u$) that allows the model to perceive the lack of evidence. In our model, subjective logic provides uncertainty for each view and also for the overall classification, which is important for trusted classification and explainable fusion.

\subsection{Theoretical Analysis of the Combination Rule}
\label{sec:analysis}
To illustrate the advantages of using the reduced combination rule for trusted multi-view classification, we theoretically analyze it by providing four propositions from the perspectives of \textit{classification accuracy} and \textit{uncertainty estimation}.

Specifically, given the original opinion $\mathbb{M}^o=\big\{\{b^o_k\}_{k=1}^{K}, u^o\big\}$, we aim to theoretically analyze whether the integration of another opinion $\mathbb{M}^a=\big\{\{b^a_k\}_{k=1}^{K}, u^a\big\}$ will reduce the performance (both in terms of accuracy and uncertainty estimation) of the classifier. Suppose that after fusing $\mathbb{M}^o$ and $\mathbb{M}^a$, the new opinion is $\mathbb{M}=\big\{\{b_k\}_{k=1}^{K}, u\big\}$. Specifically, we have the following conclusions:
\begin{itemize}
    \item[\textbf{Accuracy}:] (1) According to Prop.~\ref{prop1}, integrating an additional opinion into the original opinion can potentially improve the classification accuracy of the model. (2) According to Prop.~\ref{prop2}, when integrating an additional opinion into the original opinion, the possible performance degradation of the model is limited under mild conditions. 
    \item[\textbf{Uncertainty}:] (3) According to Prop.~\ref{prop3}, integrating an additional opinion reduces the integrated uncertainty mass. (4) According to Prop.~\ref{prop4}, when the uncertainties of two views are both large, the integrated uncertainty will also be large.
\end{itemize}
The following propositions provide theoretical analysis to support the above conclusions.

\begin{prop}
\label{prop1}
Under the conditions $b_t^a\geq b_m^o$, where $t$ is the index of ground-truth class and $b_m^o$ is the largest belief mass in $\{b_k^o\}_{k=1}^{K}$, integrating another opinion $\mathbb{M}^a$ makes the new opinion satisfy $b_t\geq b_t^o$.
\end{prop}
\begin{sproof}
\begin{equation}
\begin{aligned}
\nonumber
b_t & = \frac{b_t^o b_t^a + b_t^o u^a + b_t^a u^o}{\sum_{k=1}^Kb_k^ob_k^a+u^a+u^o-u^ou^a}\\
&\geq \frac{b_t^o( b_t^a + u^a + u^o)}{b_m^o(1-u^ a)+u^a+u^o-u^ou^a}\\
&\geq \frac{b_t^o( b_t^a + u^a + u^o)}{b_m^o+u^a+u^o}\geq b_t^o.
\end{aligned}
\end{equation}
\end{sproof}
\begin{prop}
\label{prop2}
When $u^a$ is large, $b_t^o-b_t$ will be limited, and it will have a negative correlation with $u^a$. As a special case, when $u^a$ is large enough (i.e., $u^a = 1$), integrating another opinion will not reduce the performance (i.e., $b_t = b_t^o$).
\end{prop}
\begin{sproof}
\begin{equation}
\begin{aligned}
\nonumber
b_t^o-b_t&=b_t^o - \frac{b_t^o b_t^a + b_t^o u^a + b_t^a u^o}{\sum_{k=1}^Kb_k^ob_k^a+u^a+u^o-u^ou^a}\\
&\leq b_t^o - \frac{b_t^o u^a}{b_m^a+u^a+u^o-u^au^o}\\
&\leq b_t^o - \frac{b_t^o u^a}{1+u^o-u^au^o}=b_t^o\frac{1+u^o}{\frac{1}{1-u^a}+u^o}.
\end{aligned}
\end{equation}
\end{sproof}
\begin{prop}\label{prop3}
After integrating another opinion $\mathbb{M}^a$ into the original opinion $\mathbb{M}^o$, the obtained uncertainty mass $u$ will be reduced, i.e., $u \leq u^o$.
\end{prop}
\begin{sproof}
\begin{equation}
\begin{aligned}
\nonumber
u&=\frac{u^o u^a}{1-C}=\frac{u^o u^a}{\sum_{k=1}^Kb_k^ob_k^a+u^a+u^o-u^ou^a}\\
&\leq \frac{u^o u^a}{u^a + u^o -u^o u^a} \leq u^o.
\end{aligned}
\end{equation}
\end{sproof}
\begin{prop}
\label{prop4}
$u$ has a positive correlation with $u^a$ and $u^o$. 
\end{prop}
\begin{proof}
\begin{equation}
\begin{aligned}
\nonumber
u & = \frac{u^a u^o}{\sum_{k=1}^K\{b^a_kb^o_k + b^a_ku^o + b^o_ku^a\}+u^a u^o}\\
& = \frac{1}{\sum_{k=1}^K\{\frac{b^a_kb^o_k}{u^a u^o} + \frac{b^a_k}{u^a} + \frac{b^o_k}{u^o}\}+1}.
\end{aligned}
\end{equation}
\end{proof}

Note that, although the above analysis is based on the integration of two opinions, it can easily be generalized to multiple opinions.

\begin{figure}[!t]
\centering
\includegraphics[width=0.95\linewidth,height=0.345\linewidth]{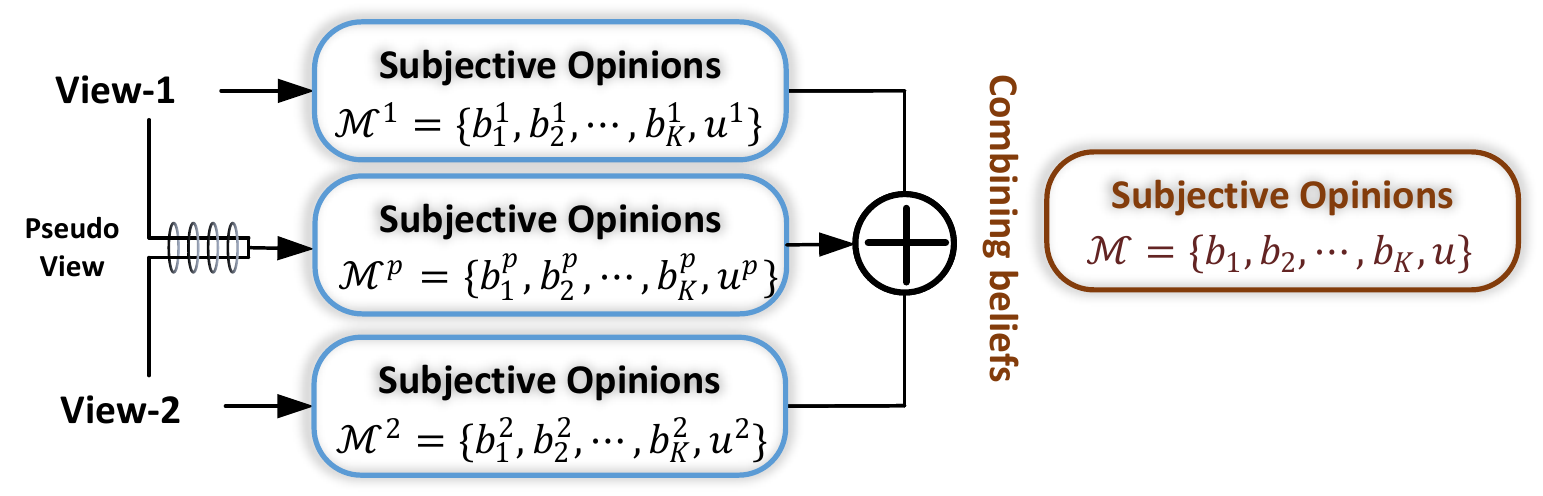}
{\caption{Overview of enhanced  trusted multi-view classification, where the subjective opinions of each view can be obtained using steps \term{1} and \term{2} in Fig.~\ref{fig:framework}.}}
\label{fig:framework_extension}
\end{figure}

\subsection{Pseudo-View Enhanced TMC}
\label{sec:pseudo}
Although the proposed TMC can provide trusted classification, it lacks full interaction between different views due to the late-fusion strategy. Specifically, unlike in the traditional multi-view learning algorithms, the interaction between different views only exists in the decision-making layer in TMC, which may limit its performance in some scenarios, e.g., for scenes requiring complementary information from different views at the feature or representation level. To address this, we introduce a pseudo-view to promote interaction between different views, while maintaining the trustworthiness of the original model. As discussed in Sec.~\ref{sec:analysis}, integrating a new opinion into the original one can potentially improve the performance of the model under reasonable conditions and avoid negative effects. More specifically, we add the pseudo-view (view-$p$), which is generated in a flexible way based on the original views. The framework of the model is shown in Fig~\ref{fig:framework_extension}.

A direct way to ensure that the generated pseudo-view contains complementary information from multiple views is to simply concatenate the original features. In addition, we can flexibly adopt appropriate multi-view representation learning algorithms according to specific applications. Assume that the belief and uncertainty masses corresponding to the pseudo-view are $\mathbb{M}^p = \{\{b_k^p\}_{k=1}^{K},u^p\}$. Then, comprehensive results can be derived as $\mathbb{M}=\mathbb{M}^o\oplus\mathbb{M}^p$, where $\mathbb{M}^o = \mathbb{M}^1\oplus \mathbb{M}^2 \oplus \cdots \mathbb{M}^M$.

Unlike traditional early-fusion-based multi-view classification methods, our TMC can dynamically evaluate the quality of different views and make trusted decisions based on the uncertainty. Furthermore, in contrast to TMC, the proposed enhanced  trusted multi-view classification (ETMC) can simultaneously explore interactions between different views at a representation level by adding a pseudo-view to improve model performance, with theoretical guarantees. 

\subsection{Training to Form Multi-View Opinions}
\label{sec:3.5}
We now focus on jointly training neural networks to obtain Dirichlet distribution for each view, which can then be used to infer the corresponding masses $\{\mathbb{M}^v\}_{v=1}^V$ and $\mathbb{M}$.
In conventional neural-network-based classifiers, the cross-entropy loss is widely used:
\begin{equation}
    \mathcal{L}_{c e} = -\sum_{j=1}^{K} y_{i j}\log (p_{i j}),
    \label{eq:crossentropy}
\end{equation}
where $p_{i j}$ is the predicted probability for the $i$th sample belonging to class $j$. 

In our model, to obtain the Dirichlet distribution, varitional loss function for single view is proposed in Sec. \ref{sec:varitional}. The loss function for view $m$ can be written as:
\textcolor{black}{
\begin{equation}
\begin{aligned}
\mathcal{L}^m&(\mathbf{x}^m, \mathbf{y}) =  \mathbb{E}_{q_{\theta}(\boldsymbol{\mu}^m\mid \mathbf{x}^m)} \left[\log p(\mathbf{y}\mid \boldsymbol{\mu}^m)\right] 
\\
& - \lambda_{t} K L\left[D\left(\boldsymbol{\mu}^m \mid \tilde{\boldsymbol{\alpha}}^m\right) \| D\left(\boldsymbol{\mu}^m |[1, \cdots, 1]\right)\right].
\label{eq:la}
\end{aligned}
\end{equation}
}
Considering $N$ i.i.d. multi-view observations and the corresponding labels $\{\{\mathbf{x}^m_n\}_{m=1}^M, \mathbf{y}_n\}_{n=1}^N$, we employ a multi-task strategy to enable all views to jointly form reasonable subjective opinions and accordingly improve the fused opinion:
\textcolor{black}{
\begin{equation}
\begin{aligned}
\label{eq:loss}
& \mathcal{L}^{overall} = \sum_{i=1}^{N} \mathcal{L}^{fused}(\{\mathbf{x}^m_n\}_{m=1}^M, \mathbf{y}_n)+ \sum_{i=1}^{N}\sum_{m=1}^{M}\mathcal{L}^m(\mathbf{x}_n^m, \mathbf{y}_n),  \\
& \text{with } \mathcal{L}^{fused}(\{\mathbf{x}^m_n\}_{m=1}^M, \mathbf{y}_n) = \mathbb{E}_{\boldsymbol{\mu}\sim Dir(\boldsymbol{\mu}\mid \boldsymbol{\alpha})} \left[\log p(\mathbf{y}\mid \boldsymbol{\mu})\right] \\ 
&\quad \quad - \lambda_t D_{KL}\left[Dir(\boldsymbol{\mu} \mid \tilde{\boldsymbol{\alpha}}) \|Dir(\boldsymbol{\mu} \mid [1, \cdots, 1]) \right],
\end{aligned}
\end{equation}
}
where $Dir(\boldsymbol{\mu\mid \boldsymbol{\alpha}})$ is obtained with subjective logic and DS-combination rule which has been elaborated in Sec.~\ref{sec:evidence} and Sec.~\ref{sec:ds} respectively.

After introducing the pseudo-view proposed in Sec.~\ref{sec:pseudo}, the objective function is:
\textcolor{black}{
\begin{equation}
\begin{aligned}
\label{eq:loss_with_p}
& \mathcal{L}^{overall} = \sum_{i=1}^{N} \mathcal{L}^{fused}(\{\mathbf{x}^m_n\}_{m=1}^M, \mathbf{y}_n) \\&+ \sum_{i=1}^{N} \mathcal{L}^{pseudo}(\{\mathbf{x}^m_n\}_{m=1}^M, \mathbf{y}_n) +\sum_{i=1}^{N}\sum_{m=1}^{M}\mathcal{L}^m(\mathbf{x}_n^m, \mathbf{y}_n), \\&
\text{with  } \mathcal{L}^{pseudo}(\{\mathbf{x}^m_n\}_{m=1}^M, \mathbf{y}_n)= \mathbb{E}_{\boldsymbol{\mu}^p\sim Dir(\boldsymbol{\mu}^p\mid \boldsymbol{\alpha}^p)} \left[\log p(\mathbf{y}\mid \boldsymbol{\mu}^p)\right] \\ 
&\quad \quad - \lambda_t D_{KL}\left[Dir(\boldsymbol{\mu}^p \mid \tilde{\boldsymbol{\alpha}^p}) \|Dir(\boldsymbol{\mu}^p \mid [1, \cdots, 1]) \right],
\end{aligned}
\end{equation}
}
where $Dir(\boldsymbol{\mu}^p\mid \boldsymbol{\alpha}^p)$ is obtained with pseudo-view classifier.

The optimization process for the proposed model is summarized in Algorithm~\ref{alg:alg1}.

\begin{algorithm}[!h]
\SetAlgoLined
\caption{Algorithm for Trusted Multi-View Classification}
\textbf{/*Training*/}\\
\KwIn{Multi-view dataset: $\mathcal{D} = \{\{\mathbf{X}_n^m\}_{m=1}^{M}, y_n\}_{n=1}^N$.}
\textbf{Initialize:} {Initialize the parameters of the neural network}.\\
\While{not converged}{
\For{$m=1:M$ }
{
$Dir(\boldsymbol{\mu}^m\mid \mathbf{x}^m)\leftarrow$ variational network output;\\
Subjective opinion $\mathbf{M}^m\leftarrow Dir(\boldsymbol{\mu}^m\mid \mathbf{x}^m)$ ;\\
}
Obtain joint opinion $\mathbb{M}$ with Eq.~\ref{eq:fusion2};\\
Obtain $Dir(\boldsymbol{\mu} \mid \boldsymbol{\alpha})$ with Eq.~\ref{eq:sl2};\\
Obtain the overall loss by updating $\boldsymbol{\alpha}$ and $\{\mathbf{\boldsymbol{\alpha}}^v\}_{v=1}^{V}$ in Eq.~\ref{eq:loss};\\
Maximize Eq.~\ref{eq:loss} and update the networks with gradient descent;
}
\KwOut{networks parameters.}
\textbf{/*Test*/}\\
Calculate the joint belief and the uncertainty masses.
\label{alg:alg1}
\end{algorithm}

\section{Experiments}
In this section, we conduct experiments on diverse scenarios to comprehensively evaluate our algorithm. Specifically, we conduct experiments on six real-world multi-feature datasets, and then adopt our method for end-to-end evaluation on RGB-D scene recognition and image-text classification tasks.
\subsection{Experiments on Vector-Type Data}
\begin{table*}[!htbp]
\setlength{\arraycolsep}{0.3pt}
  \centering
  \small 
  \caption{Evaluation of the classification performance. The best two results are in bold and marked with a superscript.}
    \begin{tabular}{cccccccc}
    \toprule
    Data  & Metric & MCDO  & DE    & UA    & EDL   & TMC & ETMC \\
    \midrule
    \multirow{2}[1]{*}{Handwritten} & ACC   & 97.37$\pm$0.80 & 98.30$\pm$0.31 & 97.45$\pm$0.84 & 97.67$\pm$0.32 & \textbf{98.51$\pm$0.15}$^2$ &  \textbf{98.75$\pm$0.00}$^1$ \\
          & AUROC & 99.70$\pm$0.07 & 99.79$\pm$0.05 & 99.67$\pm$0.10 & 99.83$\pm$0.02 & \textbf{99.97$\pm$0.00}$^1$ & \textbf{99.95$\pm$0.00}$^2$ \\
              \hline
    \multirow{2}[0]{*}{CUB} & ACC   & 89.78$\pm$0.52 & 90.19$\pm$0.51 & 89.75$\pm$1.43 & 89.50$\pm$1.17 & \textbf{91.00$\pm$0.42}$^2$ & \textbf{91.04$\pm$0.69}$^1$\\
          & AUROC & \textbf{99.29$\pm$0.03}$^1$ & 98.77$\pm$0.03 & 98.69$\pm$0.39 & 98.71$\pm$0.03 & 99.06$\pm$0.03 & \textbf{99.13$\pm$0.03}$^2$\\
              \hline
    \multirow{2}[0]{*}{PIE} & ACC  & 84.09$\pm$1.45 & 70.29$\pm$3.17 & 83.70$\pm$2.70 & 84.36$\pm$0.87 & \textbf{91.99$\pm$1.01}$^2$ & \textbf{93.75$\pm$1.08}$^1$\\
          & AUROC & 98.90$\pm$0.31 & 95.71$\pm$0.88 & 98.06$\pm$0.56 & 98.74$\pm$0.17 & \textbf{99.69$\pm$0.05}$^2$ &\textbf{99.77$\pm$0.05}$^1$\\
              \hline
    \multirow{2}[0]{*}{Caltech101} & ACC   & 91.73$\pm$0.58 & 91.60$\pm$0.82 & 92.37$\pm$0.72 & 90.84$\pm$0.56 & \textbf{92.93$\pm$0.20}$^1$ & \textbf{92.83$\pm$0.33}$^2$\\
          & AUROC & \textbf{99.91$\pm$0.01}$^2$ & \textbf{99.94$\pm$0.01}$^1$ & 99.85$\pm$0.05 & 99.74$\pm$0.03 & 99.90$\pm$0.01 & 99.89$\pm$0.02\\
          \hline
    \multirow{2}[1]{*}{Scene15} & ACC   & 52.96$\pm$1.17 & 39.12$\pm$0.80 & 41.20$\pm$1.34 & 46.41$\pm$0.55 & \textbf{67.74$\pm$0.36}$^2$ & \textbf{71.61$\pm$0.28}$^1$\\
          & AUROC & 92.90$\pm$0.31 & 74.64$\pm$0.47 & 85.26$\pm$0.32 & 91.41$\pm$0.05 & \textbf{95.94$\pm$0.02}$^2$ & \textbf{96.17$\pm$0.02}$^1$\\
         \hline
    \multirow{2}[0]{*}{HMDB} & ACC   & 52.92$\pm$1.28 & 57.93$\pm$1.02 & 53.32$\pm$1.39 & 59.88$\pm$1.19 & \textbf{65.26$\pm$0.76}$^2$ & \textbf{69.43$\pm$0.67}$^1$\\
          & AUROC & 93.57$\pm$0.28 & 94.01$\pm$0.21 & 91.68$\pm$0.69 & 94.00$\pm$0.25 & \textbf{96.18$\pm$0.10}$^1$ & \textbf{95.58$\pm$0.12}$^2$\\
    \bottomrule
    \end{tabular}%
  \label{tab:addlabel}%
\end{table*}%
\begin{figure*}[!htbp]
\centering
\subfigure[Handwritten]{
\centering
\begin{minipage}[t]{0.3\linewidth}
\centering
\includegraphics[width=1\linewidth,height=0.7\linewidth]{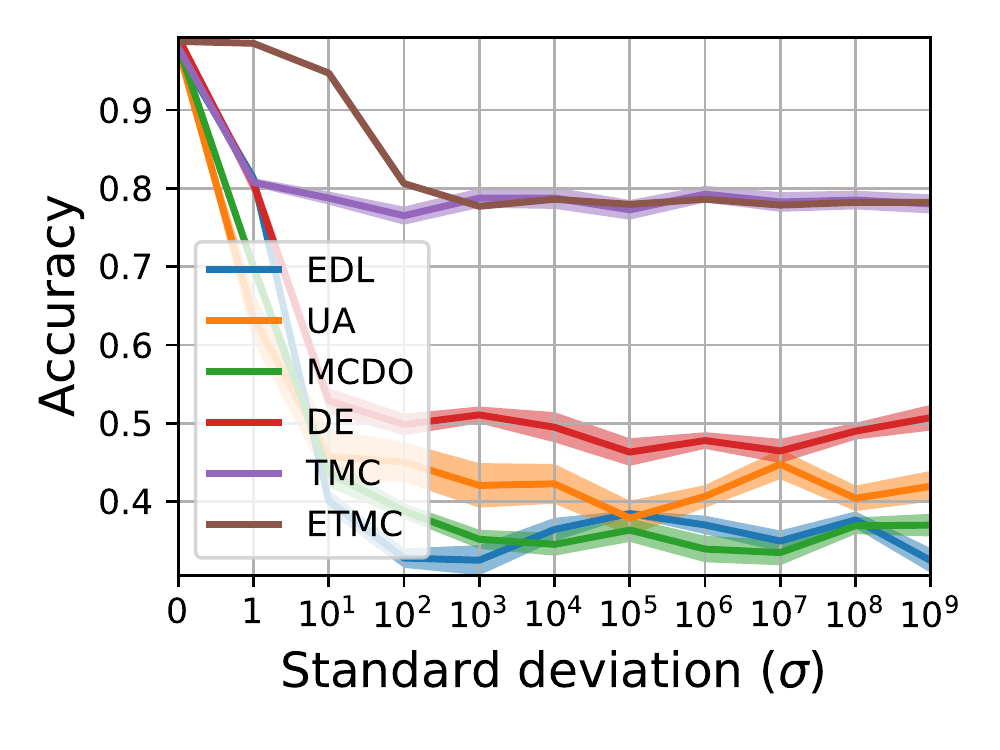}
\centering
\end{minipage}}
\subfigure[CUB]{
\begin{minipage}[t]{0.3\linewidth}
\centering
\includegraphics[width=1\linewidth,height=0.7\linewidth]{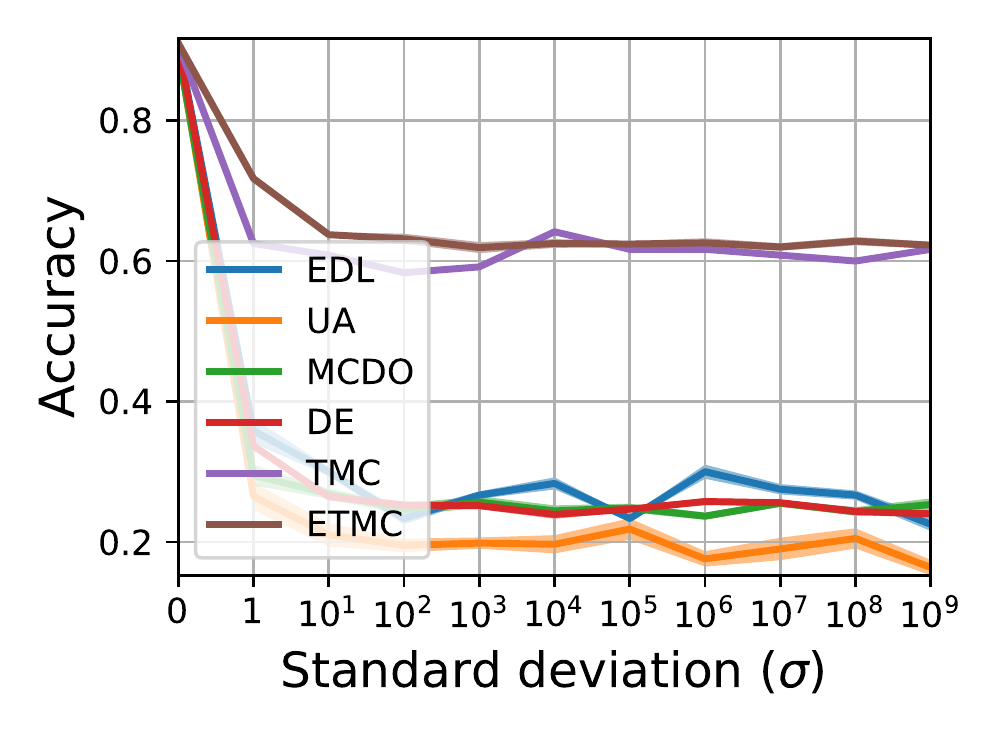}
\end{minipage}}
\centering
\subfigure[PIE]{
\begin{minipage}[t]{0.3\linewidth}
\centering
\includegraphics[width=1\linewidth,height=0.7\linewidth]{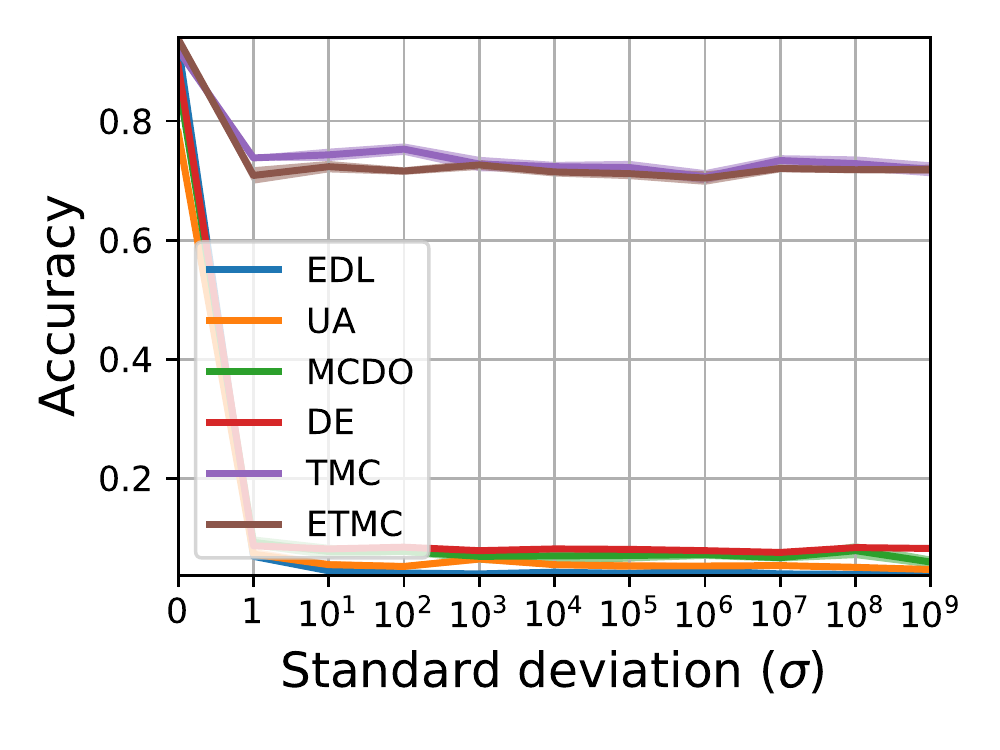}
\end{minipage}}
\subfigure[Caltech101]{
\begin{minipage}[t]{0.3\linewidth}
\centering
\includegraphics[width=1\linewidth,height=0.7\linewidth]{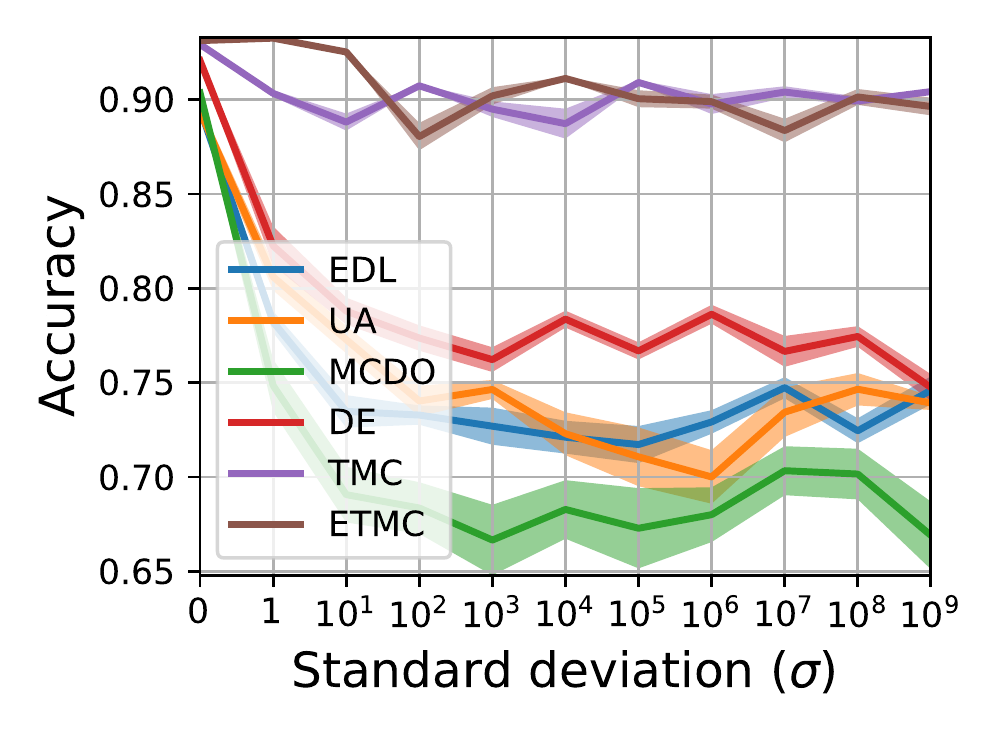}
\end{minipage}}
\centering
\centering
\subfigure[Scene15]{
\begin{minipage}[t]{0.3\linewidth}
\centering
\includegraphics[width=1\linewidth,height=0.7\linewidth]{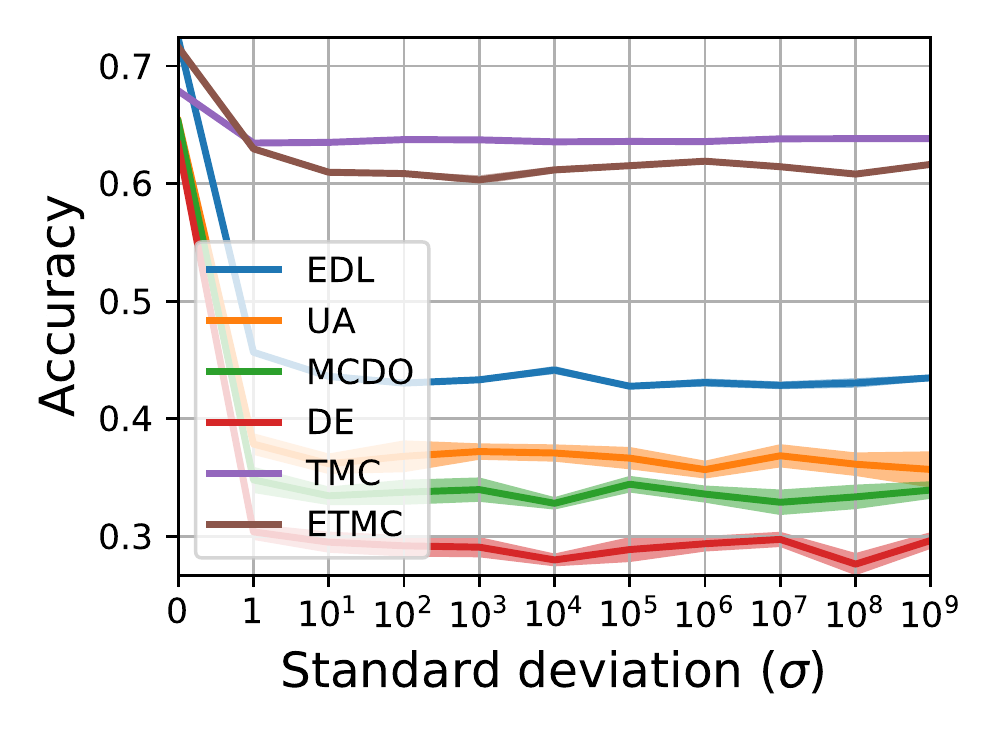}
\end{minipage}}
\subfigure[HMDB]{
\begin{minipage}[t]{0.3\linewidth}
\centering
\includegraphics[width=1\linewidth,height=0.7\linewidth]{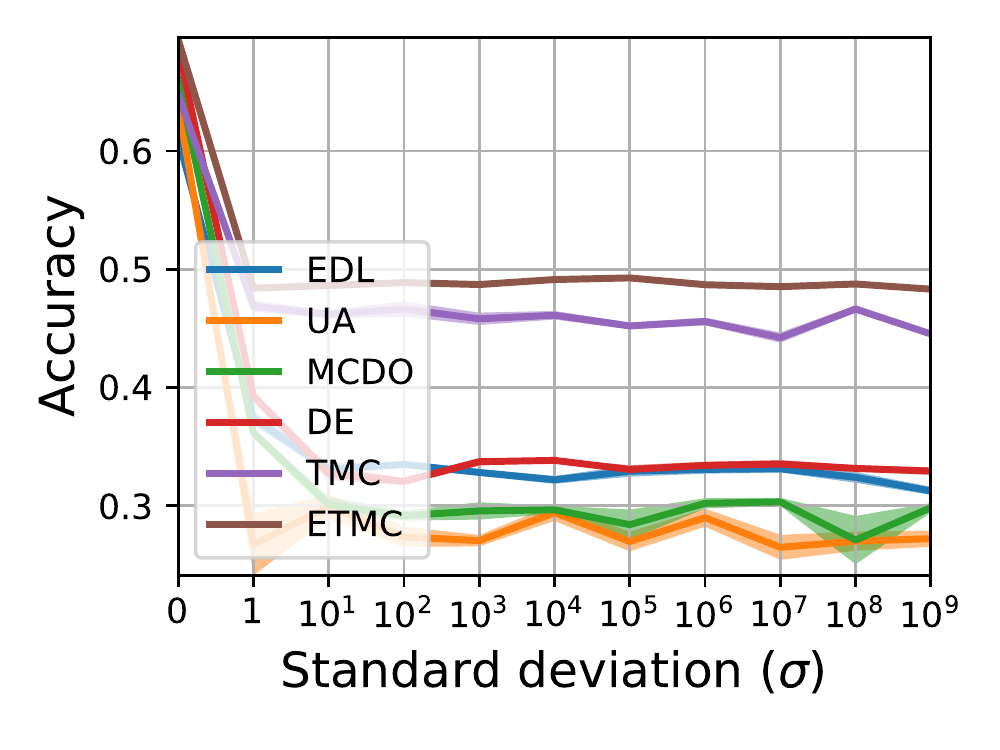}
\end{minipage}}
\centering
\caption{Performance comparison on multi-view data with different levels of noise.}
\label{fig:linechart}
\end{figure*}
First, we conduct experiments on six real-world datasets associated with multiple types of feature vectors, which can be conveniently used to validate our multi-view classification algorithms due to the simplified networks. These datasets are as follows: $\diamond$ \textbf{Handwritten}\footnote{https://archive.ics.uci.edu/ml/datasets/Multiple+Features} consists of 2,000 samples of 10 classes from digits `0' to `9', with 200 samples per class, where six types of descriptors are used. $\diamond$ \textbf{CUB} \cite{wah2011caltech} consists of 11,788 images associated with text descriptions for 200 different categories of birds. We use the first 10 categories, employing GoogleNet and doc2vec to extract image features and corresponding text features, respectively. $\diamond$ \textbf{Caltech101} \cite{fei2004learning} consists of 8,677 images from 101 classes. We extract two types of deep features with DECAF and VGG19, respectively. $\diamond$ \textbf{PIE}\footnote{http://www.cs.cmu.edu/afs/cs/project/PIE/MultiPie/Multi-Pie/Home.html} consists of 680 facial images of 68 subjects. Three types of features including intensity, LBP and Gabor, are extracted. $\diamond$ \textbf{Scene15}\cite{fei2005bayesian} consists of 4,485 images from 15 indoor and outdoor scene categories. Three types features (GIST, PHOG and LBP) are extracted. $\diamond$ \textbf{HMDB}\cite{kuehne2011hmdb} is one of the largest human action recognition datasets. It contains 6,718 samples from 51 categories of actions, where HOG and MBH features are extracted.

We compare our method with the following uncertainty based models: (a) MCDO (Monte Carlo dropout) \cite{gal2016dropout} casts dropout network training as approximate inference in a Bayesian neural network; (b) DE (deep ensemble) \cite{lakshminarayanan2017simple} is a simple, non-Bayesian method, which involves training multiple deep models; (c) UA (uncertainty-aware attention) \cite{heo2018uncertainty} generates attention weights following a Gaussian distribution with a learned mean and variance, which allows  heteroscedastic uncertainty to be captured and yields a more accurate calibration of prediction uncertainty; (d) EDL (evidential deep learning) \cite{sensoy2018evidential} designs a predictive distribution for classification by placing a Dirichlet distribution on the class probabilities.

For the proposed TMC and ETMC, we employ fully connected networks for all datasets, and $l_2$-norm regularization is used with a trade-off parameter of 0.0001. For ETMC, we concatenate the original feature vectors to construct the pseudo-view. The Adam \cite{kingma2014adam} optimizer is used to train the network, and five-fold cross-validation is employed to tune the hyperparameters. For all datasets, $20\%$ of samples are used as the test set. We run all methods 30 times and report the mean values and standard deviations.

\begin{figure*}[!t]
\centering
\subfigure[Handwritten]{
\centering
\begin{minipage}[t]{0.3\linewidth}
\centering
\includegraphics[width=0.9\linewidth,height=0.744\linewidth]{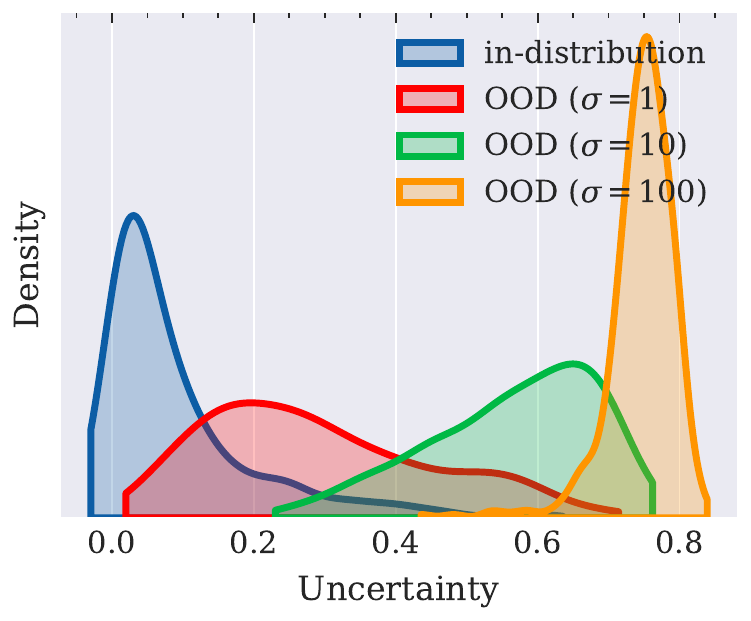}
\centering
\end{minipage}}
\subfigure[CUB]{
\begin{minipage}[t]{0.3\linewidth}
\centering
\includegraphics[width=0.9\linewidth,height=0.744\linewidth]{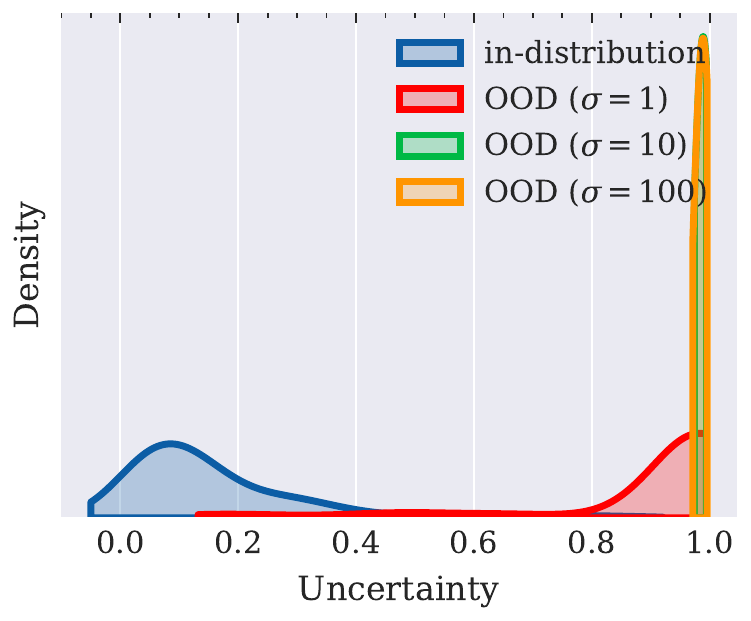}
\end{minipage}}
\centering
\subfigure[PIE]{
\begin{minipage}[t]{0.3\linewidth}
\centering
\includegraphics[width=0.9\linewidth,height=0.744\linewidth]{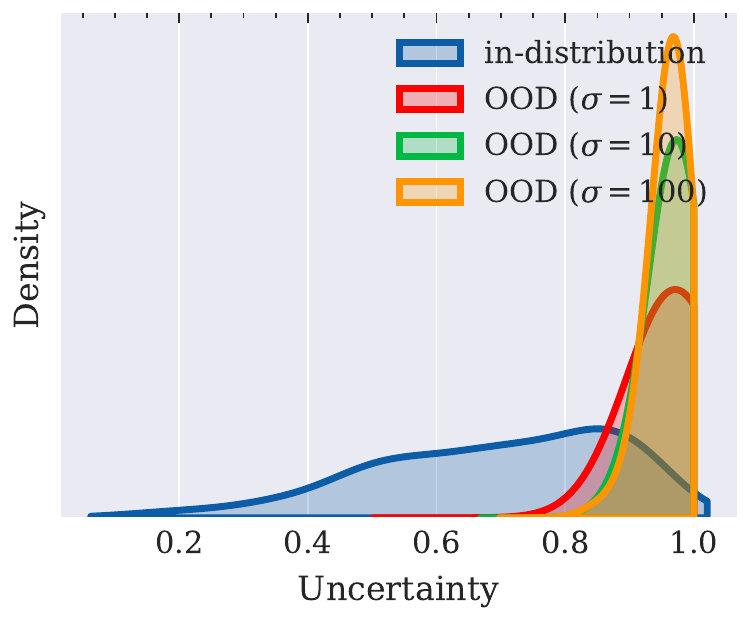}
\end{minipage}}
\centering
\subfigure[Caltech101]{
\begin{minipage}[t]{0.3\linewidth}
\centering
\includegraphics[width=0.9\linewidth,height=0.744\linewidth]{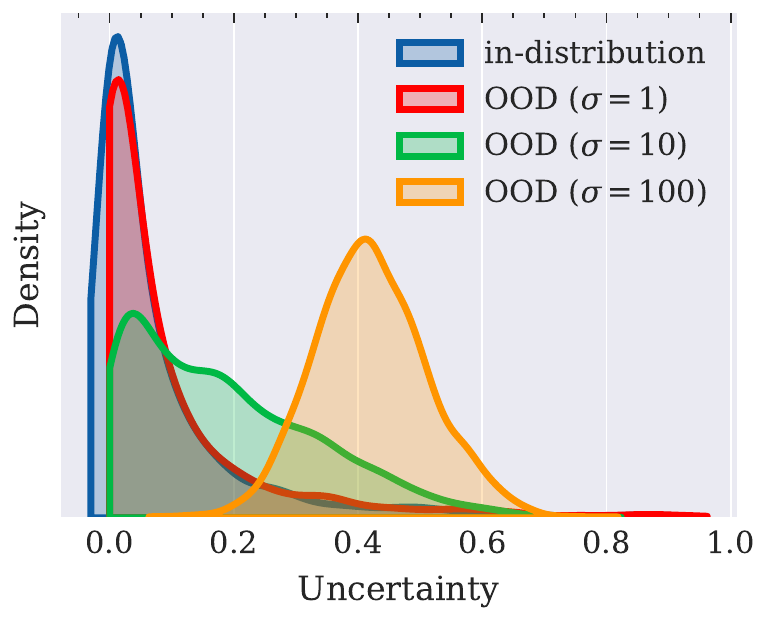}
\end{minipage}}
\centering
\subfigure[Scene15]{
\begin{minipage}[t]{0.3\linewidth}
\centering
\includegraphics[width=0.9\linewidth,height=0.744\linewidth]{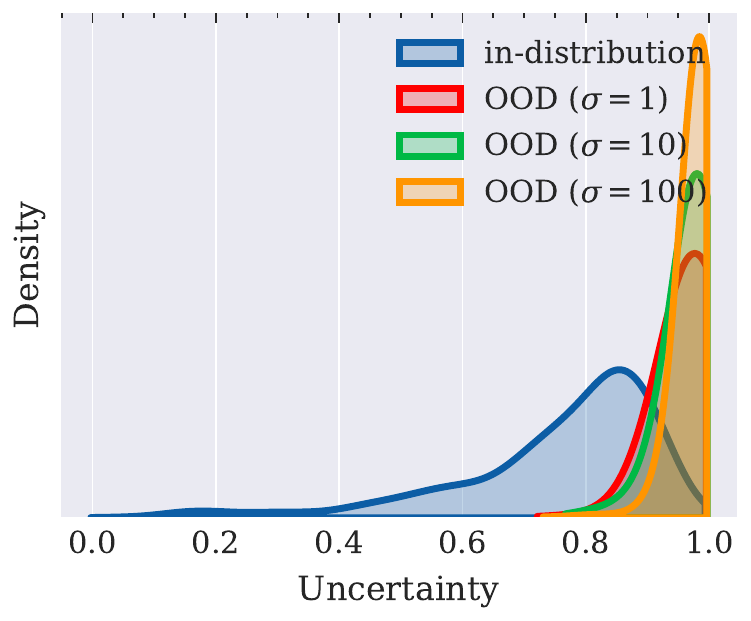}
\end{minipage}}
\centering
\subfigure[HMDB]{
\begin{minipage}[t]{0.3\linewidth}
\centering
\includegraphics[width=0.9\linewidth,height=0.744\linewidth]{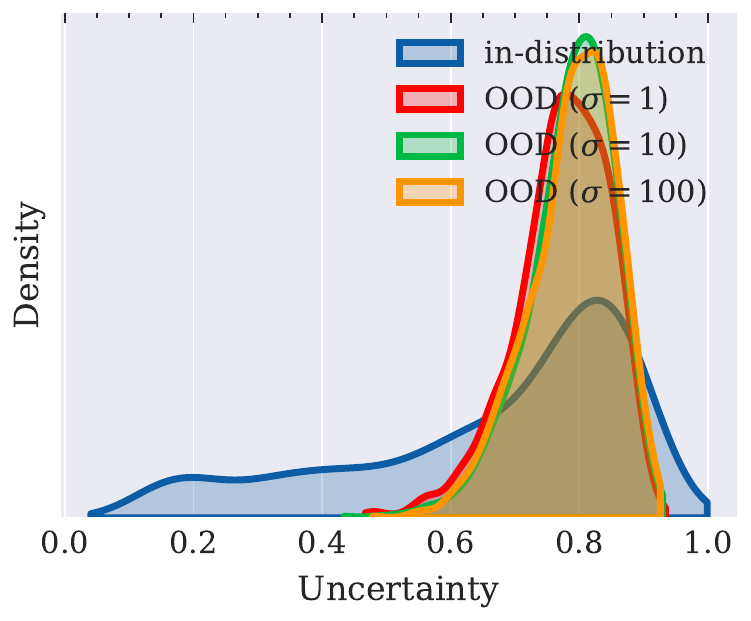}
\end{minipage}}
\caption{Density of uncertainty obtained using TMC algorithm.}
\label{fig:noisedensity}
\end{figure*}

\textbf{Comparison with Uncertainty-Based Methods Using the Best View.} First, we compare our method with existing uncertainty-based classification models, and report the experimental results in Table~\ref{tab:addlabel}. Existing uncertainty-based methods are designed for single-view data, so we report their results using the best-performing view. As shown in Table~\ref{tab:addlabel}, our method outperforms all models on all datasets in terms of accuracy. Taking the results on PIE and Scene15 as examples, our ETMC improves the accuracy by about 9.4\% and 18.7\% compared to the second-best models (EDL/MCDO). Compared with TMC, benefiting from the pseudo view, ETMC usually achieves better performance. Although the proposed model is more effective than single-view uncertainty-based methods, we will further investigate the performance when all algorithms utilize multiple views.

\begin{figure*}[!htbp]
\centering
\subfigure[Handwritten]{
\centering
\begin{minipage}[t]{0.3\linewidth}
\centering
\includegraphics[width=0.9\linewidth,height=0.744\linewidth]{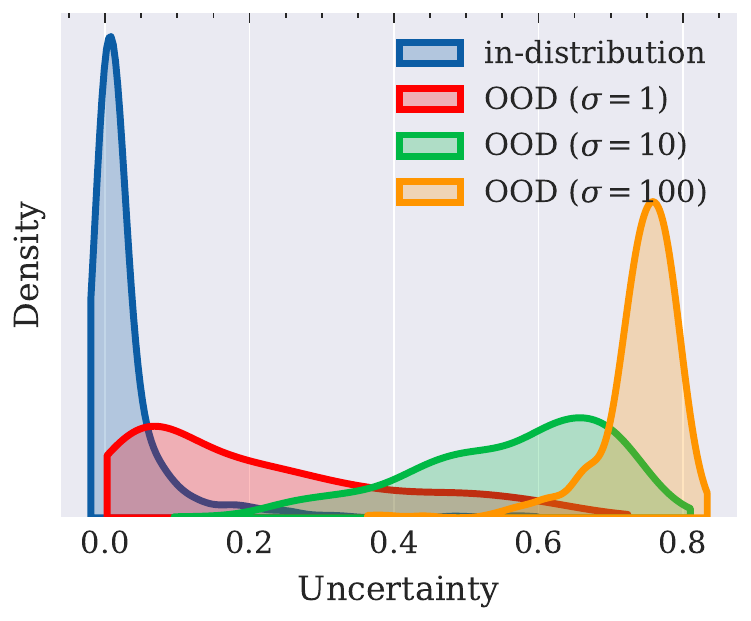}
\centering
\end{minipage}}
\subfigure[CUB]{
\begin{minipage}[t]{0.3\linewidth}
\centering
\includegraphics[width=0.9\linewidth,height=0.744\linewidth]{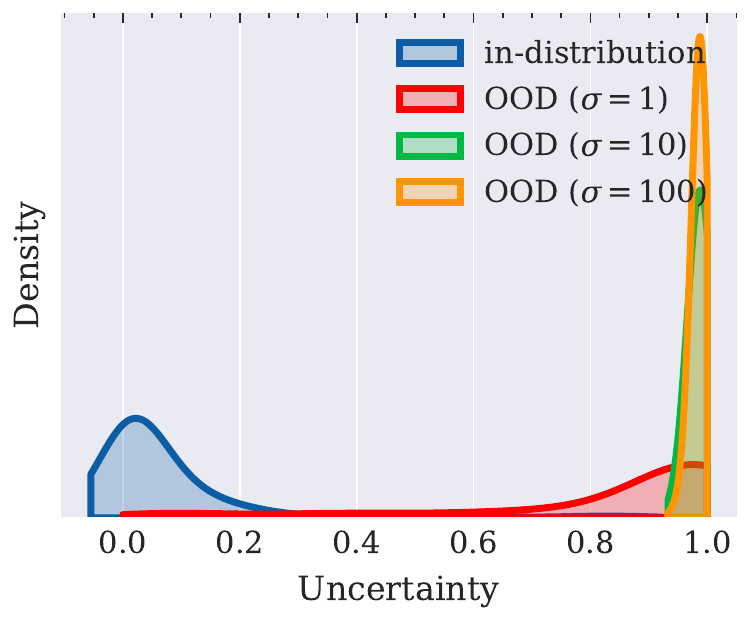}
\end{minipage}}
\centering
\subfigure[PIE]{
\begin{minipage}[t]{0.3\linewidth}
\centering
\includegraphics[width=0.9\linewidth,height=0.744\linewidth]{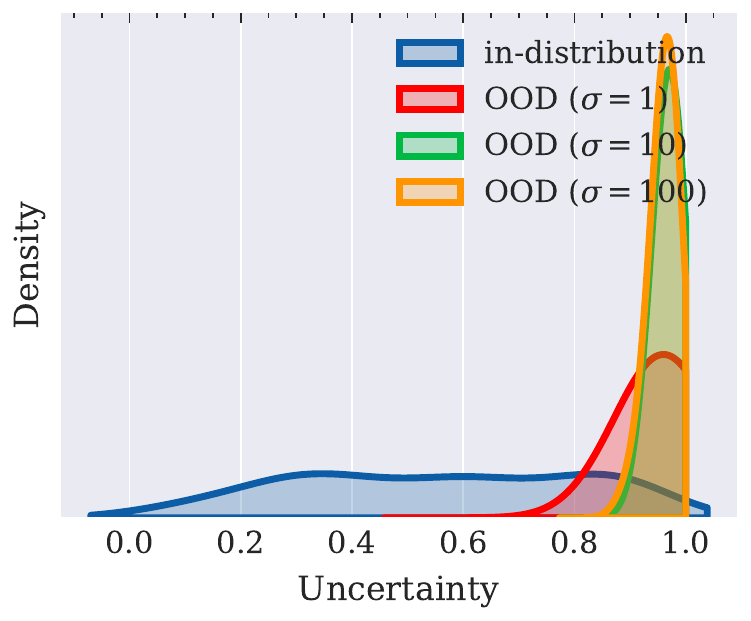}
\end{minipage}}
\centering
\subfigure[Caltech101]{
\begin{minipage}[t]{0.3\linewidth}
\centering
\includegraphics[width=0.9\linewidth,height=0.744\linewidth]{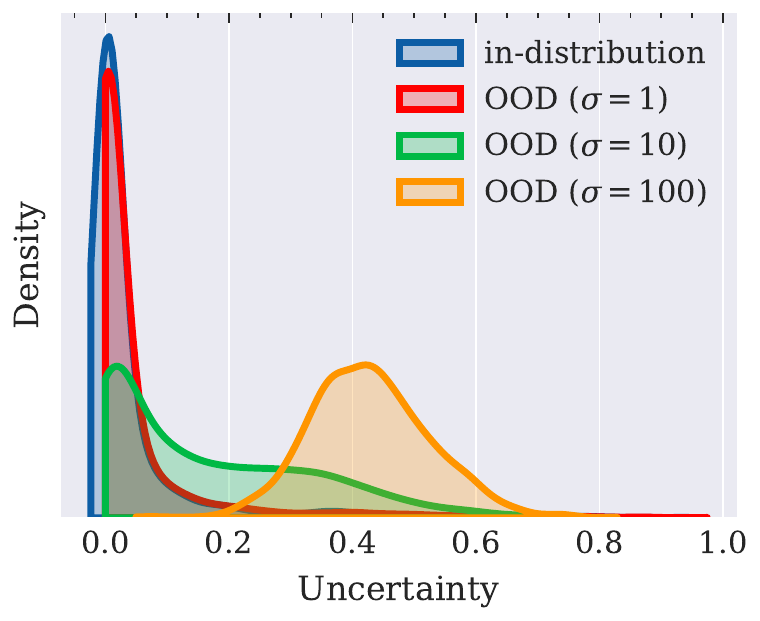}
\end{minipage}}
\centering
\subfigure[Scene15]{
\begin{minipage}[t]{0.3\linewidth}
\centering
\includegraphics[width=0.9\linewidth,height=0.744\linewidth]{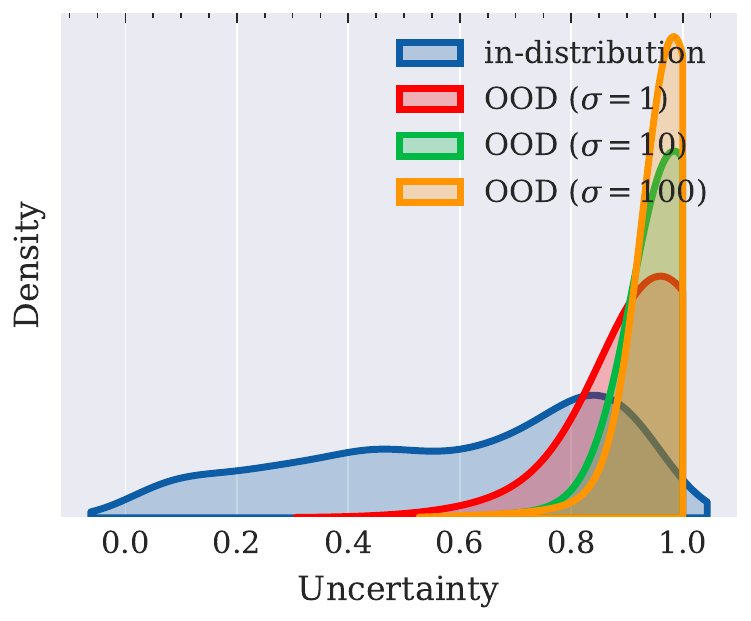}
\end{minipage}}
\centering
\subfigure[HMDB]{
\begin{minipage}[t]{0.3\linewidth}
\centering
\includegraphics[width=0.9\linewidth,height=0.744\linewidth]{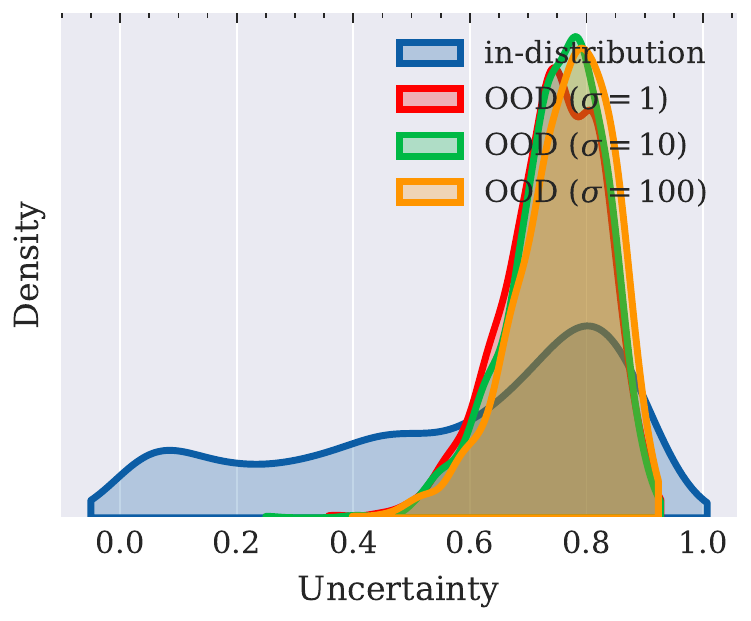}
\end{minipage}}
\caption{Density of uncertainty obtained using the ETMC algorithm.}
\label{fig:noisedensity_ETMC}
\end{figure*}

\textbf{Comparison with Uncertainty-Based Algorithms Using Multiple Views.} To further demonstrate the effectiveness our model in integrating different views, we concatenate the original feature vectors for all comparison methods. We add Gaussian noise with different standard deviations ($\sigma$) to the test data. More specifically, we add noise to one view for all datasets except for the Handwritten dataset.  Handwritten has six views, and adding noise to only one of them does not significantly reduce the overall performance.
Therefore, noise is added to three views to better evaluate the effectiveness of our algorithm. The results are reported in Fig.~\ref{fig:linechart}. We find that when the data is free of noise, our method can achieve competitive results. Meanwhile for the data with noise, the performance significantly decreases for all comparison methods. Fortunately, benefiting from the uncertainty-based fusion strategy, our method is aware of the view-specific noise and thus achieves promising results on all datasets.  In addition, we conduct experiments after adding different level of noise to each view on all datasets. From Fig.~\ref{fig:linechart}, we have the following observations. The proposed method is quite robust to abnormal multi-view data. Thanks to the uncertainty-based integration, the negative effect of noisy views on the final prediction is limited. Similarly, compared with TMC, although the integrated pseudo-view is noisy, ETMC still achieves competitive performance. However, it will be more convincing to explicitly investigate the performance in uncertainty estimation.

\textbf{Uncertainty Estimation.} To evaluate the estimated uncertainty, we visualize the distribution of in-/out-of-distribution samples in terms of uncertainty. The original samples and the samples with Gaussian noise are considered as in-distribution and out-of-distribution samples respectively. Specifically, we add Gaussian noise with different standard deviations (\emph{i.e.}, $\sigma = 1, \sigma = 10$ and $\sigma = 100$) to the test samples. The experimental results of TMC and ETMC are shown in Fig.~\ref{fig:noisedensity} and Fig.~\ref{fig:noisedensity_ETMC}, 
respectively. From the results, we draw the following observations: (1) Datasets with higher classification accuracy (\emph{e.g.}, Handwritten) are usually associated with lower uncertainty for the in-distribution samples. (2) In contrast, datasets with lower accuracy are usually associated with higher uncertainty for the in-distribution samples. (3) Much higher uncertainties are usually estimated for out-of-distribution samples on all datasets. Meanwhile, as the noise strength of the out of distribution data increases, the uncertainty of the data will also increase usually. These observations suggest that our model is effective at characterizing uncertainty, since it can facilitate discrimination between these data. Fig.~\ref{fig:threshold} shows that TMC and ETMC provide much more accurate predictions as the prediction uncertainty decreases. This implies that trusted decisions are supported based on the output (classification and its corresponding uncertainty) of our model.
\begin{table*}[!htbp]
\small
 \centering
 {
 \caption{Comparison with CCA-based algorithms. The best two results are in bold and marked with a superscript.}
   \begin{tabular}{ccccccc}
   \toprule
   Data & CCA   & DCCA  & DCCAE & BCCA  & TMC & ETMC\\
   \midrule
   {Handwritten} & 97.25$\pm$0.01 & 97.55$\pm$0.38 & 97.35$\pm$0.35 & 95.75$\pm$1.23 & \textbf{98.51$\pm$0.13}$^2$ & \textbf{98.75$\pm$0.00}$^1$\\
   \midrule
   {CUB} & 85.83$\pm$1.97 & 82.00$\pm$3.15 & 85.50$\pm$1.39 & 77.67$\pm$2.99 & \textbf{90.83$\pm$3.23}$^2$ & \textbf{91.04$\pm$0.69}$^1$\\
   \midrule
   {PIE} & 80.88$\pm$0.95 & 80.59$\pm$1.52 & 82.35$\pm$1.04 & 76.42$\pm$1.37 & \textbf{91.91$\pm$0.11}$^2$ & \textbf{93.75$\pm$1.08}$^1$\\
   \midrule
   {Caltech101} & 90.50$\pm$0.00 & 88.84$\pm$0.41 & 89.97$\pm$0.41 & 88.11$\pm$0.40 & \textbf{92.93$\pm$0.20}$^1$ & \textbf{92.83$\pm$0.33}$^2$\\
   \midrule
   {Scene15} & 55.77$\pm$0.22 & 54.85$\pm$1.00 & 55.03$\pm$0.34 & 53.82$\pm$0.24 & \textbf{67.74$\pm$0.36}$^2$ & \textbf{71.61$\pm$0.28}$^1$\\
   \midrule
   {HMDB} & 54.34$\pm$0.75 & 46.73$\pm$0.97 & 49.16$\pm$1.07 & 49.12$\pm$1.01 & \textbf{65.26$\pm$2.52}$^2$ & \textbf{69.43$\pm$0.67}$^1$\\
   \bottomrule
   \end{tabular}%
    \label{tab:cca_result}%
   }
\end{table*}%
\begin{figure}
\centering
\subfigure[TMC]{\begin{minipage}[t]{0.49\linewidth}
\includegraphics[width=1\linewidth,height=0.75\linewidth]{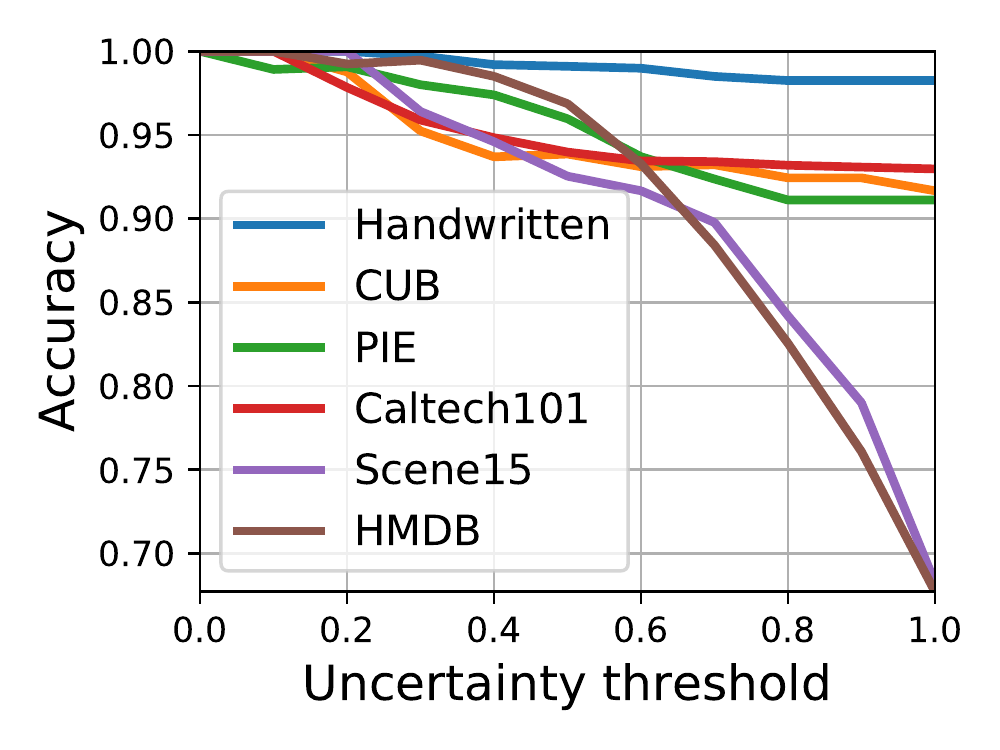}
\end{minipage}}
\subfigure[ETMC]{\begin{minipage}[t]{0.49\linewidth}
\includegraphics[width=1\linewidth,height=0.75\linewidth]{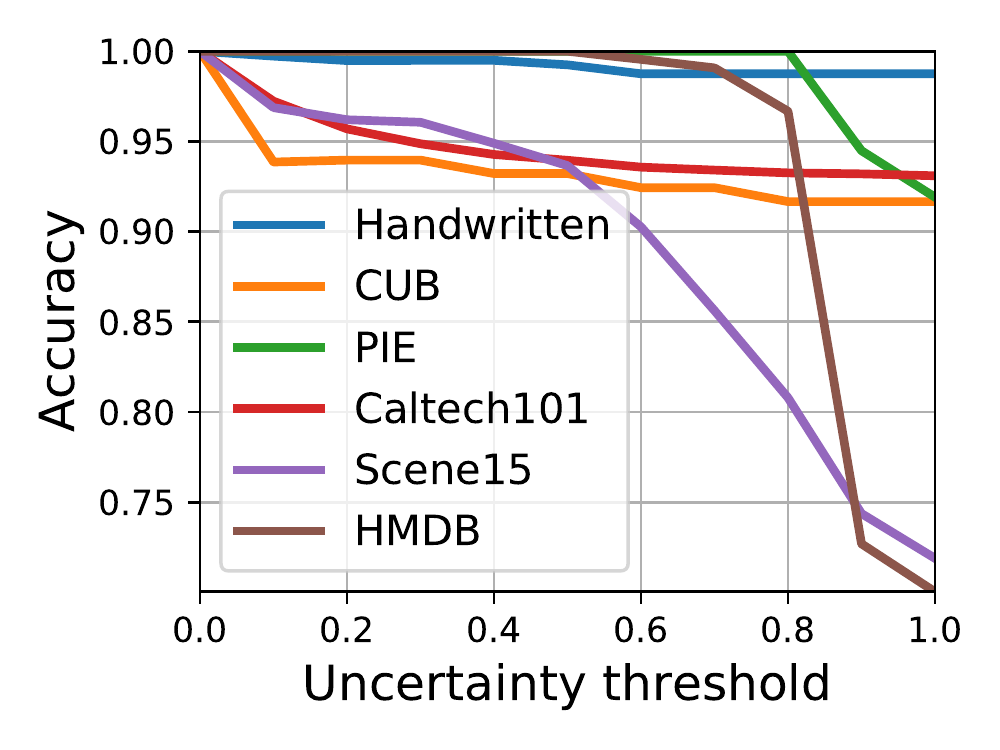}
\end{minipage}}
\caption{Accuracy with uncertainty thresholding.}
\centering
\label{fig:threshold}
\end{figure}

\textbf{Comparison with CCA-Based Algorithms}. We also compare our method with CCA-based algorithms. Specifically, we employ CCA-based methods to obtain the latent representations and then the linear SVM classifier is used for classification. The following CCA-based algorithms are used as baselines. (1) CCA \cite{hotelling1992relations} is a classical algorithm for seeking the correlations between multiple types of features. (2) DCCA \cite{andrew2013deep} obtains the correlations through DNNs. (3) DCCAE \cite{wang2015deep} employs autoencoders to seek the common representations. (4) BCCA \cite{wang2007variational} uses a Bayesian model selection algorithm for CCA based on a probabilistic interpretation. Due to the randomness (e.g., training/test partition and optimization) involved, we run each method 30 times and report their mean and standard deviations in terms of classification accuracy. The experimental results are shown in Table \ref{tab:cca_result}. On all datasets, our method consistently achieves better performance compared with the CCA-based algorithms. We notice that the CCA outperforms CCA-advanced methods, the possible reason is that stronger correlation maximization algorithms are easy to overfit on the training set.

\begin{table}[!h]
  \centering
  \caption{Comparison with state-of-the-art methods on SUN RGB-D dataset (in terms of classification accuracy). The best two results are in bold and marked with a superscript.}
    \begin{tabular}{llll}
    \toprule
    Methods & \multicolumn{1}{l}{RGB} & \multicolumn{1}{l}{Depth} & \multicolumn{1}{l}{Fusion} \\
    \midrule
    Wang et al.\cite{wang2016modality} & 40.4  & 36.5  & 48.1 \\
    Du et al.\cite{du2018depth} & 42.6  & 43.3  & 53.3 \\
    Li et al.\cite{li2018df} & 46.3  & 39.2  & 54.6 \\
    Song et al.\cite{song2018learning} & 44.6  & 42.7  & 53.8 \\
    Yuan et al.\cite{yuan2019acm} & 45.7  & - & 55.1 \\
    Du et al.\cite{du2019translate} & 50.6  & 47.9  & 56.7 \\
    Ayub et al.\cite{Ayub_2020_BMVC}  & 48.8  & 37.3  & 59.5 \\
    Ours (TMC) &\textbf{57.1}$^1$ &\textbf{52.6}$^2$ &\textbf{60.7}$^2$\\
    Ours (ETMC) & \textbf{56.2}$^2$ & \textbf{52.7}$^1$ & \textbf{61.3}$^1$\\
    \bottomrule
    \end{tabular}%
  \label{tab:sunrgbd}%
\end{table}%
\begin{table}[!h]
  \centering
  \caption{Comparison with state-of-the-art methods on the NYUD Depth V2 Dataset (in terms of classification accuracy). ``Aug" in the table indicates that the network fine-tuned on the SUN RGB-D dataset is employed as the backbone. The best two results are in bold and marked with a superscript.}
    \begin{tabular}{llll}
    \toprule
    Methods & RGB   & Depth & Fusion \\
    \midrule
    Wang et al.\cite{wang2016modality} & 53.5  & 51.5  & 63.9 \\
    Song et al.\cite{song2017combining} & -     & -     & 66.7 \\
    Li et al.\cite{li2018df} & 61.1  & 54.8  & 65.4 \\
    Du et al.\cite{du2018depth} & 53.7  & 59    & 67.5 \\
    Yuan et al.\cite{yuan2019acm} & 55.4  & -     & 67.4 \\
    Song et al.\cite{song2018learning} & 53.4  & 56.4  & 67.5 \\
    Du et al.\cite{du2019translate} & 60.2  & 55.2  & 65.5 \\
    Du et al.\cite{du2019translate} (Aug) & 64.8  & 57.7  & 69.2 \\
    Ayub et al.\cite{Ayub_2020_BMVC} & \textbf{66.4}$^1$  & 49.5  & \textbf{70.9}$^2$ \\
    Ours (TMC) &65.9 &\textbf{63.8}$^1$&70.0\\
    Ours (ETMC) & \textbf{66.2}$^2$ & \textbf{63.3}$^2$ & \textbf{72.5}$^1$\\
    \bottomrule
    \end{tabular}%
  \label{tab:nyudv2}%
\end{table}%
\subsection{RGB-D Scene Recognition}
\begin{figure*}[!htbp]
\centering
\subfigure[\textcolor{black}{ NYUD Depth V2}]{\begin{minipage}[t]{0.40\linewidth}
\includegraphics[width=0.9\linewidth,height=0.75\linewidth]{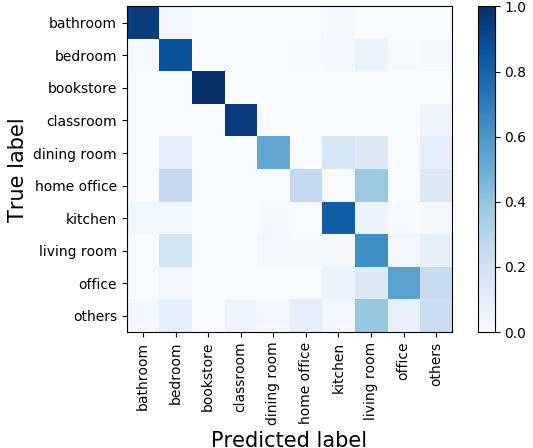}
\end{minipage}}
\subfigure[\textcolor{black}{ SUN RGB-D }]{\begin{minipage}[t]{0.40\linewidth}
\includegraphics[width=0.9\linewidth,height=0.75\linewidth]{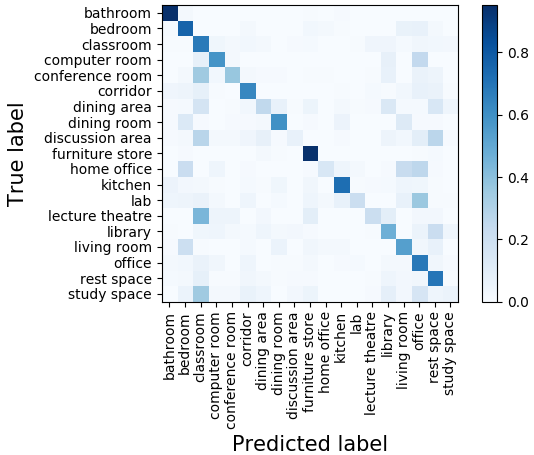}
\end{minipage}}
\caption{\textcolor{black}{Confusion matrices of NYUD Depth V2 and SUN RGB-D datasets.}}
\centering
\label{fig:cmnyud}
\end{figure*}

\begin{figure*}[!htbp]
\centering
\subfigure[\textcolor{black}{NYUD Depth V2}]{\begin{minipage}[t]{0.40\linewidth}
\includegraphics[width=0.9\linewidth,height=0.75\linewidth]{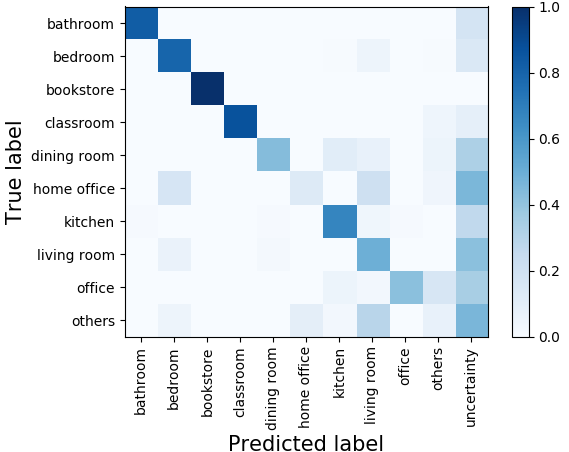}
\end{minipage}}
\subfigure[\textcolor{black}{ SUN RGB-D}]{\begin{minipage}[t]{0.40\linewidth}
\includegraphics[width=0.9\linewidth,height=0.75\linewidth]{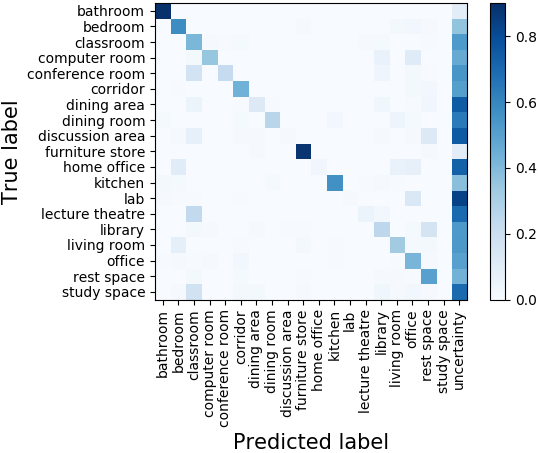}
\end{minipage}}
\caption{\textcolor{black}{Subjective confusion matrices of NYUD Depth V2 and SUN RGB-D datasets.}}
\centering
\label{fig:sun}
\end{figure*}

\begin{figure*}[!h]
\centering
\subfigure[Confident sample 1 (Furniture store)]{
\centering
\begin{minipage}[t]{0.30\linewidth}
\centering
\includegraphics[width=0.9\linewidth,height=0.45\linewidth]{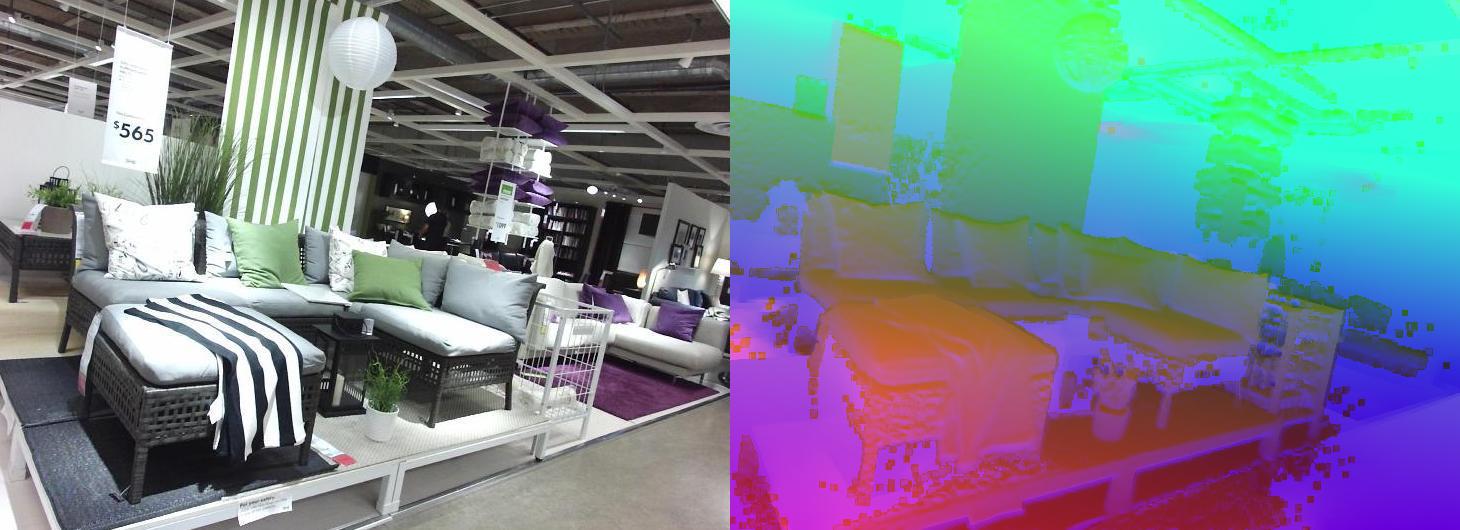}
\centering
\end{minipage}}
\subfigure[Confident sample 2 (Bathroom)]{
\centering
\begin{minipage}[t]{0.30\linewidth}
\centering
\includegraphics[width=0.9\linewidth,height=0.45\linewidth]{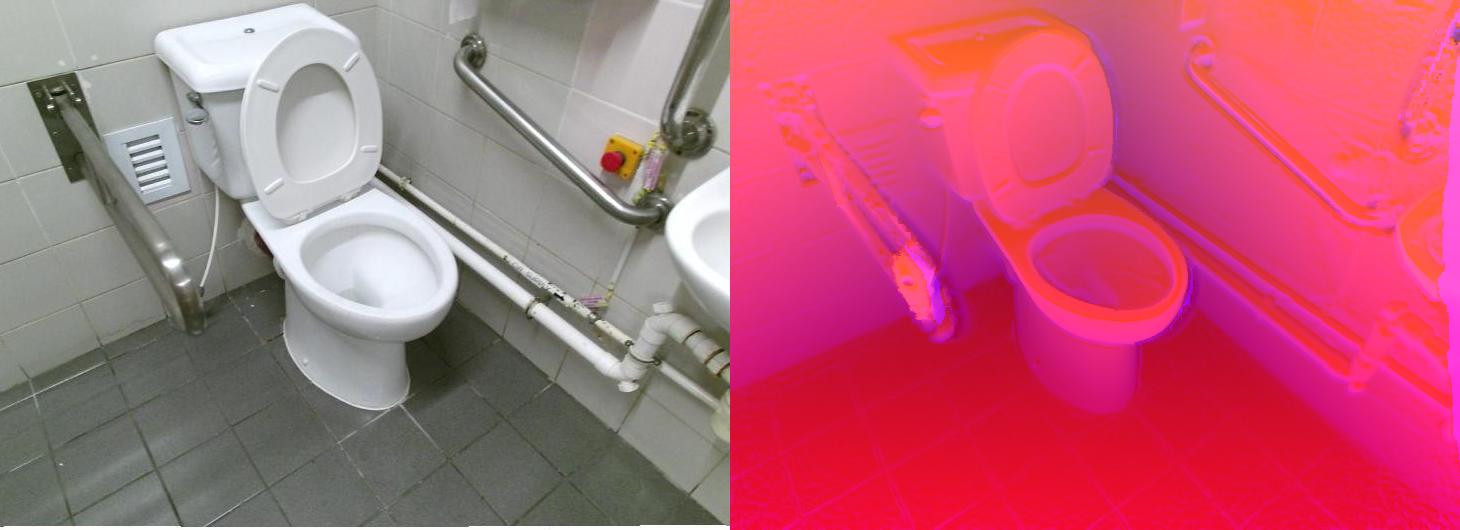}
\centering
\end{minipage}}
\subfigure[Confident sample 3 (Furniture store)]{
\centering
\begin{minipage}[t]{0.30\linewidth}
\centering
\includegraphics[width=0.9\linewidth,height=0.45\linewidth]{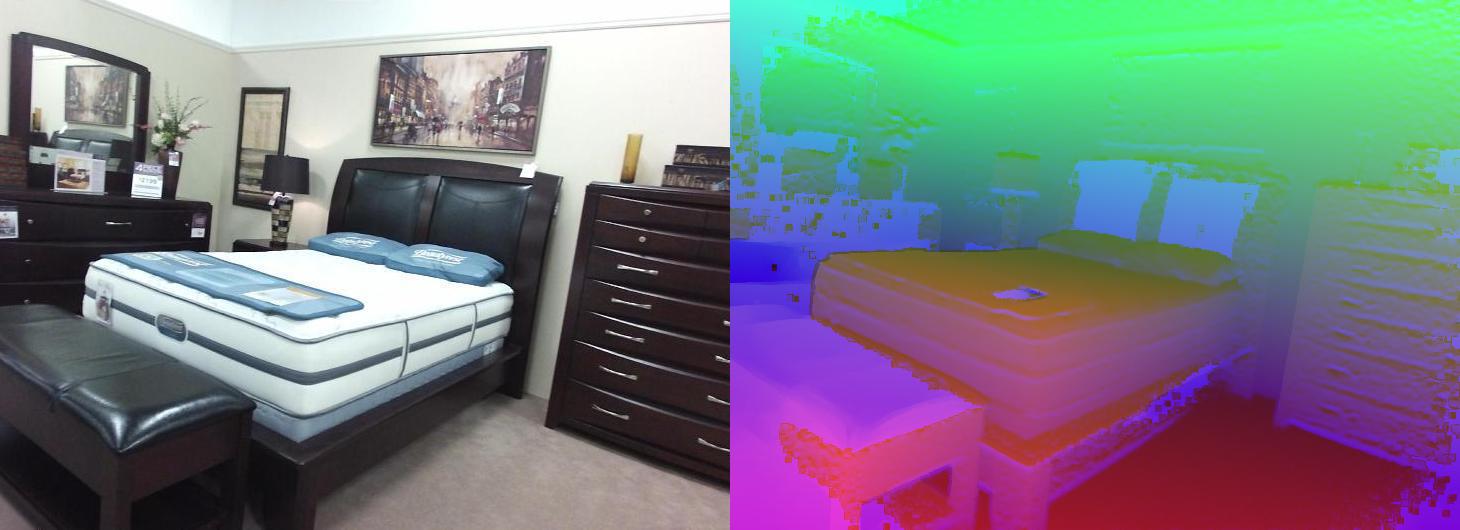}
\centering
\end{minipage}}
\centering
\subfigure[Uncertain sample 1 (Discussion area)]{
\centering
\begin{minipage}[t]{0.30\linewidth}
\centering
\includegraphics[width=0.9\linewidth,height=0.45\linewidth]{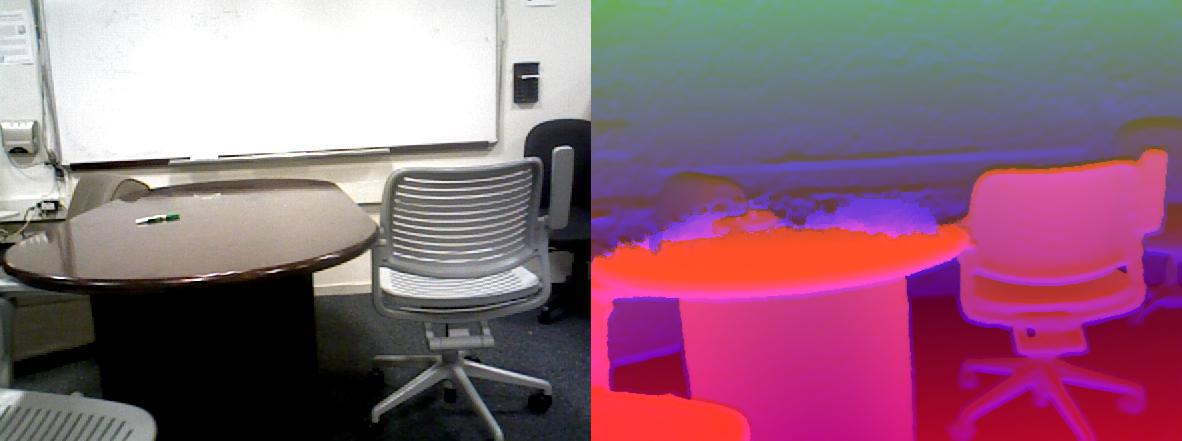}
\centering
\end{minipage}}
\subfigure[Uncertain sample 2 (Office)]{
\centering
\begin{minipage}[t]{0.30\linewidth}
\centering
\includegraphics[width=0.9\linewidth,height=0.45\linewidth]{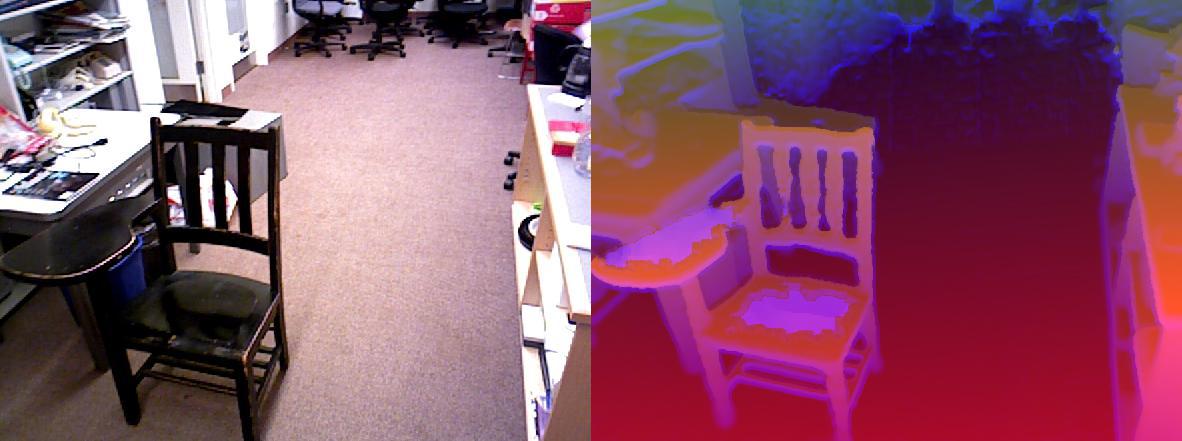}
\centering
\end{minipage}}
\subfigure[Uncertain sample  3 (Rest space)]{
\centering
\begin{minipage}[t]{0.30\linewidth}
\centering
\includegraphics[width=0.9\linewidth,height=0.45\linewidth]{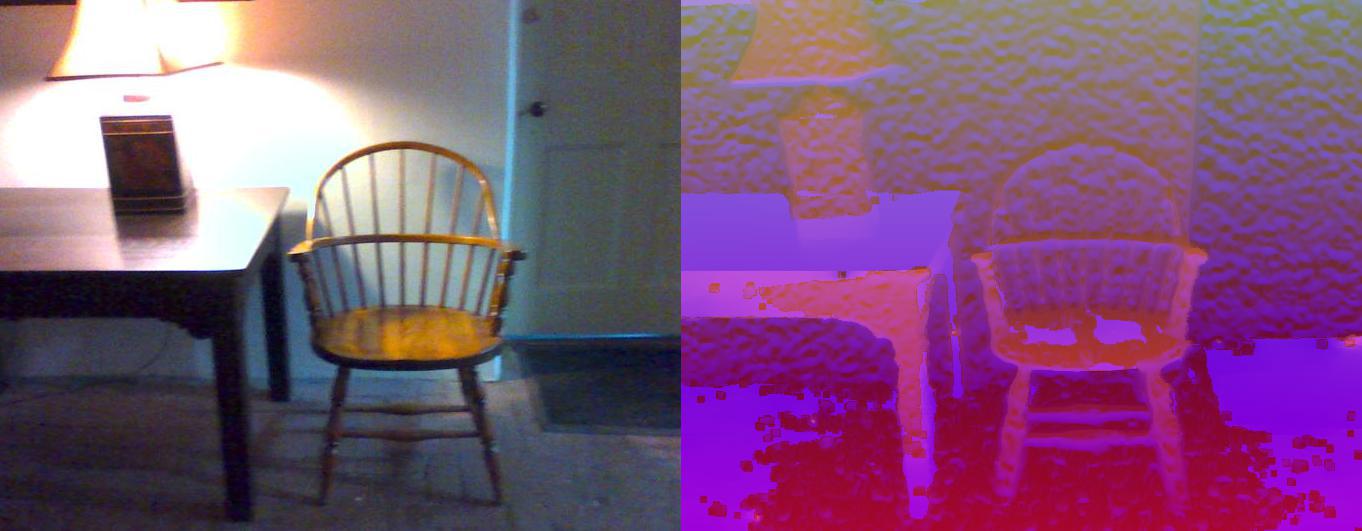}
\centering
\end{minipage}}
\caption{Examples with the highest (bottom row) and lowest (top row) prediction uncertainty using the ETMC algorithm on the SUN RGB-D test set. The ground-truth labels are in brackets. The confident samples are correctly classified, while for the uncertain samples only the rightmost sample is correctly recognized.}
\label{fig:ETMC_SUN_vis}
\end{figure*}
\begin{figure*}[!h]
\centering
\subfigure[Confident sample 1 (Bedroom)]{
\centering
\begin{minipage}[t]{0.30\linewidth}
\centering
\includegraphics[width=0.9\linewidth,height=0.45\linewidth]{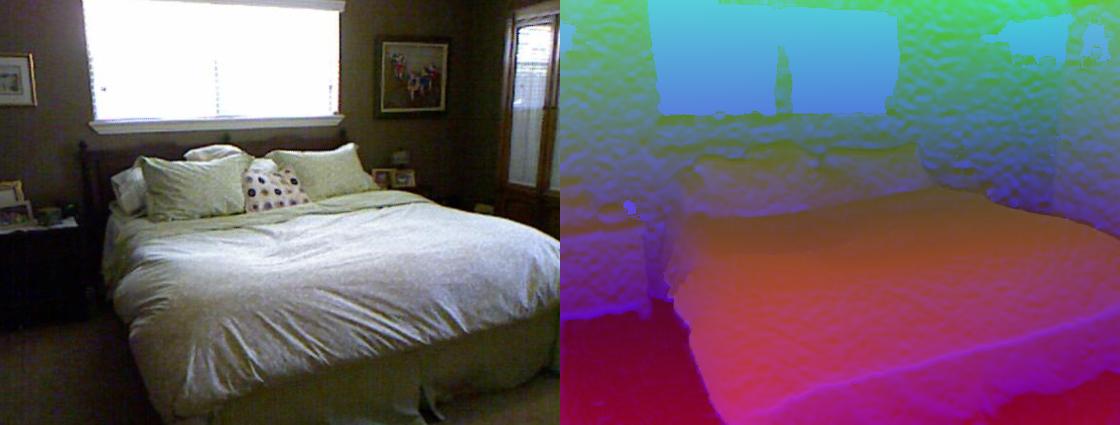}
\centering
\end{minipage}}
\subfigure[Confident sample 2 (Bedroom)]{
\centering
\begin{minipage}[t]{0.30\linewidth}
\centering
\includegraphics[width=0.9\linewidth,height=0.45\linewidth]{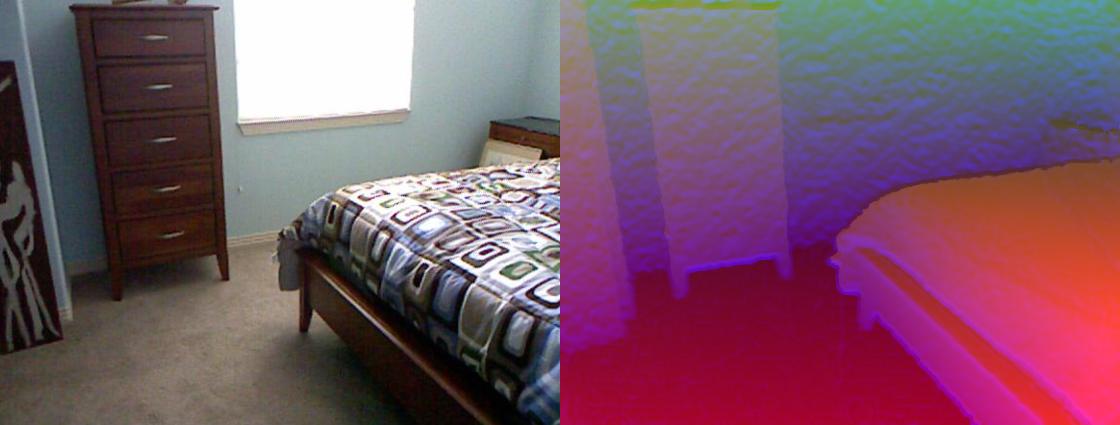}
\centering
\end{minipage}}
\subfigure[Confident sample 3 (Bedroom)]{
\centering
\begin{minipage}[t]{0.30\linewidth}
\centering
\includegraphics[width=0.9\linewidth,height=0.45\linewidth]{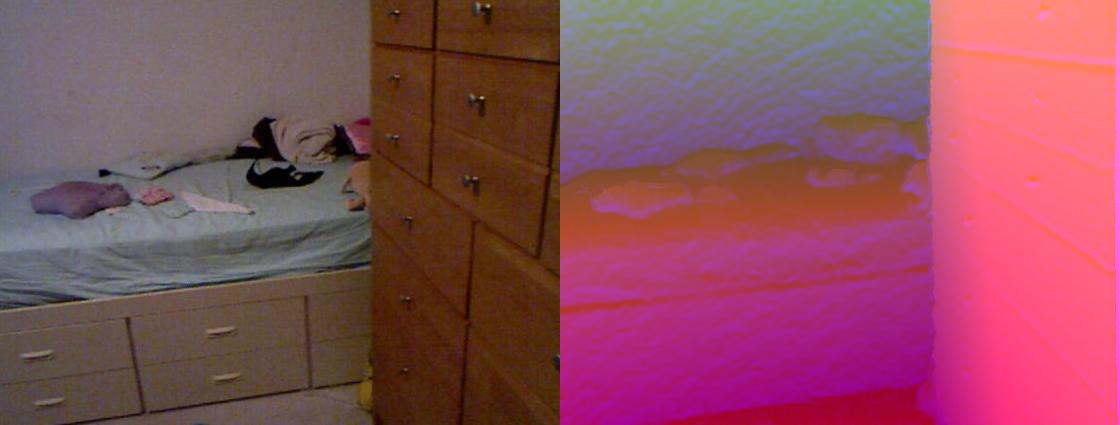}
\centering
\end{minipage}}
\centering
\subfigure[Uncertain sample 1 (Bedroom)]{
\centering
\begin{minipage}[t]{0.30\linewidth}
\centering
\includegraphics[width=0.9\linewidth,height=0.45\linewidth]{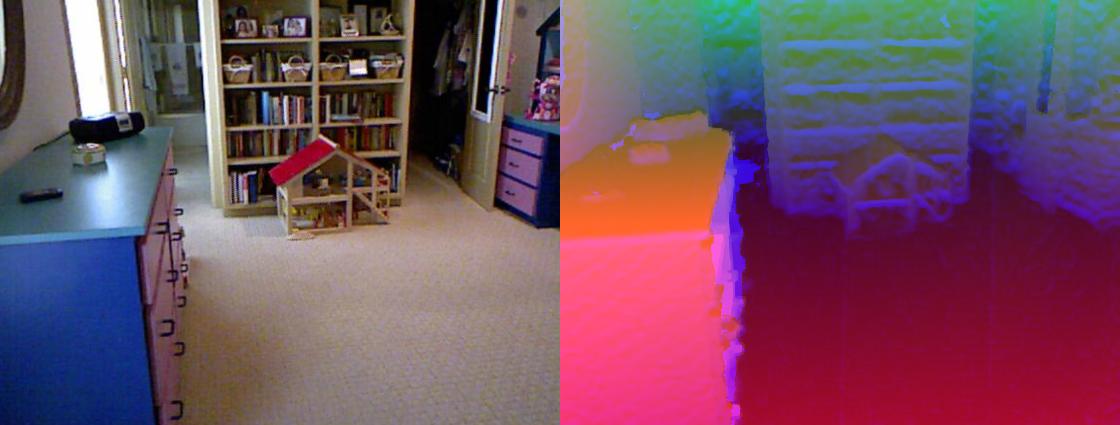}
\centering
\end{minipage}}
\subfigure[Uncertain sample 2 (Others)]{
\centering
\begin{minipage}[t]{0.30\linewidth}
\centering
\includegraphics[width=0.9\linewidth,height=0.45\linewidth]{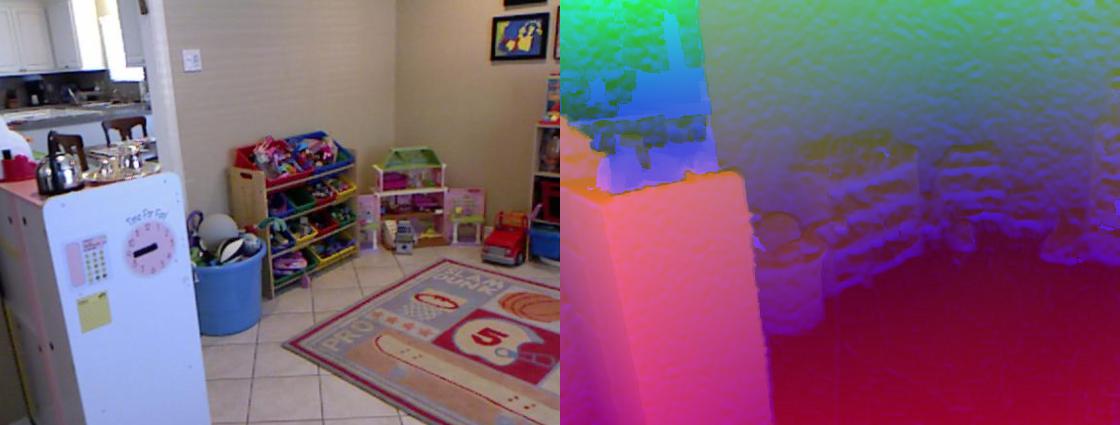}
\centering
\end{minipage}}
\subfigure[Uncertain sample 3 (Kitchen)]{
\centering
\begin{minipage}[t]{0.30\linewidth}
\centering
\includegraphics[width=0.9\linewidth,height=0.45\linewidth]{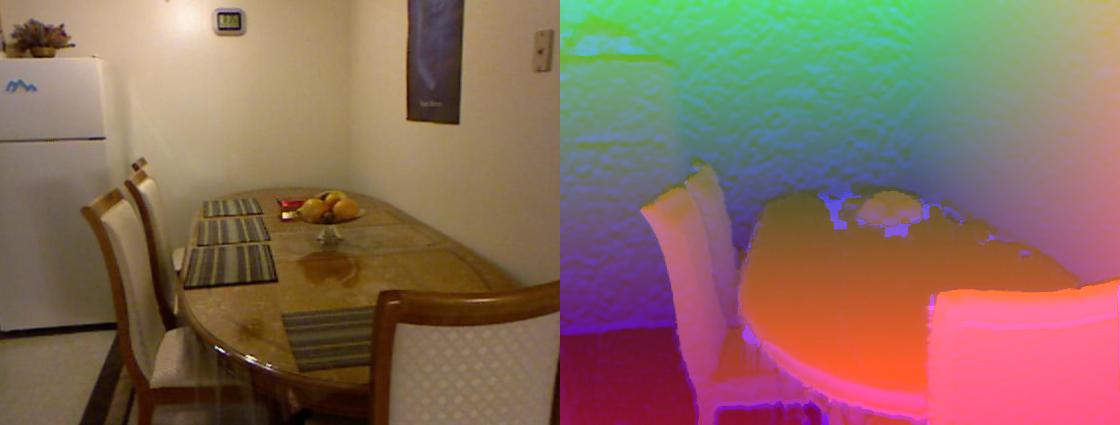}
\centering
\end{minipage}}
\caption{Examples with the highest (bottom row) and lowest (top row) prediction uncertainty using the ETMC algorithm on the NYUD Depth V2 test data. The ground-truth labels are in brackets. The confident samples are correctly classified. For uncertain samples, only the sample 1 is correctly classified.}
\label{fig:ETMC_NYUD_vis}
\end{figure*}
We conduct experiments for RGB-D scene recognition in an end-to-end manner. There are two standard and widely used RGB-D scene datasets, i.e., SUN RGB-D and NYU Depth V2. The SUN RGB-D dataset consists of 10,335 RGB-depth image pairs. Following existing work \cite{song2015sun}, we use the 19 major scene categories, each of which contains at least 80 images. We use 4,845 and 4,659 samples as training and test data, respectively. NYUD Depth V2 (NYUD2) is a relatively small RGB-D dataset, which consists of 1,449 RGB-depth image pairs of 27 categories. Following the standard spilt~\cite{silberman2012indoor}, we reorganize the 27 categories into 10 categories with 9 usual scenes and one "others" category. We then use 795 and 654 samples for training and testing, respectively.

We employ ResNet-18 \cite{he2016deep} pre-trained on ImageNet as the backbone network. A TITAN Xp GPU is used for training the neural network. The input images are resized to $256\times256$ and randomly cropped to $224\times224$. The Adam optimizer \cite{kingma2014adam} with weight and learning rate decay is employed to train the neural network. For ETMC, the pseudo-view is constructed by simply concatenating the outputs of the two backbone networks.

\textbf{Quantitative Results.} We compare the proposed model with the current state-of-the-art methods on SUN RGB-D. The experimental results are shown in Table~\ref{tab:sunrgbd}. For most methods, we report the results for the fusion of RGB and depth, as well as the results using a single view. The proposed ETMC achieves the best accuracy of $\textbf{61.3\%}$. Further, it is still robust and also achieves the best performance when only one view is used in the test stage. This is likely due to the multi-task strategy adopted by our method, which enables it to not only focus on improving the accuracy after fusion, but also on improving the performance when only a single view is used. Compared with TMC, the accuracy of ETMC increases by $0.6\%$ after fusion, and the performance is slightly degraded when only the RGB modality is involved. We conduct comparative experiments with eight current state-of-the-art methods on NYUD Depth V2. The results are shown in Table \ref{tab:nyudv2}. Similar to the experiments on the SUN RGB-D dataset, we report the results using RGB-D image pairs as well as the results using each single view. The proposed ETMC again achieves the best classification performance at $72.5\%$. TMC is competitive when different views 
are used in the test stage. For example, when only the depth image is used, it achieves the highest classification accuracy.

\textbf{Qualitative Results.}
We present typical examples (from SUN RGB-D and NYUD Dpeth V2) with the highest and lowest uncertainties in prediction in Fig.~\ref{fig:ETMC_SUN_vis} and Fig.~\ref{fig:ETMC_NYUD_vis}, respectively. For samples that are easy to intuitively distinguish, our model confidently and correctly recognizes them. For example, Fig.~\ref{fig:ETMC_NYUD_vis}(b) has very discriminative features that support classification. Uncertain samples, on the other hand, are usually difficult for even humans to recognize, and our method also tends to output high uncertainties. Please refer to the appendix for more examples. \textcolor{black}{In addition, we show the classification confusion matrices of NYUD Depth V2 and SUN RGB-D datasets in Fig.~\ref{fig:cmnyud}. It can be seen that different classes have different classification difficulties and the underlying reason is the difference in the number of training samples and the similarity between different classes. Further, to explore the benefits of introducing subjective uncertainty into decision-making, we also present the subjective confusion matrix in Fig.~\ref{fig:sun}.  The subjective confusion matrix is obtained through the following steps. When $u>max\left(\{b_k\}_{k=1}^K\right)$, we consider the final prediction result of the sample to be uncertain. Then we calculate the proportion of each class that is predicted to be in each class and the corresponding uncertainty. According to these subjective confusion matrices, we have the following observations: (1) The proportion of samples classified as uncertain is quite different in different classes, which further validates the different difficulty degrees for different classes. (2) After removing uncertain samples, our method can significantly reduce the misclassification rate (e.g., from 0.387 to 0.187 on SUN RGB-D dataset), indicating that our model can perceive hard-to-recognize classes through subjective uncertainty. We have clarified this in the revision.
}

\textcolor{black}{\textbf{Computational cost analysis.} We conduct experiments on the following models and report the corresponding FLOPs and number of parameters: (1) A single view classifier (SingleView) obtains prediction by using only one view. (2) Late fusion (LateFusion) obtains the final result by averaging the decisions of different views. (3) TRecg\cite{du2019translate} employs view translation to enhance the representation learning ability of the model and performs intermediate fusion. The experimental results are shown in Table \ref{tab:computationalcost}. From the experimental results, it is observed that our model has a similar computational cost to simple Latefusion. At the same time, we also compare ours with TRecg. Due to the involved complex fusion strategies and components, the computation cost of TRecg is much higher than that of ours.}

\begin{table}[!h]
  \centering
  \caption{\textcolor{black}{Computational cost comparison with different methods (in terms of GFLOPs and number of parameters).}}
    \textcolor{black}{
    \begin{tabular}{lll}
    \toprule\
    Method & GFLOPs & Parameters \\
    \midrule
    SingleView & 1.82 & 11.19M \\
    LateFusion & 3.64 & 22.37M \\
    Trecg\cite{du2019translate} & 13.84 & 24.47M \\
    TMC   & 3.64 & 22.37M \\
    ETMC  & 3.64 & 22.39M \\
    \bottomrule \
    \end{tabular}%
    }
  \label{tab:computationalcost}
\end{table}%

\subsection{Image-Text Classification}
We further conduct experiments in an end-to-end manner on image-text classification. We use the UMPC-FOOD101 dataset, which consists of 86,796 samples taken from web pages. Each sample is provided with both an image and text view \cite{wang2015recipe}. Following existing work \cite{kiela2019supervised}, there are 60,101 samples used as the training set, 5,000 samples used as the validation set, and the remaining 21,695 samples as the test set. In this experiment, we use the original datasets, which contain noise due to the collection process. We employ ResNet-152 \cite{he2016deep} pre-trained on ImageNet and BERT \cite{devlin2018bert} as the backbone network for images and text, respectively. For comparison algorithms, we report the results for the best performing views, and concatenate the outputs of the two backbone networks as input to the classifier to simulate the late fusion strategy. For ETMC, we use the same strategy to construct the pseudo-view. Due to the randomness involved, we run all algorithms five times and report the mean accuracy and standard deviation. The Adam \cite{kingma2014adam} optimizer is employed for all methods with learning rate decaying for all methods.

The experimental results on the UMPC-FOOD101 dataset are shown in Table~\ref{tab:food101}. We report the accuracy of each compared algorithm with the best-performing views (`best view' in Table~\ref{tab:food101}) and combined views (`feature fusion' in Table~\ref{tab:food101}). As can be seen, our algorithm outperforms other methods on the Food101 dataset. Our method achieves about $91.5\%$ in terms of accuracy, compared to $90.5\%$ for the second-best model. The performance when using the best single view is relatively low.

\begin{table}[!htbp]
    \centering
    \begin{tabular}{ccc}
    \toprule\
    method & Food101  \\
    \midrule
    {MCDO (best view)} & 74.87$\pm$1.04  \\
    {DE (best view)} & 83.27$\pm$0.36 \\
    {UA (best view)} & 84.49$\pm$0.34 \\
    {EDL (best view)} & 87.40$\pm$0.48 \\
    \midrule
    {MCDO (feature fusion)}  & 79.49$\pm$1.54  \\
    {DE (feature fusion)}    & 86.45$\pm$0.44  \\
    {UA (feature fusion)}    & 88.50$\pm$0.47  \\
    {EDL (feature fusion)}   & 90.50$\pm$0.32  \\
    \midrule
    TMC  & \textbf{91.30$\pm$0.21} \\
    ETMC  & \textbf{91.47$\pm$0.31} \\
    \bottomrule
    \end{tabular}%
    \caption{Evaluation of the classification performance on UMPC-FOOD101.}
    \label{tab:food101}
\end{table}
\section{Conclusion}
In this work, we propose a novel  trusted multi-view classification (TMC) algorithm which, based on the Dempster-Shafer evidence theory, can produce trusted classification decision on multi-view data. Our algorithm focuses on decision-making by fusing the uncertainty of multiple views, which is essential for making reliable multi-view fusion. The TMC model can dynamically identify the views which are risky for decision making, and exploits informative views in the final decision. Furthermore, our model can determine the uncertainty of a current decision when making the final classification, providing explainability. Both theoretical analysis and empirical results validate the effectiveness of the proposed algorithm in classification accuracy and uncertainty estimation.
\ifCLASSOPTIONcompsoc
  \section*{Acknowledgments}
 This work was supported in part by the National Key Research and Development Program of China under Grant 2019YFB2101900, the National Natural Science Foundation of China (61976151, 61925602, 61732011), the A*STAR AI3 HTPO Seed Fund (C211118012), the Joey Tianyi Zhou's A*STAR SERC Central Research Fund and AME Programmatic Funding Scheme (Project No. A18A1b0045). We gratefully acknowledge the support of MindSpore, CANN (Compute Architecture for Neural Networks), and Ascend AI Processor used for this research. We also want to use TMC on MindSpore, which is a new deep learning computing framework. These problems are left for future work.
\else
  \section*{Acknowledgment}
\fi
{
\bibliographystyle{IEEEtran}
\bibliography{bare_jrnl_compsoc}

\begin{thebibliography}{100}
\providecommand{\url}[1]{#1}
\csname url@samestyle\endcsname
\providecommand{\newblock}{\relax}
\providecommand{\bibinfo}[2]{#2}
\providecommand{\BIBentrySTDinterwordspacing}{\spaceskip=0pt\relax}
\providecommand{\BIBentryALTinterwordstretchfactor}{4}
\providecommand{\BIBentryALTinterwordspacing}{\spaceskip=\fontdimen2\font plus
\BIBentryALTinterwordstretchfactor\fontdimen3\font minus
  \fontdimen4\font\relax}
\providecommand{\BIBforeignlanguage}[2]{{%
\expandafter\ifx\csname l@#1\endcsname\relax
\typeout{** WARNING: IEEEtran.bst: No hyphenation pattern has been}%
\typeout{** loaded for the language `#1'. Using the pattern for}%
\typeout{** the default language instead.}%
\else
\language=\csname l@#1\endcsname
\fi
#2}}
\providecommand{\BIBdecl}{\relax}
\BIBdecl

\bibitem{liu2021self}
J.~Liu, X.~Liu, Y.~Zhang, P.~Zhang, W.~Tu, S.~Wang, S.~Zhou, W.~Liang, S.~Wang,
  and Y.~Yang, ``Self-representation subspace clustering for incomplete
  multi-view data,'' in \emph{Proceedings of the ACM International Conference
  on Multimedia}, 2021, pp. 2726--2734.

\bibitem{lin2021completer}
Y.~Lin, Y.~Gou, Z.~Liu, B.~Li, J.~Lv, and X.~Peng, ``Completer: Incomplete
  multi-view clustering via contrastive prediction,'' in \emph{Proceedings of
  the IEEE/CVF Conference on Computer Vision and Pattern Recognition}, 2021,
  pp. 11\,174--11\,183.

\bibitem{zhang2019cpm}
C.~Zhang, Z.~Han, H.~Fu, J.~T. Zhou, Q.~Hu \emph{et~al.}, ``Cpm-nets: Cross
  partial multi-view networks,'' in \emph{Advances in Neural Information
  Processing Systems}, 2019, pp. 559--569.

\bibitem{xu2014large}
C.~Xu, D.~Tao, and C.~Xu, ``Large-margin multi-view information bottleneck,''
  \emph{IEEE Transactions on Pattern analysis and Machine Intelligence},
  vol.~36, no.~8, pp. 1559--1572, 2014.

\bibitem{liu2021one}
X.~Liu, L.~Liu, Q.~Liao, S.~Wang, Y.~Zhang, W.~Tu, C.~Tang, J.~Liu, and E.~Zhu,
  ``One pass late fusion multi-view clustering,'' in \emph{International
  Conference on Machine Learning}, 2021, pp. 6850--6859.

\bibitem{peng2019comic}
X.~Peng, Z.~Huang, J.~Lv, H.~Zhu, and J.~T. Zhou, ``Comic: Multi-view
  clustering without parameter selection,'' in \emph{International Conference
  on Machine Learning}, 2019, pp. 5092--5101.

\bibitem{li2017discriminative}
J.~Li, C.~Xu, W.~Yang, C.~Sun, and D.~Tao, ``Discriminative multi-view
  interactive image re-ranking,'' \emph{IEEE Transactions on Image Processing},
  vol.~26, no.~7, pp. 3113--3127, 2017.

\bibitem{esteva2017dermatologist}
A.~Esteva, B.~Kuprel, R.~A. Novoa, J.~Ko, S.~M. Swetter, H.~M. Blau, and
  S.~Thrun, ``Dermatologist-level classification of skin cancer with deep
  neural networks,'' \emph{Nature}, vol. 542, no. 7639, pp. 115--118, 2017.

\bibitem{bojarski2016end}
M.~Bojarski, D.~Del~Testa, D.~Dworakowski, B.~Firner, B.~Flepp, P.~Goyal, L.~D.
  Jackel, M.~Monfort, U.~Muller, J.~Zhang \emph{et~al.}, ``End to end learning
  for self-driving cars,'' \emph{arXiv preprint arXiv:1604.07316}, 2016.

\bibitem{foresti2002distributed}
G.~L. Foresti and L.~Snidaro, ``A distributed sensor network for video
  surveillance of outdoor environments,'' in \emph{Proceedings International
  Conference on Image Processing}, vol.~1, 2002, pp. I--I.

\bibitem{simonyan2014two}
K.~Simonyan and A.~Zisserman, ``Two-stream convolutional networks for action
  recognition in videos,'' in \emph{Advances in Neural Information Processing
  Systems}, 2014, pp. 568--576.

\bibitem{shutova2016black}
E.~Shutova, D.~Kiela, and J.~Maillard, ``Black holes and white rabbits:
  Metaphor identification with visual features,'' in \emph{Proceedings of the
  Conference of the North American Chapter of the Association for Computational
  Linguistics: Human Language Technologies}, 2016, pp. 160--170.

\bibitem{tsai2019multimodal}
Y.-H.~H. Tsai, S.~Bai, P.~P. Liang, J.~Z. Kolter, L.-P. Morency, and
  R.~Salakhutdinov, ``Multimodal transformer for unaligned multimodal language
  sequences,'' in \emph{Proceedings of the Conference. Association for
  Computational Linguistics. Meeting}, vol. 2019, 2019, p. 6558.

\bibitem{natarajan2012multimodal}
P.~Natarajan, S.~Wu, S.~Vitaladevuni, X.~Zhuang, S.~Tsakalidis, U.~Park,
  R.~Prasad, and P.~Natarajan, ``Multimodal feature fusion for robust event
  detection in web videos,'' in \emph{IEEE Conference on Computer Vision and
  Pattern Recognition}, 2012, pp. 1298--1305.

\bibitem{perez2019mfas}
J.-M. P{\'e}rez-R{\'u}a, V.~Vielzeuf, S.~Pateux, M.~Baccouche, and F.~Jurie,
  ``Mfas: Multimodal fusion architecture search,'' in \emph{Proceedings of the
  IEEE/CVF Conference on Computer Vision and Pattern Recognition}, 2019, pp.
  6966--6975.

\bibitem{yan2004learning}
R.~Yan, J.~Yang, and A.~G. Hauptmann, ``Learning query-class dependent weights
  in automatic video retrieval,'' in \emph{Proceedings of the ACM International
  Conference on Multimedia}, 2004, pp. 548--555.

\bibitem{atrey2010multimodal}
P.~K. Atrey, M.~A. Hossain, A.~El~Saddik, and M.~S. Kankanhalli, ``Multimodal
  fusion for multimedia analysis: a survey,'' \emph{Multimedia Systems},
  vol.~16, no.~6, pp. 345--379, 2010.

\bibitem{poria2015deep}
S.~Poria, E.~Cambria, and A.~Gelbukh, ``Deep convolutional neural network
  textual features and multiple kernel learning for utterance-level multimodal
  sentiment analysis,'' in \emph{Proceedings of the Conference on Empirical
  Methods in Natural Language Processing}, 2015, pp. 2539--2544.

\bibitem{perrin2009multimodal}
R.~J. Perrin, A.~M. Fagan, and D.~M. Holtzman, ``Multimodal techniques for
  diagnosis and prognosis of alzheimer's disease,'' \emph{Nature}, vol. 461,
  no. 7266, pp. 916--922, 2009.

\bibitem{sui2018multimodal}
J.~Sui, S.~Qi, T.~G. van Erp, J.~Bustillo, R.~Jiang, D.~Lin, J.~A. Turner,
  E.~Damaraju, A.~R. Mayer, Y.~Cui \emph{et~al.}, ``Multimodal neuromarkers in
  schizophrenia via cognition-guided mri fusion,'' \emph{Nature
  Communications}, vol.~9, no.~1, pp. 1--14, 2018.

\bibitem{hou2019neural}
H.~Hou, Q.~Zheng, Y.~Zhao, A.~Pouget, and Y.~Gu, ``Neural correlates of optimal
  multisensory decision making under time-varying reliabilities with an
  invariant linear probabilistic population code,'' \emph{Neuron}, vol. 104,
  no.~5, pp. 1010--1021, 2019.

\bibitem{charpentier2020posterior}
B.~Charpentier, D.~Z{\"u}gner, and S.~G{\"u}nnemann, ``Posterior network:
  Uncertainty estimation without ood samples via density-based pseudo-counts,''
  \emph{Advances in Neural Information Processing Systems}, vol.~33, pp.
  1356--1367, 2020.

\bibitem{mackay1992bayesian}
D.~J. MacKay, ``Bayesian methods for adaptive models,'' Ph.D. dissertation,
  California Institute of Technology, 1992.

\bibitem{bernardo2009bayesian}
J.~M. Bernardo and A.~F. Smith, \emph{Bayesian theory}.\hskip 1em plus 0.5em
  minus 0.4em\relax John Wiley \& Sons, 2009, vol. 405.

\bibitem{neal2012bayesian}
R.~M. Neal, \emph{Bayesian learning for neural networks}.\hskip 1em plus 0.5em
  minus 0.4em\relax Springer Science \& Business Media, 2012, vol. 118.

\bibitem{mackay1992practical}
D.~J. MacKay, ``A practical bayesian framework for backpropagation networks,''
  \emph{Neural computation}, vol.~4, no.~3, pp. 448--472, 1992.

\bibitem{graves2011practical}
A.~Graves, ``Practical variational inference for neural networks,'' in
  \emph{Advances in Neural Information Processing Systems}, 2011, pp.
  2348--2356.

\bibitem{ranganath2014black}
R.~Ranganath, S.~Gerrish, and D.~Blei, ``Black box variational inference,'' in
  \emph{Proceedings of the International Conference on Artificial Intelligence
  and Statistics}, 2014, pp. 814--822.

\bibitem{blundell2015weight}
C.~Blundell, J.~Cornebise, K.~Kavukcuoglu, and D.~Wierstra, ``Weight
  uncertainty in neural network,'' in \emph{International Conference on Machine
  Learning}, 2015, pp. 1613--1622.

\bibitem{gal2016dropout}
Y.~Gal and Z.~Ghahramani, ``Dropout as a bayesian approximation: Representing
  model uncertainty in deep learning,'' in \emph{International Conference on
  Machine Learning}, 2016, pp. 1050--1059.

\bibitem{srivastava2014dropout}
N.~Srivastava, G.~Hinton, A.~Krizhevsky, I.~Sutskever, and R.~Salakhutdinov,
  ``Dropout: a simple way to prevent neural networks from overfitting,''
  \emph{The Journal of Machine Learning Research}, vol.~15, no.~1, pp.
  1929--1958, 2014.

\bibitem{lakshminarayanan2017simple}
B.~Lakshminarayanan, A.~Pritzel, and C.~Blundell, ``Simple and scalable
  predictive uncertainty estimation using deep ensembles,'' in \emph{Advances
  in Neural Information Processing Systems}, 2017, pp. 6402--6413.

\bibitem{sensoy2018evidential}
M.~Sensoy, L.~Kaplan, and M.~Kandemir, ``Evidential deep learning to quantify
  classification uncertainty,'' in \emph{Advances in Neural Information
  Processing Systems}, 2018, pp. 3179--3189.

\bibitem{van2020uncertainty}
J.~Van~Amersfoort, L.~Smith, Y.~W. Teh, and Y.~Gal, ``Uncertainty estimation
  using a single deep deterministic neural network,'' in \emph{International
  Conference on Machine Learning}, 2020, pp. 9690--9700.

\bibitem{han2021trusted}
Z.~Han, C.~Zhang, H.~Fu, and J.~T. Zhou, ``Trusted multi-view classification,''
  2021.

\bibitem{hotelling1992relations}
H.~Hotelling, ``Relations between two sets of variates,'' in
  \emph{Breakthroughs in statistics}.\hskip 1em plus 0.5em minus 0.4em\relax
  Springer, 1992, pp. 162--190.

\bibitem{akaho2006kernel}
S.~Akaho, ``A kernel method for canonical correlation analysis,'' \emph{arXiv
  preprint cs/0609071}, 2006.

\bibitem{wang2007variational}
C.~Wang, ``Variational bayesian approach to canonical correlation analysis,''
  \emph{IEEE Transactions on Neural Networks}, vol.~18, no.~3, pp. 905--910,
  2007.

\bibitem{andrew2013deep}
G.~Andrew, R.~Arora, J.~Bilmes, and K.~Livescu, ``Deep canonical correlation
  analysis,'' in \emph{International Conference on Machine Learning}, 2013, pp.
  1247--1255.

\bibitem{wang2015deep}
W.~Wang, R.~Arora, K.~Livescu, and J.~Bilmes, ``On deep multi-view
  representation learning,'' in \emph{International Conference on Machine
  Learning}, 2015, pp. 1083--1092.

\bibitem{wang2016deep}
W.~Wang, X.~Yan, H.~Lee, and K.~Livescu, ``Deep variational canonical
  correlation analysis,'' \emph{arXiv preprint arXiv:1610.03454}, 2016.

\bibitem{zhang2017hierarchical}
H.~Zhang, V.~M. Patel, and R.~Chellappa, ``Hierarchical multimodal metric
  learning for multimodal classification,'' in \emph{Proceedings of the IEEE
  Conference on Computer Vision and Pattern Recognition}, 2017, pp. 3057--3065.

\bibitem{tian2019contrastive}
Y.~Tian, D.~Krishnan, and P.~Isola, ``Contrastive multiview coding,'' in
  \emph{European Conference on Computer Vision}, 2020, pp. 776--794.

\bibitem{bachman2019learning}
P.~Bachman, R.~D. Hjelm, and W.~Buchwalter, ``Learning representations by
  maximizing mutual information across views,'' in \emph{Advances in Neural
  Information Processing Systems}, 2019, pp. 15\,535--15\,545.

\bibitem{chen2020simple}
T.~Chen, S.~Kornblith, M.~Norouzi, and G.~Hinton, ``A simple framework for
  contrastive learning of visual representations,'' in \emph{International
  Conference on Machine Learning}, 2020, pp. 1597--1607.

\bibitem{hassani2020contrastive}
K.~Hassani and A.~H. Khasahmadi, ``Contrastive multi-view representation
  learning on graphs,'' in \emph{International Conference on Machine Learning},
  2020, pp. 4116--4126.

\bibitem{kiela2018efficient}
D.~Kiela, E.~Grave, A.~Joulin, and T.~Mikolov, ``Efficient large-scale
  multi-modal classification,'' in \emph{Proceedings of the AAAI Conference on
  Artificial Intelligence}, vol.~32, no.~1, 2018.

\bibitem{bian2017revisiting}
Y.~Bian, C.~Gan, X.~Liu, F.~Li, X.~Long, Y.~Li, H.~Qi, J.~Zhou, S.~Wen, and
  Y.~Lin, ``Revisiting the effectiveness of off-the-shelf temporal modeling
  approaches for large-scale video classification,'' \emph{arXiv preprint
  arXiv:1708.03805}, 2017.

\bibitem{kiela2019supervised}
D.~Kiela, S.~Bhooshan, H.~Firooz, and D.~Testuggine, ``Supervised multimodal
  bitransformers for classifying images and text,'' \emph{arXiv preprint
  arXiv:1909.02950}, 2019.

\bibitem{wang2020makes}
W.~Wang, D.~Tran, and M.~Feiszli, ``What makes training multi-modal
  classification networks hard?'' in \emph{Proceedings of the IEEE/CVF
  Conference on Computer Vision and Pattern Recognition}, 2020, pp.
  12\,695--12\,705.

\bibitem{wang2019multi}
S.~Wang, X.~Liu, E.~Zhu, C.~Tang, J.~Liu, J.~Hu, J.~Xia, and J.~Yin,
  ``Multi-view clustering via late fusion alignment maximization.'' in
  \emph{Proceedings of the International Joint Conference on Artificial
  Intelligence}, 2019, pp. 3778--3784.

\bibitem{liu2018late}
X.~Liu, X.~Zhu, M.~Li, L.~Wang, C.~Tang, J.~Yin, D.~Shen, H.~Wang, and W.~Gao,
  ``Late fusion incomplete multi-view clustering,'' \emph{IEEE Transactions on
  Pattern analysis and Machine Intelligence}, vol.~41, no.~10, pp. 2410--2423,
  2018.

\bibitem{liong2020amvnet}
V.~E. Liong, T.~N.~T. Nguyen, S.~Widjaja, D.~Sharma, and Z.~J. Chong, ``Amvnet:
  Assertion-based multi-view fusion network for lidar semantic segmentation,''
  \emph{arXiv preprint arXiv:2012.04934}, 2020.

\bibitem{morvant2014majority}
E.~Morvant, A.~Habrard, and S.~Ayache, ``Majority vote of diverse classifiers
  for late fusion,'' in \emph{Joint IAPR International Workshops on Statistical
  Techniques in Pattern Recognition (SPR) and Structural and Syntactic Pattern
  Recognition (SSPR)}, 2014, pp. 153--162.

\bibitem{potamianos2003recent}
G.~Potamianos, C.~Neti, G.~Gravier, A.~Garg, and A.~W. Senior, ``Recent
  advances in the automatic recognition of audiovisual speech,''
  \emph{Proceedings of the IEEE}, vol.~91, no.~9, pp. 1306--1326, 2003.

\bibitem{evangelopoulos2013multimodal}
G.~Evangelopoulos, A.~Zlatintsi, A.~Potamianos, P.~Maragos, K.~Rapantzikos,
  G.~Skoumas, and Y.~Avrithis, ``Multimodal saliency and fusion for movie
  summarization based on aural, visual, and textual attention,'' \emph{IEEE
  Transactions on Multimedia}, vol.~15, no.~7, pp. 1553--1568, 2013.

\bibitem{wang2021late}
S.~Wang, X.~Liu, L.~Liu, S.~Zhou, and E.~Zhu, ``Late fusion multiple kernel
  clustering with proxy graph refinement,'' \emph{IEEE Transactions on Neural
  Networks and Learning Systems}, 2021.

\bibitem{zhang2021late}
T.~Zhang, X.~Liu, L.~Gong, S.~Wang, X.~Niu, and L.~Shen, ``Late fusion multiple
  kernel clustering with local kernel alignment maximization,'' \emph{IEEE
  Transactions on Multimedia}, 2021.

\bibitem{wei2019surface}
W.~Wei, Q.~Dai, Y.~Wong, Y.~Hu, M.~Kankanhalli, and W.~Geng,
  ``Surface-electromyography-based gesture recognition by multi-view deep
  learning,'' \emph{IEEE Transactions on Biomedical Engineering}, vol.~66,
  no.~10, pp. 2964--2973, 2019.

\bibitem{wang2021mogonet}
T.~Wang, W.~Shao, Z.~Huang, H.~Tang, J.~Zhang, Z.~Ding, and K.~Huang, ``Mogonet
  integrates multi-omics data using graph convolutional networks allowing
  patient classification and biomarker identification,'' \emph{Nature
  Communications}, vol.~12, pp. 1--13, 2021.

\bibitem{wang2019generative}
L.~Wang, Z.~Ding, Z.~Tao, Y.~Liu, and Y.~Fu, ``Generative multi-view human
  action recognition,'' in \emph{Proceedings of the IEEE/CVF International
  Conference on Computer Vision}, 2019, pp. 6212--6221.

\bibitem{ding2021cooperative}
D.~Y. Ding and R.~Tibshirani, ``Cooperative learning for multi-view analysis,''
  \emph{arXiv preprint arXiv:2112.12337}, 2021.

\bibitem{subedar2019uncertainty}
M.~Subedar, R.~Krishnan, P.~L. Meyer, O.~Tickoo, and J.~Huang,
  ``Uncertainty-aware audiovisual activity recognition using deep bayesian
  variational inference,'' in \emph{Proceedings of the IEEE/CVF International
  Conference on Computer Vision}, 2019, pp. 6301--6310.

\bibitem{tian2020uno}
J.~Tian, W.~Cheung, N.~Glaser, Y.-C. Liu, and Z.~Kira, ``Uno: Uncertainty-aware
  noisy-or multimodal fusion for unanticipated input degradation,'' in
  \emph{IEEE International Conference on Robotics and Automation}, 2020, pp.
  5716--5723.

\bibitem{denker1991transforming}
J.~S. Denker and Y.~LeCun, ``Transforming neural-net output levels to
  probability distributions,'' in \emph{Advances in Neural Information
  Processing Systems}, 1991, pp. 853--859.

\bibitem{havasi2021training}
M.~Havasi, R.~Jenatton, S.~Fort, J.~Z. Liu, J.~Snoek, B.~Lakshminarayanan,
  A.~M. Dai, and D.~Tran, ``Training independent subnetworks for robust
  prediction,'' in \emph{International Conference on Learning Representations},
  2021.

\bibitem{antoran2020depth}
J.~Antor{\'{a}}n, J.~U. Allingham, and J.~M. Hern{\'{a}}ndez{-}Lobato, ``Depth
  uncertainty in neural networks,'' in \emph{Advances in Neural Information
  Processing Systems}, 2020.

\bibitem{malinin2018predictive}
A.~Malinin and M.~Gales, ``Predictive uncertainty estimation via prior
  networks,'' vol.~31, 2018.

\bibitem{Malinin2020Ensemble}
A.~Malinin, B.~Mlodozeniec, and M.~Gales, ``Ensemble distribution
  distillation,'' in \emph{International Conference on Learning
  Representations}, 2020.

\bibitem{kopetzki2021evaluating}
A.-K. Kopetzki, B.~Charpentier, D.~Z{\"u}gner, S.~Giri, and S.~G{\"u}nnemann,
  ``Evaluating robustness of predictive uncertainty estimation: Are
  dirichlet-based models reliable?'' in \emph{International Conference on
  Machine Learning}, 2021, pp. 5707--5718.

\bibitem{dempster1967upper}
A.~Dempster, ``Upper and lower probabilities induced by a multivalued
  mapping,'' \emph{The Annals of Mathematical Statistics}, pp. 325--339, 1967.

\bibitem{dempster1968generalization}
A.~P. Dempster, ``A generalization of bayesian inference,'' \emph{Journal of
  the Royal Statistical Society: Series B (Methodological)}, vol.~30, no.~2,
  pp. 205--232, 1968.

\bibitem{shafer1976mathematical}
G.~Shafer, \emph{A mathematical theory of evidence}.\hskip 1em plus 0.5em minus
  0.4em\relax Princeton University Press, 1976, vol.~42.

\bibitem{sentz2002combination}
K.~Sentz, S.~Ferson \emph{et~al.}, \emph{Combination of evidence in
  Dempster-Shafer theory}.\hskip 1em plus 0.5em minus 0.4em\relax Citeseer,
  2002, vol. 4015.

\bibitem{josang2012interpretation}
A.~J{\o}sang and R.~Hankin, ``Interpretation and fusion of hyper opinions in
  subjective logic,'' in \emph{2012 15th International Conference on
  Information Fusion}, 2012, pp. 1225--1232.

\bibitem{liu2017weighted}
Y.-T. Liu, N.~R. Pal, A.~R. Marathe, and C.-T. Lin, ``Weighted fuzzy
  dempster--shafer framework for multimodal information integration,''
  \emph{IEEE Transactions on Fuzzy Systems}, vol.~26, no.~1, pp. 338--352,
  2017.

\bibitem{bloch1996some}
I.~Bloch, ``Some aspects of dempster-shafer evidence theory for classification
  of multi-modality medical images taking partial volume effect into account,''
  \emph{Pattern Recognition Letters}, vol.~17, no.~8, pp. 905--919, 1996.

\bibitem{kantardzic2010click}
M.~Kantardzic, C.~Walgampaya, R.~Yampolskiy, and R.~J. Woo, ``Click fraud
  prevention via multimodal evidence fusion by dempster-shafer theory,'' in
  \emph{IEEE Conference on Multisensor Fusion and Integration}, 2010, pp.
  26--31.

\bibitem{basir2007engine}
O.~Basir and X.~Yuan, ``Engine fault diagnosis based on multi-sensor
  information fusion using dempster--shafer evidence theory,''
  \emph{Information fusion}, vol.~8, no.~4, pp. 379--386, 2007.

\bibitem{le1997application}
S.~Le~Hegarat-Mascle, I.~Bloch, and D.~Vidal-Madjar, ``Application of
  dempster-shafer evidence theory to unsupervised classification in multisource
  remote sensing,'' \emph{IEEE transactions on geoscience and remote sensing},
  vol.~35, no.~4, pp. 1018--1031, 1997.

\bibitem{moon2020confidence}
J.~Moon, J.~Kim, Y.~Shin, and S.~Hwang, ``Confidence-aware learning for deep
  neural networks,'' in \emph{International Conference on Machine Learning},
  2020, pp. 7034--7044.

\bibitem{NEURIPS2018_3ea2db50}
A.~Malinin and M.~Gales, ``Predictive uncertainty estimation via prior
  networks,'' in \emph{Advances in Neural Information Processing Systems},
  vol.~31, 2018.

\bibitem{NEURIPS2020_0eac690d}
B.~Charpentier, D.~Z\"{u}gner, and S.~G\"{u}nnemann, ``Posterior network:
  Uncertainty estimation without ood samples via density-based pseudo-counts,''
  in \emph{Advances in Neural Information Processing Systems}, 2020.

\bibitem{bishop2006pattern}
C.~M. Bishop, \emph{Pattern Recognition and Machine Learning}.\hskip 1em plus
  0.5em minus 0.4em\relax springer, 2006.

\bibitem{kingma2014auto}
D.~P. {Kingma} and M.~{Welling}, ``Auto-encoding variational bayes,'' in
  \emph{International Conference on Learning Representations}, 2014.

\bibitem{higgins2017beta}
I.~{Higgins}, L.~{Matthey}, A.~{Pal}, C.~{Burgess}, X.~{Glorot},
  M.~{Botvinick}, S.~{Mohamed}, and A.~{Lerchner}, ``Beta-vae: Learning basic
  visual concepts with a constrained variational framework,'' in
  \emph{International Conference on Learning Representations}, 2017.

\bibitem{jsang2018subjective}
A.~J{\o}sang, \emph{Subjective Logic: A formalism for reasoning under
  uncertainty}.\hskip 1em plus 0.5em minus 0.4em\relax Springer Publishing
  Company, Incorporated, 2018.

\bibitem{bao2021evidential}
W.~Bao, Q.~Yu, and Y.~Kong, ``Evidential deep learning for open set action
  recognition,'' in \emph{Proceedings of the IEEE/CVF International Conference
  on Computer Vision}, 2021, pp. 13\,349--13\,358.

\bibitem{frigyik2010introduction}
B.~A. Frigyik, A.~Kapila, and M.~R. Gupta, ``Introduction to the dirichlet
  distribution and related processes,'' Department of Electrical Engineering,
  University of Washignton, Tech. Rep. 0006, 2010.

\bibitem{guo2017calibration}
C.~Guo, G.~Pleiss, Y.~Sun, and K.~Q. Weinberger, ``On calibration of modern
  neural networks,'' in \emph{International Conference on Machine Learning},
  2017, pp. 1321--1330.

\bibitem{wah2011caltech}
C.~Wah, S.~Branson, P.~Welinder, P.~Perona, and S.~Belongie, ``The
  {Caltech}-{UCSD} {Birds}-200-2011 dataset,'' California Institute of
  Technology, Tech. Rep., 2011.

\bibitem{fei2004learning}
L.~Fei-Fei, R.~Fergus, and P.~Perona, ``Learning generative visual models from
  few training examples: An incremental bayesian approach tested on 101 object
  categories,'' in \emph{Conference on Computer Vision and Pattern Recognition
  workshop}, 2004, pp. 178--178.

\bibitem{fei2005bayesian}
L.~Fei-Fei and P.~Perona, ``A bayesian hierarchical model for learning natural
  scene categories,'' in \emph{Computer Society Conference on Computer Vision
  and Pattern Recognition}, vol.~2, 2005, pp. 524--531.

\bibitem{kuehne2011hmdb}
H.~Kuehne, H.~Jhuang, E.~Garrote, T.~Poggio, and T.~Serre, ``Hmdb: a large
  video database for human motion recognition,'' in \emph{2011 International
  Conference on Computer Vision}, 2011, pp. 2556--2563.

\bibitem{heo2018uncertainty}
J.~Heo, H.~B. Lee, S.~Kim, J.~Lee, K.~J. Kim, E.~Yang, and S.~J. Hwang,
  ``Uncertainty-aware attention for reliable interpretation and prediction,''
  in \emph{Advances in Neural Information Processing Systems}, 2018, pp.
  909--918.

\bibitem{kingma2014adam}
D.~P. Kingma and J.~Ba, ``Adam: A method for stochastic optimization,''
  \emph{arXiv preprint arXiv:1412.6980}, 2014.

\bibitem{wang2016modality}
A.~Wang, J.~Cai, J.~Lu, and T.-J. Cham, ``Modality and component aware feature
  fusion for rgb-d scene classification,'' in \emph{Proceedings of the IEEE
  Conference on Computer Vision and Pattern Recognition}, 2016, pp. 5995--6004.

\bibitem{du2018depth}
D.~Du, X.~Xu, T.~Ren, and G.~Wu, ``Depth images could tell us more: Enhancing
  depth discriminability for rgb-d scene recognition,'' in \emph{IEEE
  International Conference on Multimedia and Expo}, 2018, pp. 1--6.

\bibitem{li2018df}
Y.~Li, J.~Zhang, Y.~Cheng, K.~Huang, and T.~Tan, ``Df 2 net: Discriminative
  feature learning and fusion network for rgb-d indoor scene classification,''
  in \emph{Proceedings of the AAAI Conference on Artificial Intelligence},
  vol.~32, no.~1, 2018.

\bibitem{song2018learning}
X.~Song, S.~Jiang, L.~Herranz, and C.~Chen, ``Learning effective rgb-d
  representations for scene recognition,'' \emph{IEEE Transactions on Image
  Processing}, vol.~28, no.~2, pp. 980--993, 2018.

\bibitem{yuan2019acm}
Y.~Yuan, Z.~Xiong, and Q.~Wang, ``Acm: Adaptive cross-modal graph convolutional
  neural networks for rgb-d scene recognition,'' in \emph{Proceedings of the
  AAAI Conference on Artificial Intelligence}, vol.~33, no.~01, 2019, pp.
  9176--9184.

\bibitem{du2019translate}
D.~Du, L.~Wang, H.~Wang, K.~Zhao, and G.~Wu, ``Translate-to-recognize networks
  for rgb-d scene recognition,'' in \emph{Proceedings of the IEEE Conference on
  Computer Vision and Pattern Recognition}, 2019, pp. 11\,836--11\,845.

\bibitem{Ayub_2020_BMVC}
A.~Ayub and A.~R. Wagner, ``Centroid based concept learning for rgb-d indoor
  scene classification,'' in \emph{British Machine Vision Conference}, 2020.

\bibitem{song2017combining}
X.~Song, S.~Jiang, and L.~Herranz, ``Combining models from multiple sources for
  rgb-d scene recognition.'' in \emph{Proceedings of the International Joint
  Conference on Artificial Intelligence}, 2017, pp. 4523--4529.

\bibitem{song2015sun}
S.~Song, S.~P. Lichtenberg, and J.~Xiao, ``Sun rgb-d: A rgb-d scene
  understanding benchmark suite,'' in \emph{Proceedings of the IEEE Conference
  on Computer Vision and Pattern Recognition}, 2015, pp. 567--576.

\bibitem{silberman2012indoor}
N.~Silberman, D.~Hoiem, P.~Kohli, and R.~Fergus, ``Indoor segmentation and
  support inference from rgbd images,'' in \emph{European Conference on
  Computer vision}, 2012, pp. 746--760.

\bibitem{he2016deep}
K.~He, X.~Zhang, S.~Ren, and J.~Sun, ``Deep residual learning for image
  recognition,'' in \emph{Proceedings of the IEEE Conference on Computer Vision
  and Pattern Recognition}, 2016, pp. 770--778.

\bibitem{wang2015recipe}
X.~Wang, D.~Kumar, N.~Thome, M.~Cord, and F.~Precioso, ``Recipe recognition
  with large multimodal food dataset,'' in \emph{IEEE International Conference
  on Multimedia \& Expo Workshops}, 2015, pp. 1--6.

\bibitem{devlin2018bert}
J.~Devlin, M.-W. Chang, K.~Lee, and K.~Toutanova, ``Bert: Pre-training of deep
  bidirectional transformers for language understanding,'' in \emph{North
  American Chapter of the Association for Computational Linguistics}, 2019.

\end{thebibliography}
}
\vspace{-12 mm} 
\begin{IEEEbiography}[{\includegraphics[width=1in,height=1.1in,clip,keepaspectratio]{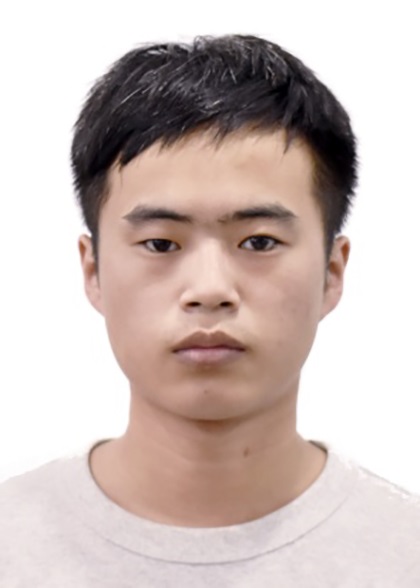}}]{Zongbo Han} is now a Ph.D. candidate at Tianjin University. He received his M.E. degree from Tianjin University in 2021, and his B.E. degree from Dalian University of Technology in 2019. He is now focusing on multi-view learning, uncertainty machine learning, and medical image analysis.
\end{IEEEbiography}
\vspace{-15 mm} 
\begin{IEEEbiography}[{\includegraphics[width=1in,height=1.1in,clip,keepaspectratio]{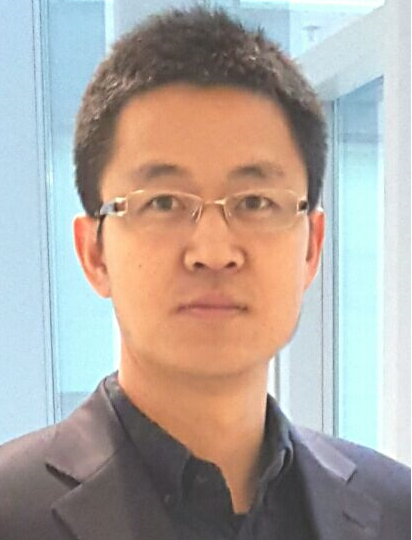}}]{Changqing Zhang}
received the B.S. and M.S. degrees from the College of Computer, Sichuan University, Chengdu, China, in 2005 and 2008, respectively, and the Ph.D. degree in Computer Science from Tianjin University, China, in 2016. He is an associate professor in the College of Intelligence and Computing, Tianjin University. He has been a postdoc research fellow in the Department of Radiology and BRIC, School of Medicine, University of North Carolina at Chapel Hill, NC, USA. His current research interests include machine learning, computer vision and medical image analysis.
\end{IEEEbiography}
\vspace{-12 mm} 
\begin{IEEEbiography}[{\includegraphics[width=1in,height=1.1in,clip,keepaspectratio]{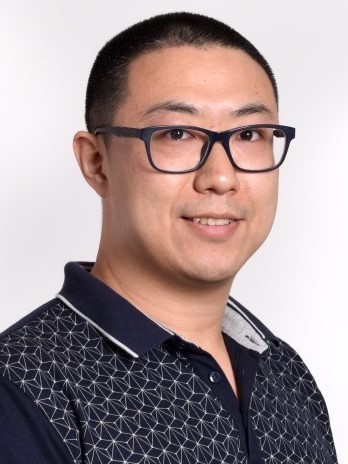}}]{Huazhu Fu} (SM'18) is a senior scientist at Institute of High Performance Computing (IHPC), Agency for Science, Technology and Research (A*STAR). He received his Ph.D. from Tianjin University in 2013. Previously, He was a Research Fellow (2013-2015) in NTU, Singapore, a Research Scientist (2015-2018) in I2R, A*STAR, Singapore, and a Senior Scientist (2018-2021) in Inception Institute of Artificial Intelligence, UAE. His research interests include computer vision, machine learning, and AI in health. He is a recipient of ICME 2021 Best Paper Award. He currently serves as an associate editor of IEEE TMI, IEEE JBHI, Scientific Reports, and IEEE Access.
\end{IEEEbiography}
\vspace{-1 cm} 
\begin{IEEEbiography}[{\includegraphics[width=1in,height=1.1in,clip,keepaspectratio]{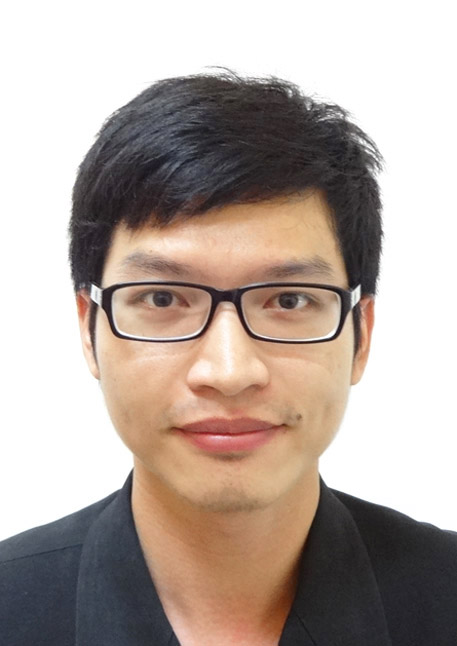}}]{Joey Tianyi Zhou} (Member, IEEE) received the Ph.D. degree in computer science from Nanyang Technological University, Singapore, in 2015. He is
currently a Senior Scientist with the Centre for Frontier AI Research (CFAR), Research Agency for Science, Technology, and Research, Singapore. He was a recipient of the Best Poster Honorable Mention at ACML in 2012, the Best Paper Award
from the BeyondLabeler Workshop on IJCAI in 2016, the Best Paper Nomination at ECCV in 2016, and the NIPS 2017 Best Reviewer Award. He has served as an Associate Editor for IEEE Transactions on Emerging Topics in Computational Intelligence (TETCI), IEEE Access, and IET Image Processing.


\end{IEEEbiography}

\par
\clearpage
\newpage
\section{Appendix}
\subsection{Decomposition of $p_{\theta^m}(\mathbf{y}\mid \mathbf{x}^m)$}
\begin{equation}
\begin{aligned}
& \log p_{\theta^m}(\mathbf{y}\mid \mathbf{x}^m) =
\mathbb{E}_{q_{\theta^m}(\boldsymbol{\mu}\mid\mathbf{x}^m)}[\log p_{\theta^m}(\mathbf{y} \mid \mathbf{x}^m)] \\
&=\mathbb{E}_{q_{\theta^m}(\boldsymbol{\mu}\mid\mathbf{x}^m)}\big[ \log  \frac{p_{\theta^m}(\mathbf{y} \mid\boldsymbol{\mu}^m, \mathbf{x}^m)p(\boldsymbol{\mu}^m\mid \mathbf{x}^m)}{p_{\theta^m}(
\boldsymbol{\mu}^m \mid \mathbf{y}, \mathbf{x}^m)} \big]\\
&=\mathbb{E}_{q_{\theta^m}(\boldsymbol{\mu}\mid\mathbf{x}^m)}\\
& \quad \quad \big[ \log  \frac{p_{\theta^m}(\mathbf{y} \mid\boldsymbol{\mu}^m, \mathbf{x}^m)p(\boldsymbol{\mu}^m\mid \mathbf{x}^m)}{p_{\theta^m}(
\boldsymbol{\mu}^m \mid \mathbf{y}, \mathbf{x}^m)} \frac{q_{\theta^m}(
\boldsymbol{\mu}^m \mid \mathbf{x}^m)}{q_{\theta^m}(
\boldsymbol{\mu}^m \mid \mathbf{x}^m)}\big]\\
&=D_{KL}\left[q_{\theta^m}{(\boldsymbol{\mu}^m\mid \mathbf{x}^m)\| p_{\theta^m}(\boldsymbol{\mu}^m \mid \mathbf{x}^m, \mathbf{y})}\right] \\
& \quad \quad + \mathbb{E}_{q_{\theta^m}(\boldsymbol{\mu}^m\mid \mathbf{x}^m)}\left[\log p(\mathbf{y}\mid \boldsymbol{\mu}^m, \mathbf{x}^m)\right]\\
& \quad \quad - D_{KL}\left[q_{\theta^m}(\boldsymbol{\mu}^m \mid \mathbf{x}^m)\| p_{\theta^m}(\boldsymbol{\mu}^m \mid \mathbf{x}^m))\right].
\end{aligned}
\end{equation}
Then under the conditional independent assumption $p_{\theta^m}(\mathbf{y}\mid \boldsymbol{\mu}^m, \mathbf{x}^m)=p_{\theta^m}(\mathbf{y}\mid\boldsymbol{\mu}^m)$, Eq.~\ref{eq:likelihood} can be derived. $p_{\theta^m}(\mathbf{y}\mid \boldsymbol{\mu}^m, \mathbf{x}^m)=p_{\theta^m}(\mathbf{y}\mid\boldsymbol{\mu}^m)$ is a mild assumption, which means that after $\boldsymbol{\mu}^m$ is given, $\mathbf{x}^m$ cannot provide more information about $\mathbf{y}$.

\subsection{Proof Details}
\begin{proof}
Proof details of Proposition~\ref{prop1}.
\begin{equation}
\begin{aligned}
\nonumber
b_t & = \frac{b_t^o b_t^a + b_t^o u^a + b_t^a u^o}{\sum_{k=1}^Kb_k^ob_k^a+u^a+u^o-u^ou^a}\\
&\geq \frac{b_t^o b_t^a + b_t^o u^a + b_m^o u^o}{\sum_{k=1}^K b_m^o b_k^a+u^a+u^o-u^ou^a}\\
&\geq \frac{b_t^o b_t^a + b_t^o u^a + b_t^o u^o}{b_m^o\sum_{k=1}^K b_k^a+u^a+u^o-u^ou^a}\\
&\geq \frac{b_t^o( b_t^a + u^a + u^o)}{b_m^o(1-u^ a)+u^a+u^o-u^ou^a}\\
&\geq \frac{b_t^o( b_t^a + u^a + u^o)}{b_m^o+u^a+u^o}\geq b_t^o.
\end{aligned}
\end{equation}
\end{proof}
\begin{proof}
Proof details of Proposition~\ref{prop2}, where $b_m^a$ is the largest belief mass in $\{b_k^a\}_{k=1}^K$.
\begin{equation}
\begin{aligned}
\nonumber
b_t^o-b_t&=b_t^o - \frac{b_t^o b_t^a + b_t^o u^a + b_t^a u^o}{\sum_{k=1}^Kb_k^ob_k^a+u^a+u^o-u^ou^a}\\
&\leq b_t^o - \frac{b_t^o b_t^a + b_t^o u^a + b_t^a u^o}{b_m^a\sum_{k=1}^Kb_k^o+u^a+u^o-u^ou^a}\\
&\leq b_t^o - \frac{b_t^o u^a}{b_m^a+u^a+u^o-u^au^o}\\
&\leq b_t^o - \frac{b_t^o u^a}{1+u^o-u^au^o}=b_t^o\frac{1+u^o}{\frac{1}{1-u^a}+u^o}.
\end{aligned}
\end{equation}
\end{proof}
\textcolor{black}{\subsection{Decomposition details of $\mathcal{L}(\mathbf{x}^m, \mathbf{y})$}}
\textcolor{black}{First, we have the following property about the integral of the Dirichlet distribution. If $\boldsymbol{\mu}=[\mu_1,\cdots,\mu_K]\sim Dir(\boldsymbol{\mu\mid \boldsymbol{\alpha}})$, then we have $\mathbb{E}_{Dir(\boldsymbol{\mu\mid \boldsymbol{\alpha}})}\left[log(\mu_k)\right]=\psi(\alpha_k)-\psi(S)$.
Then the following equation can be derived:
\begin{equation}
\begin{aligned}
& \mathbb{E}_{q_{\theta}(\boldsymbol{\mu}\mid \mathbf{x})} \left[\log p(\mathbf{y}\mid \boldsymbol{\mu})\right] = \mathbb{E}_{Dir(\boldsymbol{\mu}\mid \boldsymbol{\alpha})}\left[\log p(\mathbf{y}\mid \boldsymbol{\mu})\right] \\
& = \mathbb{E}_{Dir(\boldsymbol{\mu}\mid \boldsymbol{\alpha})}\left[\log \prod_{k=1}^K \mu_k^{y_k} \right] \\
& = \int\left[\sum_{k=1}^{K} y_{k} \log \left(\mu_{k}\right)\right] \frac{1}{B\left(\boldsymbol{\alpha}\right)} \prod_{k=1}^{K} \mu_{k}^{\alpha_{k}-1} d \boldsymbol{\mu}\\
&=\sum_{k=1}^{K}y_{k}\left[\int \log \left(\mu_{k}\right) \frac{1}{B\left(\boldsymbol{\alpha}\right)} \prod_{k=1}^{K} \mu_{k}^{\alpha_{k}-1} d \boldsymbol{\mu}\right]\\
&= \sum_{k=1}^{K} y_{k}\left(\psi\left(\alpha_{k}\right) - \psi\left(S\right)\right).
\end{aligned}
\end{equation}
\subsection{Decomposition details of Eq.~\ref{eq:kl_d}}
First, we consider Kullback-Leibler divergence between two Dirichlet distributions $Dir(\boldsymbol{\mu}\mid\boldsymbol{\alpha})$ and $Dir(\boldsymbol{\mu}\mid\boldsymbol{\beta})$. We have the following property. The divergence between these two Dirichlet distributions is
\begin{equation}
\begin{aligned}
\label{eq:kl_proof}
& KL\left(Dir(\boldsymbol{\mu}\mid\boldsymbol{\alpha}) \| Dir(\boldsymbol{\mu}\mid\boldsymbol{\beta})\right) \\
&=\log \frac{\Gamma\left(\sum_{k=1}^{K} \alpha_{k}\right)}{\prod_{k=1}^{K} \Gamma\left(\alpha_{k}\right)} \frac{\prod_{k=1}^{K} \Gamma\left(\beta_{k}\right)}{\Gamma\left(\sum_{k=1}^{K} \beta_{k}\right)}\\
&+\sum_{k=1}^{K}\left(\alpha_{k}-\beta_{k}\right)\left(\psi\left(\sum_{k=1}^K\alpha_{k}\right)-\psi\left(\sum_{k=1}^K\beta_{k}\right)\right).
\end{aligned}
\end{equation}
Then Eq~\ref{eq:kl_d} is a special case of Eq~\ref{eq:kl_proof}.
}
\subsection{Qualitative Experimental Results on End-to-End Experiment}
{We present typical examples (from SUN RGB-D, NYUD Depth V2 and UMPC-Food101 datasets) with the highest and lowest uncertainty respectively. For SUN RGB-D and NYUD Depth V2 datasets, we show 18 samples with the largest or smallest uncertainty where the proposed ETMC algorithm is employed in Fig.~\ref{fig:ETMC_SUN_vis_more_confident}, Fig.~\ref{fig:ETMC_SUN_vis_more_uncertain}, Fig.~\ref{fig:ETMC_NYUD_vis_more_confident} and  Fig.~\ref{fig:ETMC_NYUD_vis_more_uncertain} respectively. Meanwhile, in Table~\ref{tab:examples1} and Table~\ref{tab:examples2}, we show the highest and lowest prediction samples in UMPC-Food101. It is observed that some samples from SUN RGB-D and NYUD Depth V2 have a weak intuitive correlation with the 
ground-truth class label, hence it is also difficult for our model to correctly classify them, e.g., samples in Fig.~\ref{fig:ETMC_SUN_vis_more_uncertain} and Fig.~\ref{fig:ETMC_NYUD_vis_more_uncertain}. In contrast, there are some samples that are easy to be classified intuitively, and hence our model is also confident for the predictions in Fig.~\ref{fig:ETMC_SUN_vis_more_confident} and Fig.~\ref{fig:ETMC_NYUD_vis_more_confident}. 
Since the UMPC-Food101 dataset is scraped from web pages, for a few samples the corresponding text descriptions are meaningless (\emph{e.g.}, the 1st one in Table~\ref{tab:examples1}). Meanwhile, some images are also quite challenging to classify (\emph{e.g.}, the 2nd and the 4th in Table~\ref{tab:examples1}). These samples can be considered as out-of-distribution ones, since it is difficult to obtain evidence related to classification. As expected, these samples are assigned with high uncertainty (top 5) by our model. In contrast, Table~\ref{tab:examples2} presents 5 samples with lowest uncertainty. We can find that for these samples the labels are usually explicitly mentioned in the text and the images also clearly show the characteristics of `macarons'.} 
\begin{figure}[!htbp]
\centering
\subfigure[TMC SUN]{
\centering
\begin{minipage}[t]{0.48\linewidth}
\centering
\includegraphics[width=0.9\linewidth,height=0.744\linewidth]{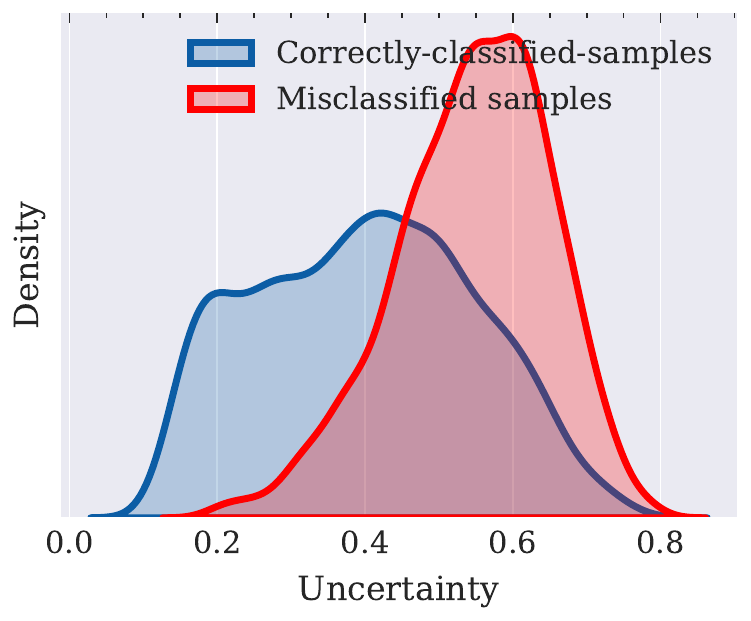}
\centering
\end{minipage}}
\subfigure[TMC NYUD]{
\begin{minipage}[t]{0.48\linewidth}
\centering
\includegraphics[width=0.9\linewidth,height=0.744\linewidth]{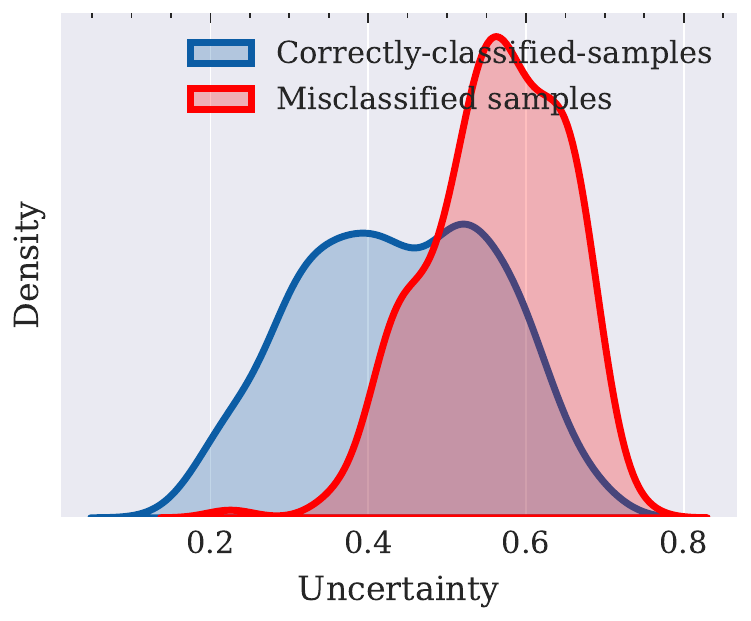}
\end{minipage}}
\subfigure[ETMC SUN]{
\centering
\begin{minipage}[t]{0.48\linewidth}
\centering
\includegraphics[width=0.9\linewidth,height=0.744\linewidth]{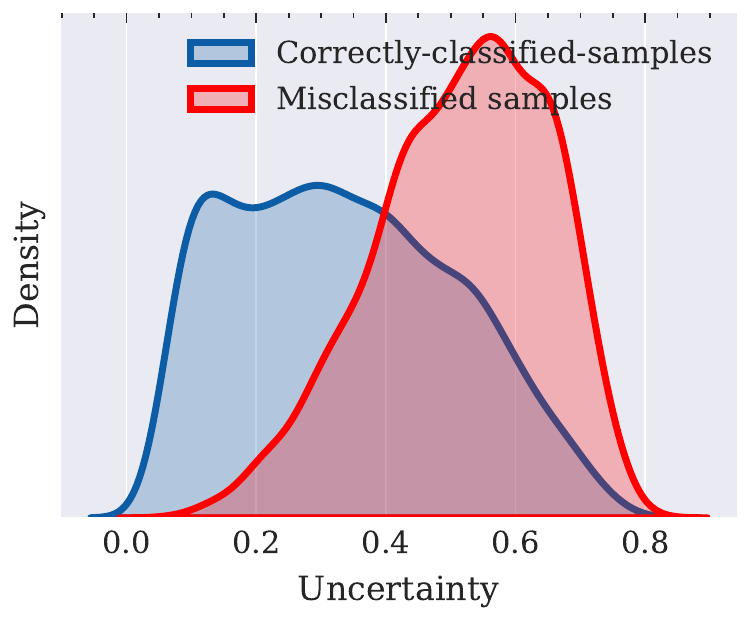}
\centering
\end{minipage}}
\subfigure[ETMC NYUD]{
\begin{minipage}[t]{0.48\linewidth}
\centering
\includegraphics[width=0.9\linewidth,height=0.744\linewidth]{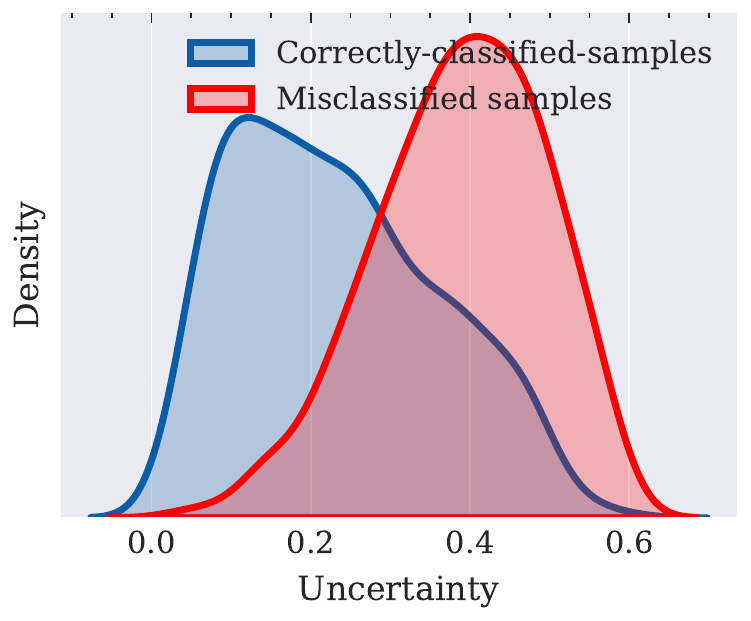}
\end{minipage}}
\caption{Uncertainty density of correctly classified samples and misclassified samples on RGB-D scene recognition test data.}
\label{fig:classification_density}
\end{figure}

To analyze the uncertainty of all test samples, we present the uncertainty density of the test samples according to classification results in Fig.~\ref{fig:classification_density}. It can be observed that the uncertainties of the correctly-classified samples are relatively low, while the uncertainties of the missclassified samples are relatively high.

\begin{figure*}[!t]
\centering
\subfigure[\normalsize{Confident sample 1 (Furniture store)}]{
\centering
\begin{minipage}[t]{0.32\linewidth}
\centering
\includegraphics[width=0.9\linewidth,height=0.45\linewidth]{./fig/most_confident/SUNRGBD_ETMC/confident_image-1.jpg}
\centering
\end{minipage}}
\subfigure[\normalsize{Confident sample 2 (Bathroom)}]{
\centering
\begin{minipage}[t]{0.32\linewidth}
\centering
\includegraphics[width=0.9\linewidth,height=0.45\linewidth]{./fig/most_confident/SUNRGBD_ETMC/confident_image-2.jpg}
\centering
\end{minipage}}
\subfigure[\normalsize{Confident sample 3 (Furniture store)}]{
\centering
\begin{minipage}[t]{0.32\linewidth}
\centering
\includegraphics[width=0.9\linewidth,height=0.45\linewidth]{./fig/most_confident/SUNRGBD_ETMC/confident_image-3.jpg}
\centering
\end{minipage}}
\subfigure[\normalsize{Confident sample 4 (Bathroom)}]{
\centering
\begin{minipage}[t]{0.32\linewidth}
\centering
\includegraphics[width=0.9\linewidth,height=0.45\linewidth]{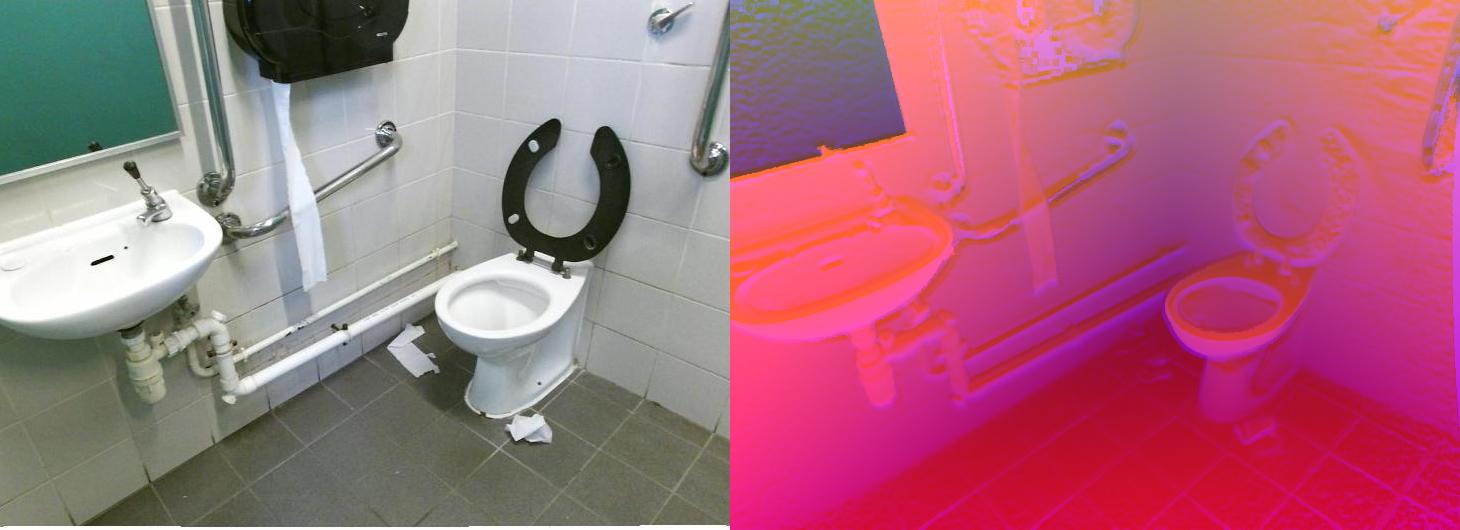}
\centering
\end{minipage}}
\subfigure[\normalsize{Confident sample 5 (Furniture store)}]{
\centering
\begin{minipage}[t]{0.32\linewidth}
\centering
\includegraphics[width=0.9\linewidth,height=0.45\linewidth]{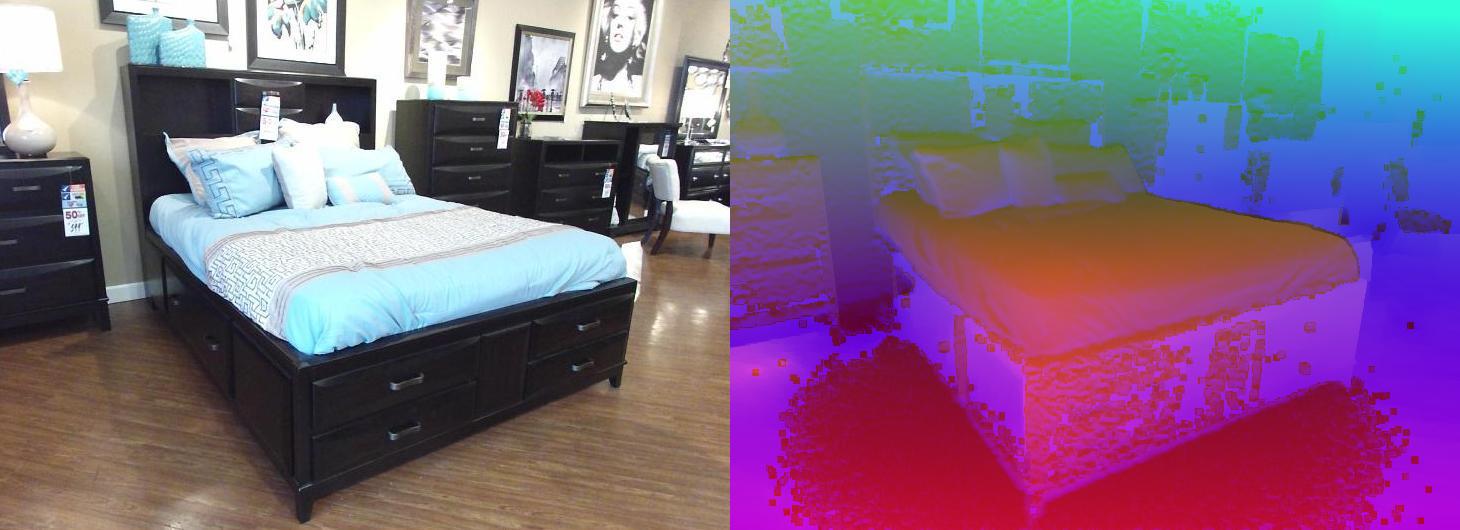}
\centering
\end{minipage}}
\subfigure[\normalsize{Confident sample 6 (Furniture store)}]{
\centering
\begin{minipage}[t]{0.32\linewidth}
\centering
\includegraphics[width=0.9\linewidth,height=0.45\linewidth]{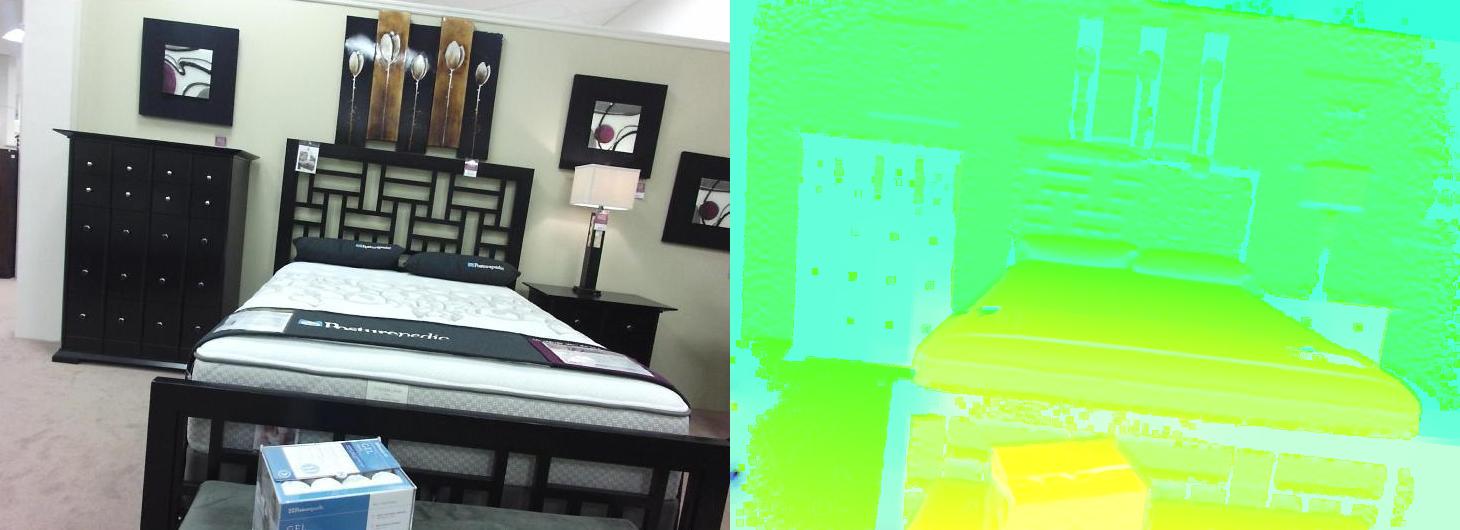}
\centering
\end{minipage}}
\subfigure[\normalsize{Confident sample 7 (Furniture store)}]{
\centering
\begin{minipage}[t]{0.32\linewidth}
\centering
\includegraphics[width=0.9\linewidth,height=0.45\linewidth]{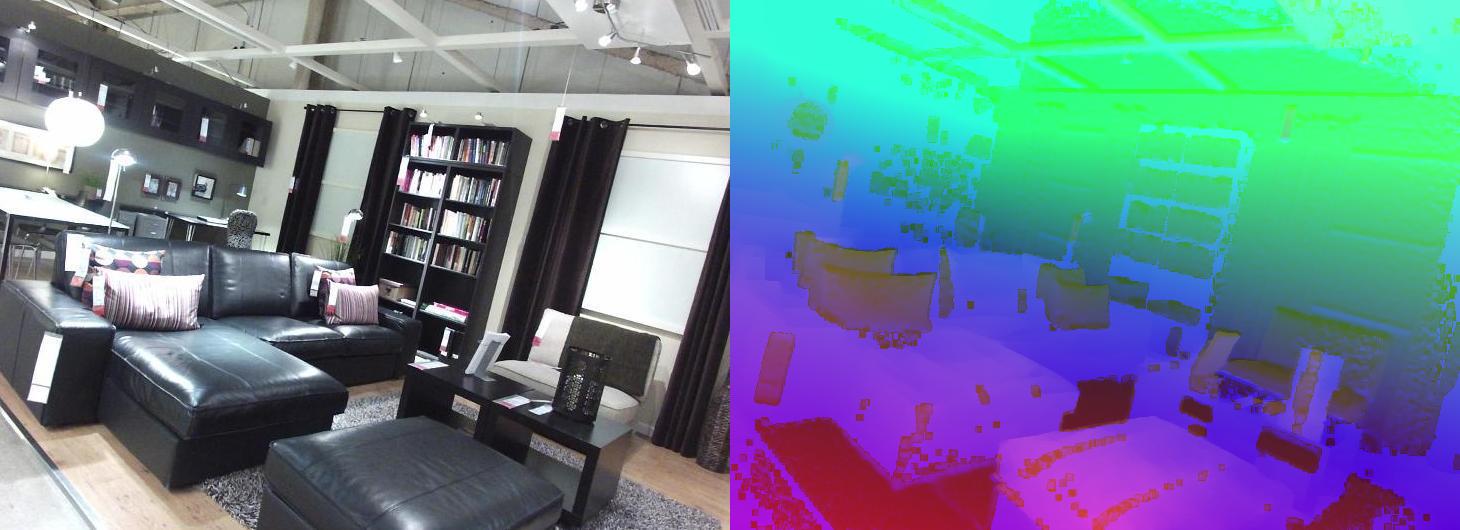}
\centering
\end{minipage}}
\subfigure[\normalsize{Confident sample 8 (Furniture store)}]{
\centering
\begin{minipage}[t]{0.32\linewidth}
\centering
\includegraphics[width=0.9\linewidth,height=0.45\linewidth]{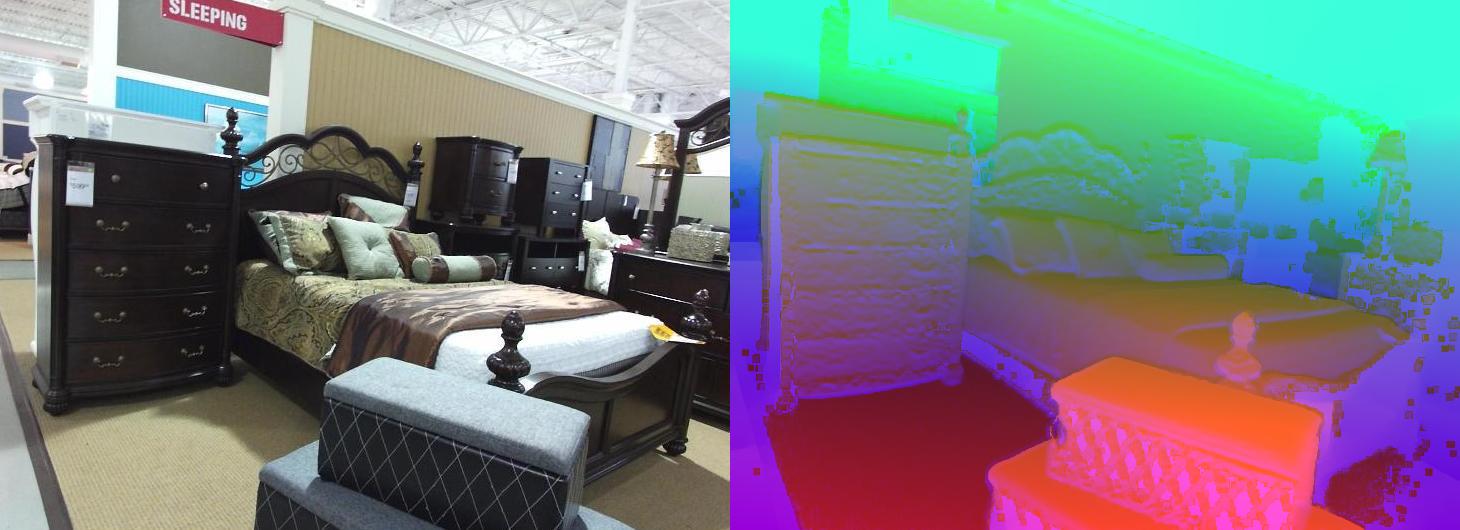}
\centering
\end{minipage}}
\subfigure[\normalsize{Confident sample 9 (Bathroom)}]{
\centering
\begin{minipage}[t]{0.32\linewidth}
\centering
\includegraphics[width=0.9\linewidth,height=0.45\linewidth]{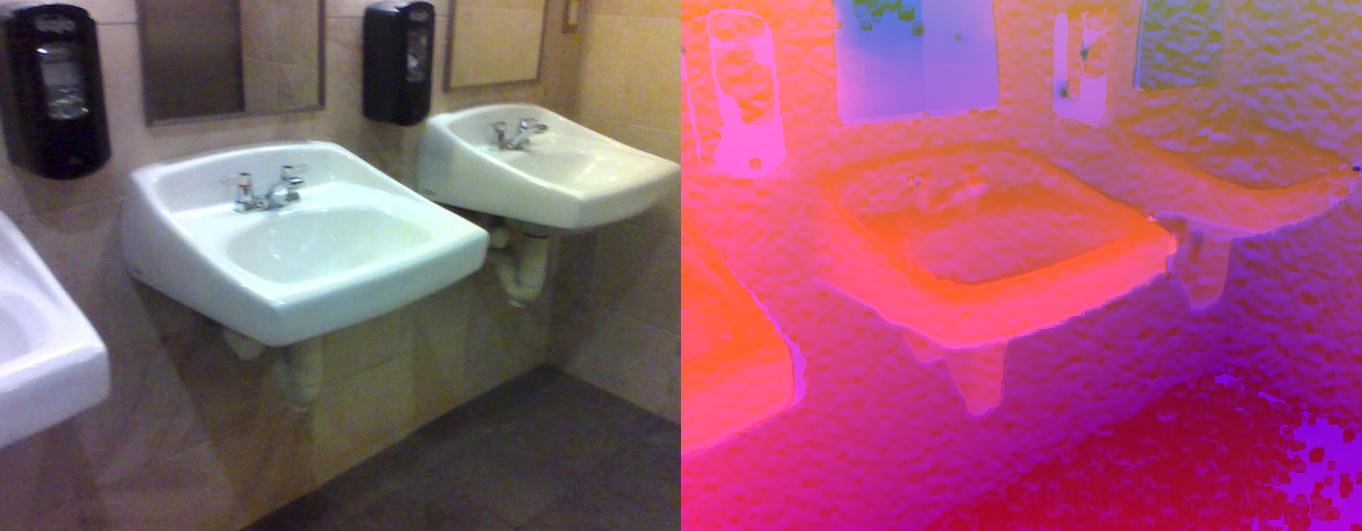}
\centering
\end{minipage}}
\subfigure[\normalsize{Confident sample 10 (Bathroom)}]{
\centering
\begin{minipage}[t]{0.32\linewidth}
\centering
\includegraphics[width=0.9\linewidth,height=0.45\linewidth]{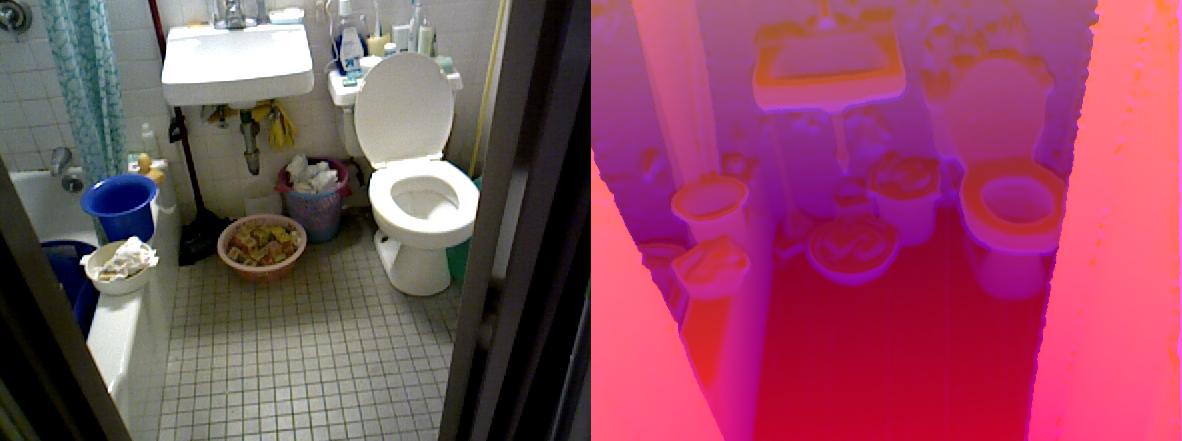}
\centering
\end{minipage}}
\subfigure[\normalsize{Confident sample 11 (Furniture store)}]{
\centering
\begin{minipage}[t]{0.32\linewidth}
\centering
\includegraphics[width=0.9\linewidth,height=0.45\linewidth]{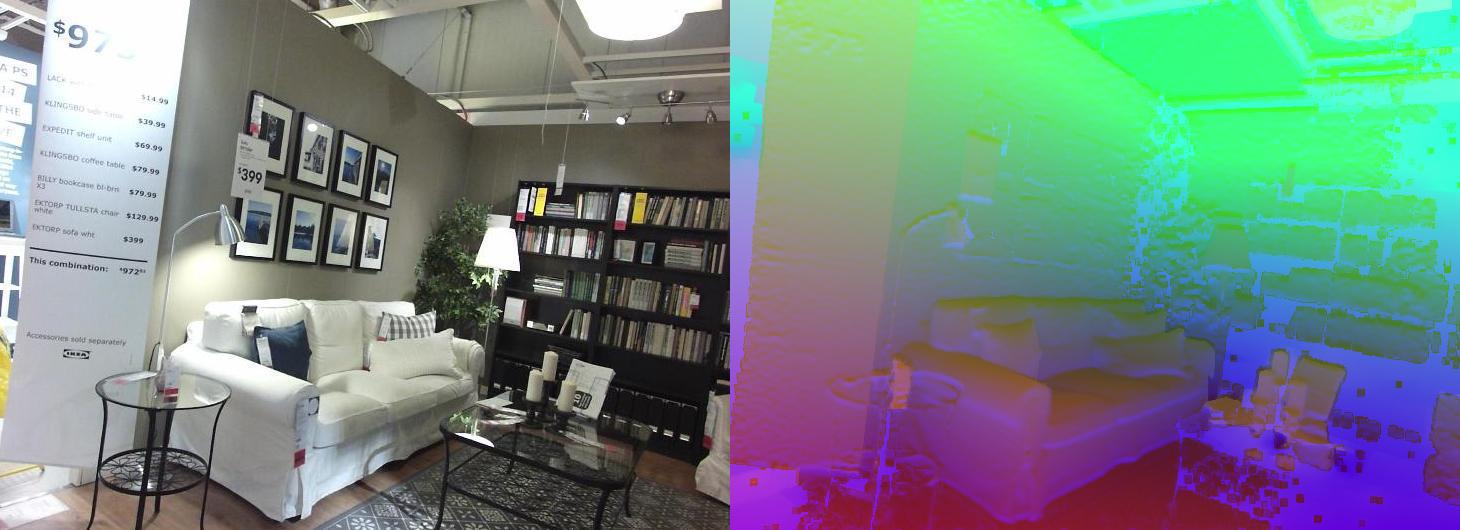}
\centering
\end{minipage}}
\subfigure[\normalsize{Confident sample 12 (Furniture store)}]{
\centering
\begin{minipage}[t]{0.32\linewidth}
\centering
\includegraphics[width=0.9\linewidth,height=0.45\linewidth]{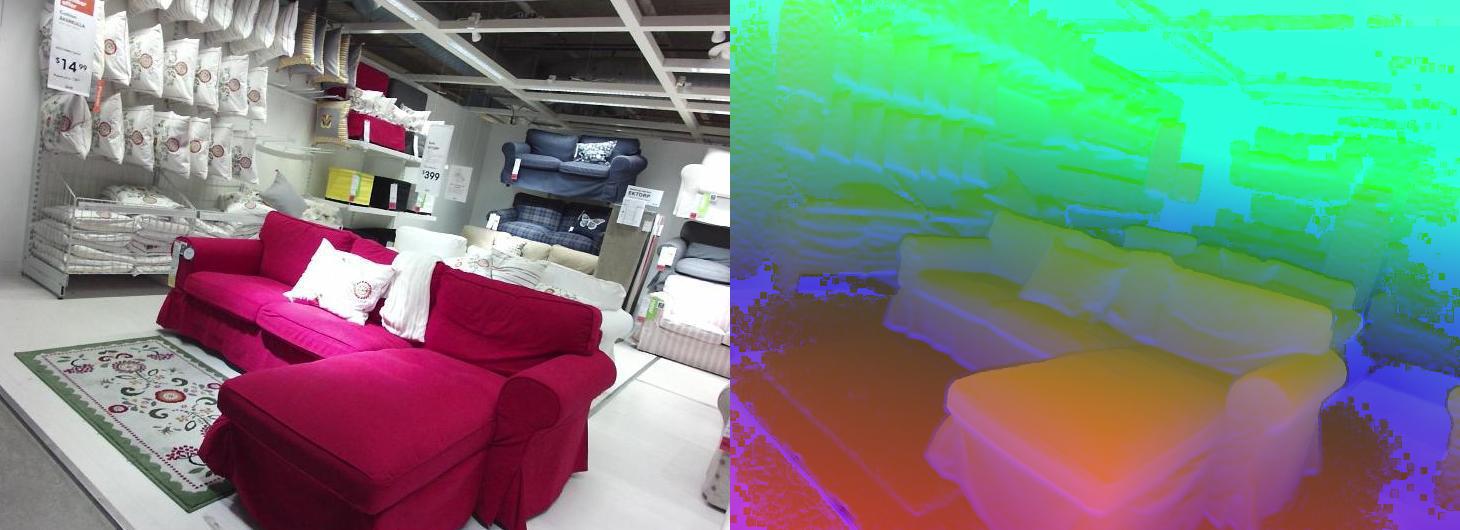}
\centering
\end{minipage}}
\subfigure[\normalsize{Confident sample 13 (Furniture store)}]{
\centering
\begin{minipage}[t]{0.32\linewidth}
\centering
\includegraphics[width=0.9\linewidth,height=0.45\linewidth]{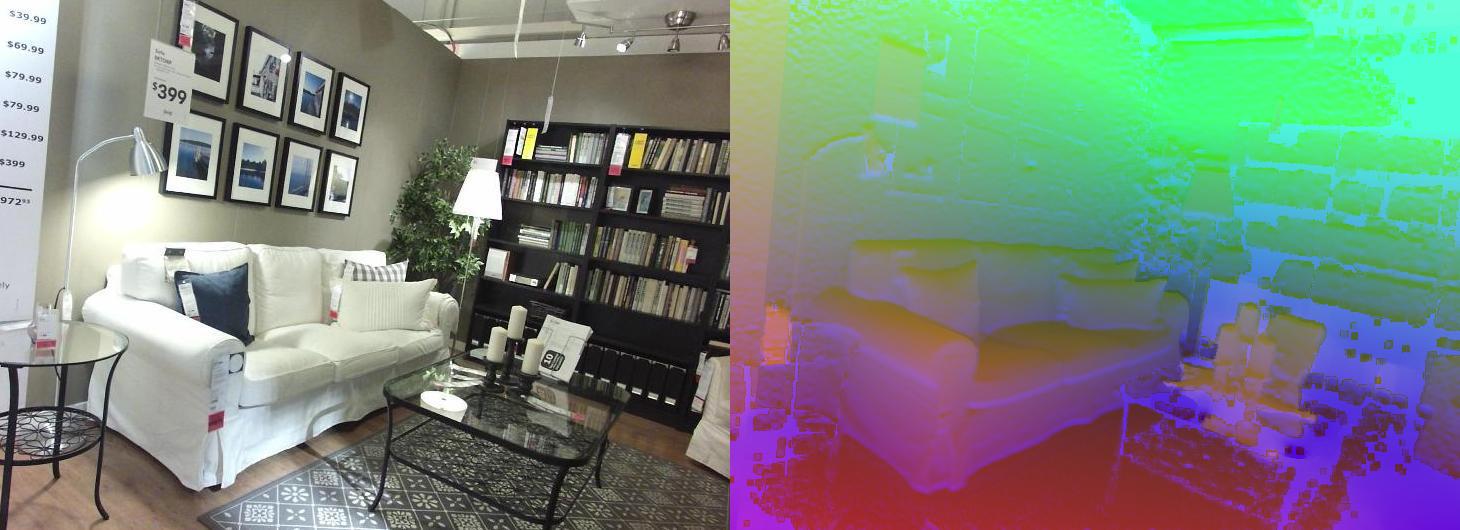}
\centering
\end{minipage}}
\subfigure[\normalsize{Confident sample 14 (Furniture store)}]{
\centering
\begin{minipage}[t]{0.32\linewidth}
\centering
\includegraphics[width=0.9\linewidth,height=0.45\linewidth]{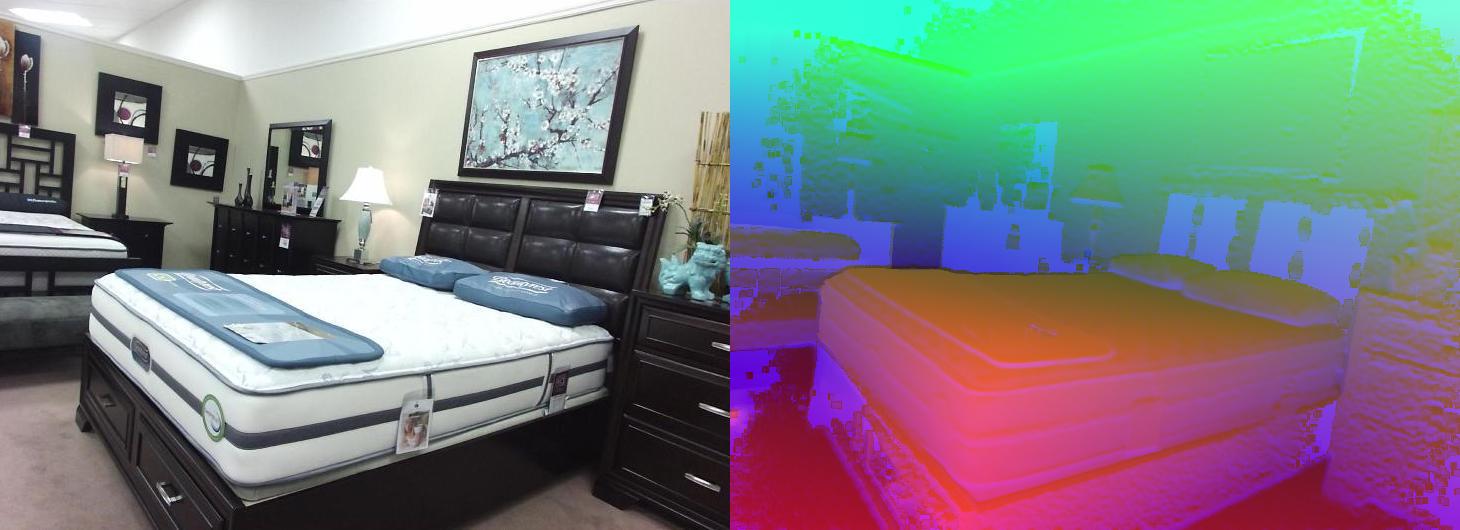}
\centering
\end{minipage}}
\subfigure[\normalsize{Confident sample 15 (Furniture store)}]{
\centering
\begin{minipage}[t]{0.32\linewidth}
\centering
\includegraphics[width=0.9\linewidth,height=0.45\linewidth]{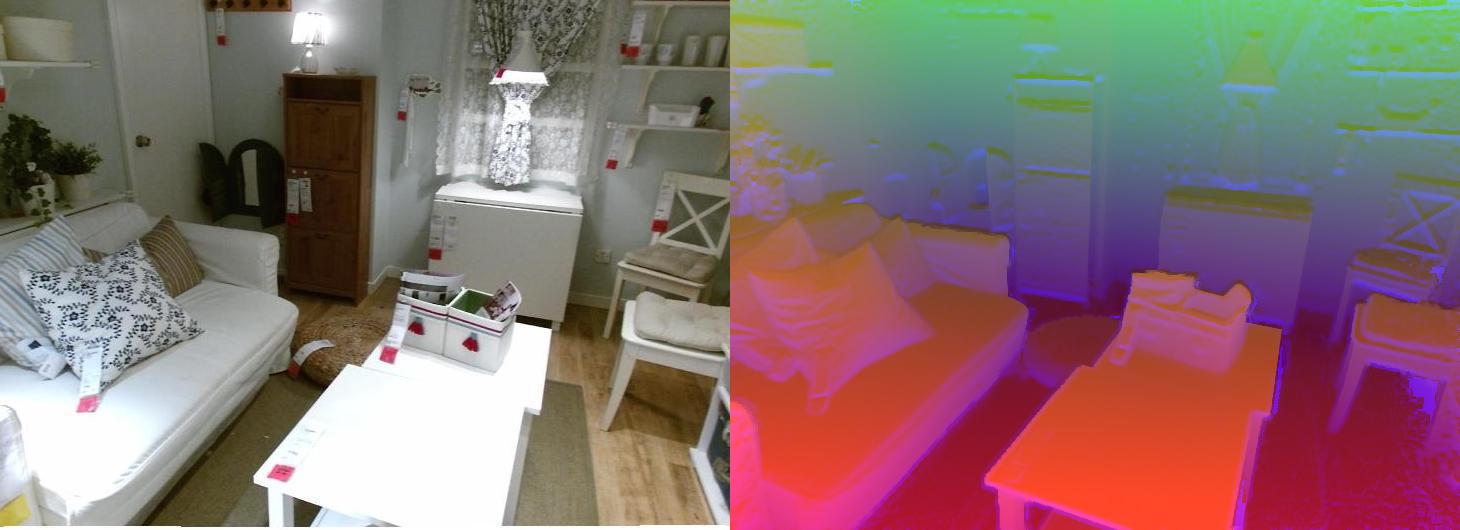}
\centering
\end{minipage}}
\subfigure[\normalsize{Confident sample 16 (Furniture store)}]{
\centering
\begin{minipage}[t]{0.32\linewidth}
\centering
\includegraphics[width=0.9\linewidth,height=0.45\linewidth]{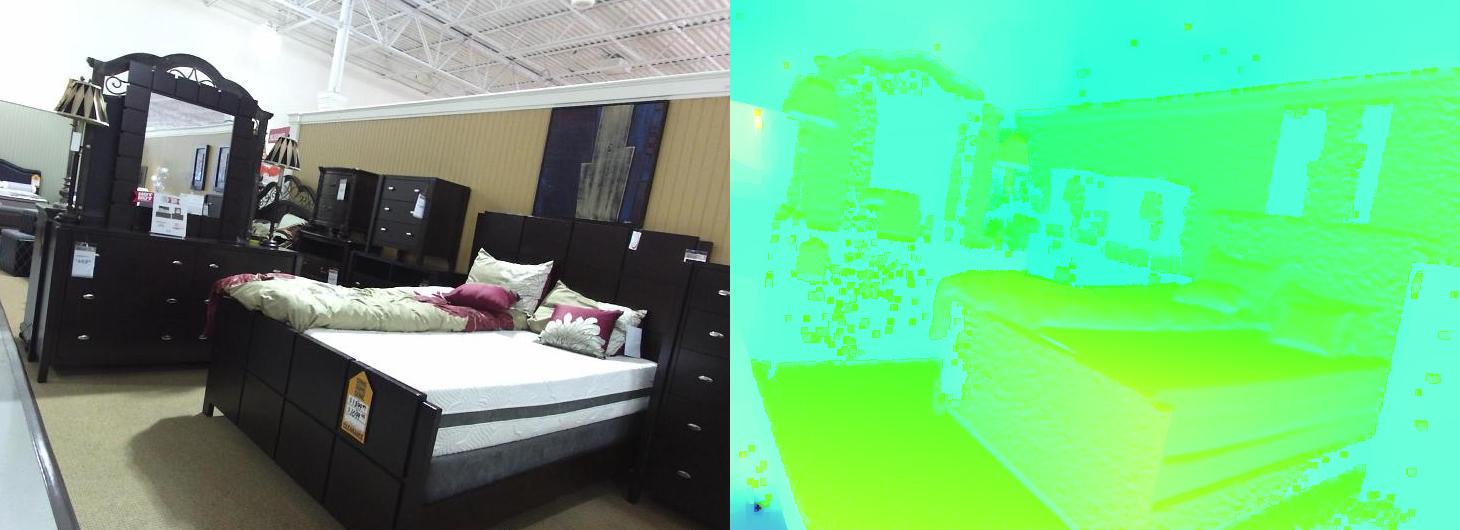}
\centering
\end{minipage}}
\subfigure[\normalsize{Confident sample 17 (Furniture store)}]{
\centering
\begin{minipage}[t]{0.32\linewidth}
\centering
\includegraphics[width=0.9\linewidth,height=0.45\linewidth]{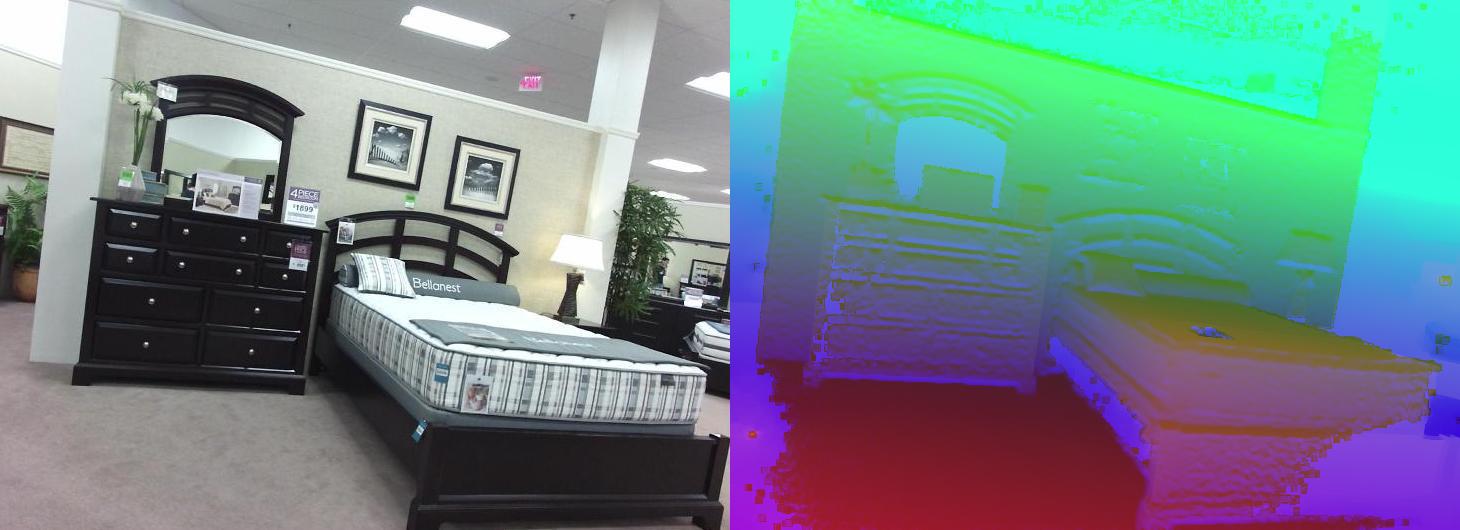}
\centering
\end{minipage}}
\subfigure[\normalsize{Confident sample 18 (Bathroom)}]{
\centering
\begin{minipage}[t]{0.32\linewidth}
\centering
\includegraphics[width=0.9\linewidth,height=0.45\linewidth]{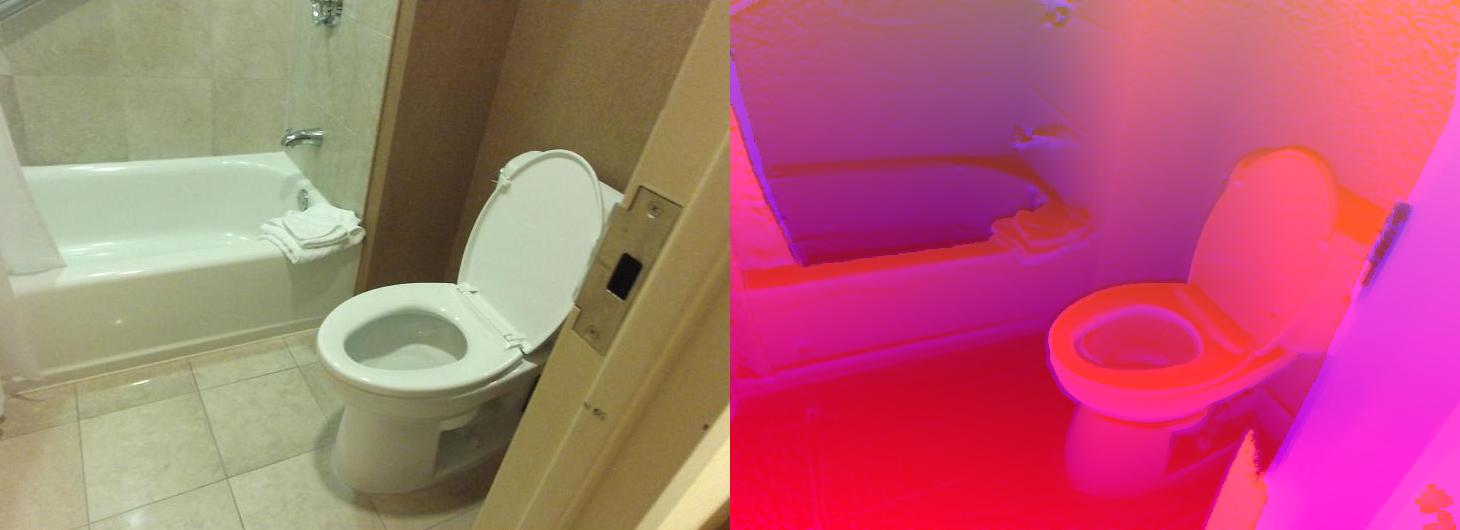}
\centering
\end{minipage}}
\caption{Examples with the lowest (top 18) uncertainty using the ETMC algorithm on SUN RGB-D test data. The ground-truth labels are in brackets. All the above examples are correctly classified.}
\label{fig:ETMC_SUN_vis_more_confident}
\end{figure*}

\begin{figure*}[!t]
\centering
\subfigure[\textcolor{red}{\normalsize{Uncertain sample 1 (Discussion area)}}]{
\centering
\begin{minipage}[t]{0.32\linewidth}
\centering
\includegraphics[width=0.9\linewidth,height=0.45\linewidth]{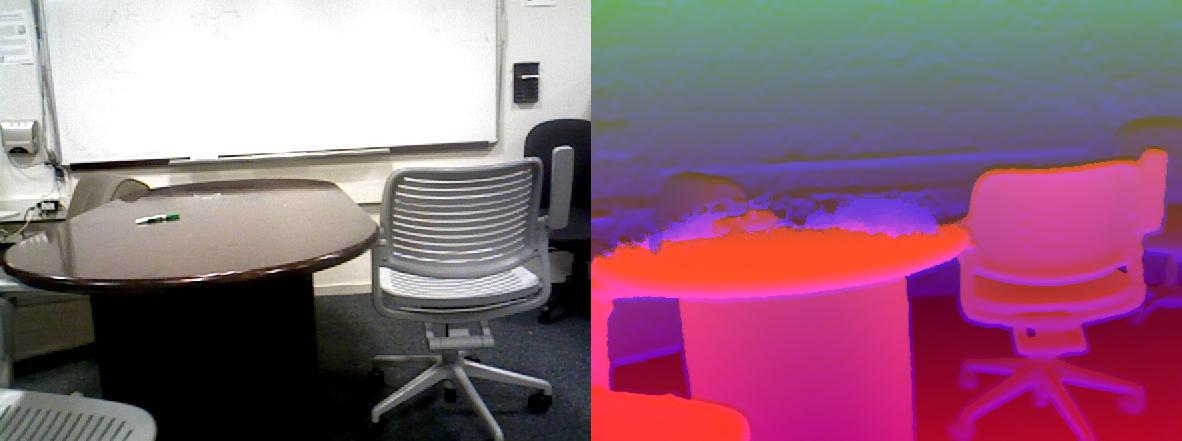}
\centering
\end{minipage}}
\subfigure[\textcolor{red}{\normalsize{Uncertain sample 2 (Office)}}]{
\centering
\begin{minipage}[t]{0.32\linewidth}
\centering
\includegraphics[width=0.9\linewidth,height=0.45\linewidth]{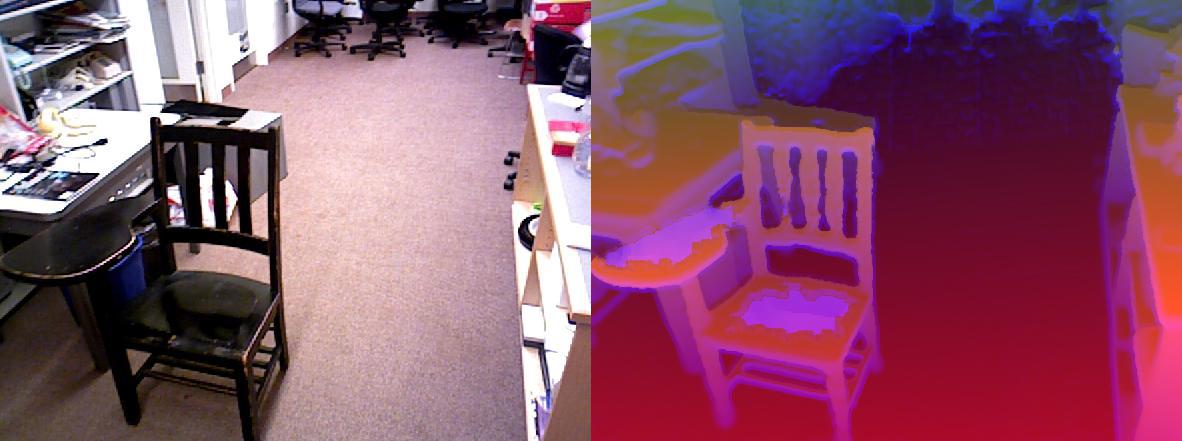}
\centering
\end{minipage}}
\subfigure[\normalsize{Uncertain sample 3 (Rest space)}]{
\centering
\begin{minipage}[t]{0.32\linewidth}
\centering
\includegraphics[width=0.9\linewidth,height=0.45\linewidth]{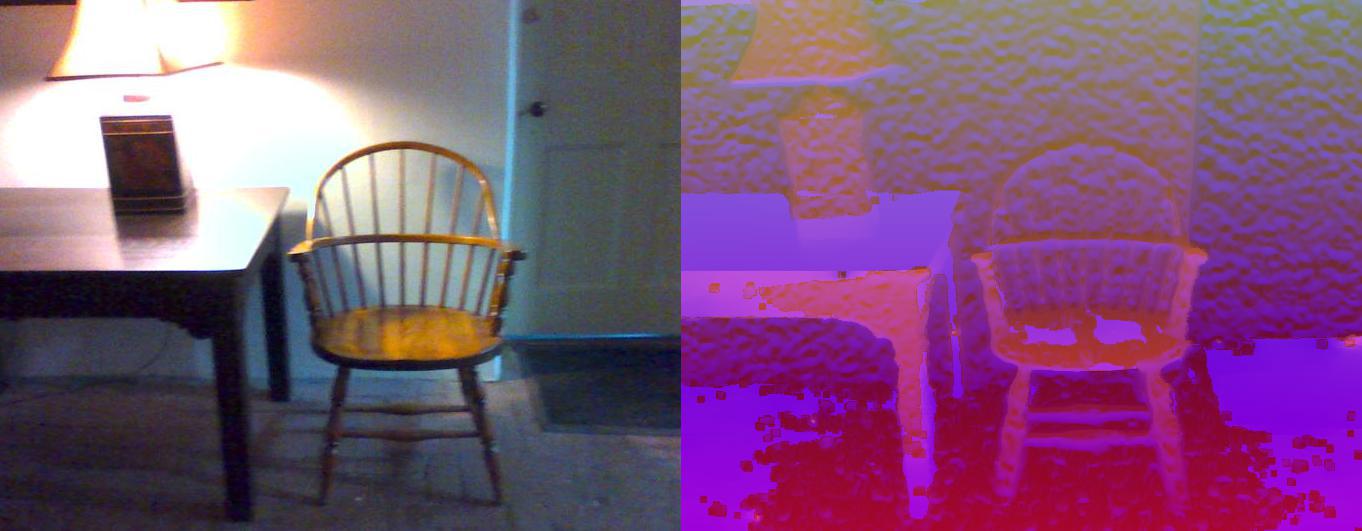}
\centering
\end{minipage}}
\subfigure[\textcolor{red}{\normalsize{Uncertain sample 4 (Dining area)}}]{
\centering
\begin{minipage}[t]{0.32\linewidth}
\centering
\includegraphics[width=0.9\linewidth,height=0.45\linewidth]{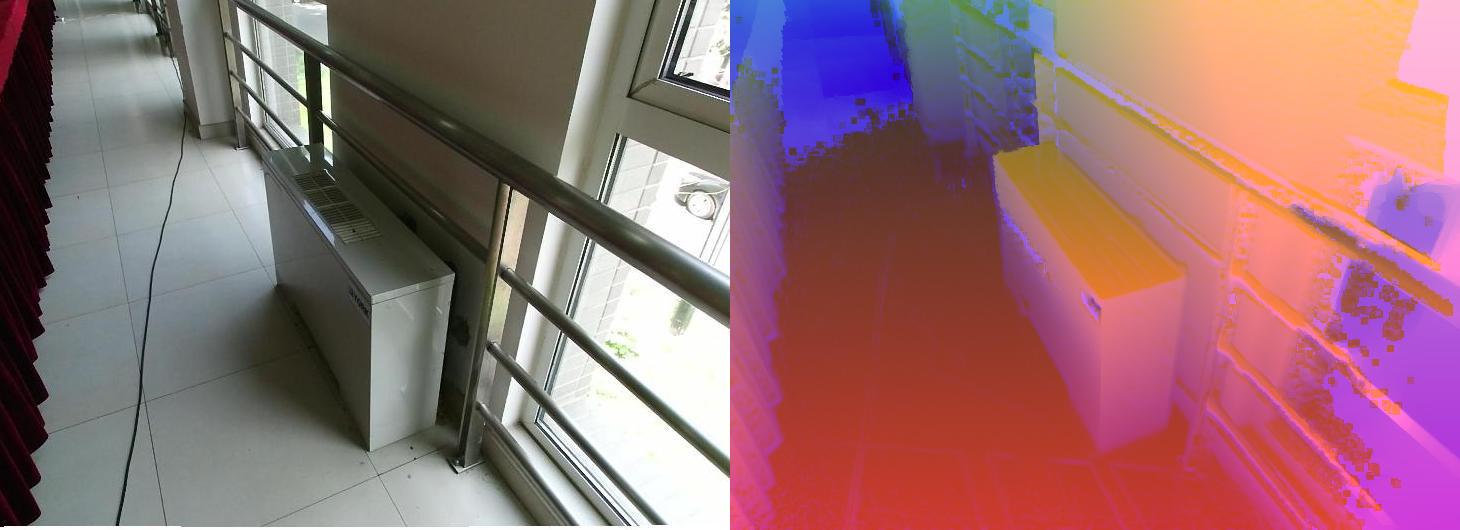}
\centering
\end{minipage}}
\subfigure[\textcolor{red}{\normalsize{Uncertain sample 5 (Study space)}}]{
\centering
\begin{minipage}[t]{0.32\linewidth}
\centering
\includegraphics[width=0.9\linewidth,height=0.45\linewidth]{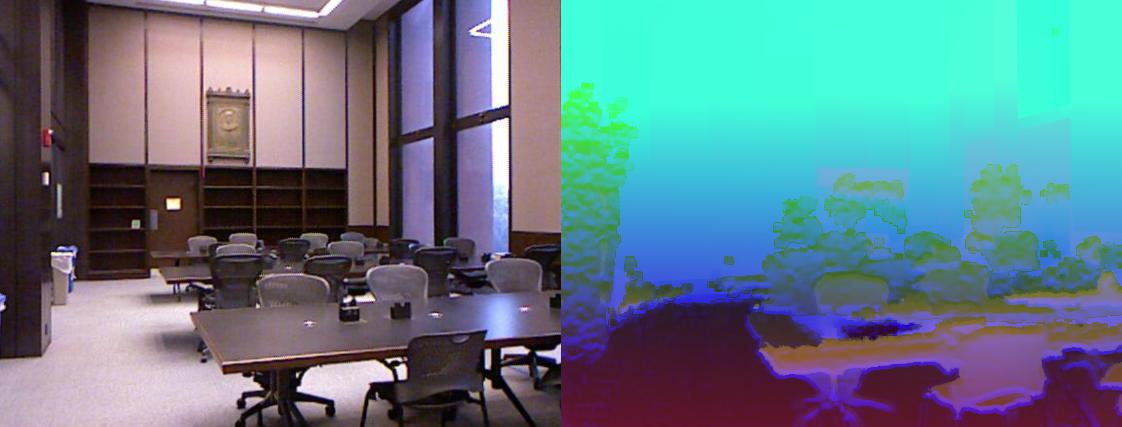}
\centering
\end{minipage}}
\subfigure[\textcolor{red}{\normalsize{Uncertain sample 6 (Bedroom)}}]{
\centering
\begin{minipage}[t]{0.32\linewidth}
\centering
\includegraphics[width=0.9\linewidth,height=0.45\linewidth]{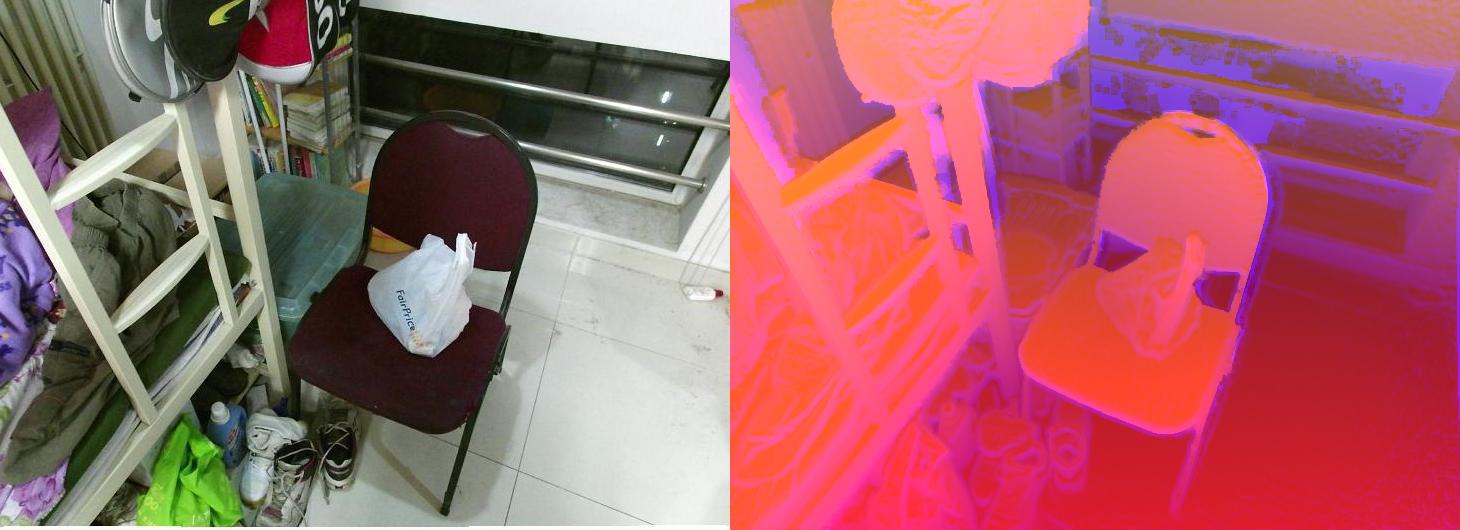}
\centering
\end{minipage}}
\subfigure[\textcolor{red}{\normalsize{Uncertain sample 7 (Office)}}]{
\centering
\begin{minipage}[t]{0.32\linewidth}
\centering
\includegraphics[width=0.9\linewidth,height=0.45\linewidth]{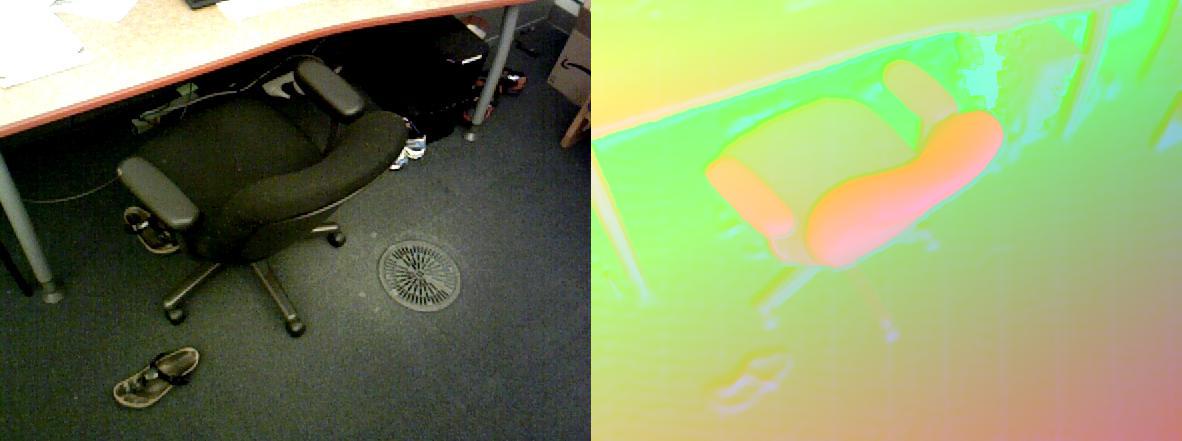}
\centering
\end{minipage}}
\subfigure[\textcolor{red}{\normalsize{Uncertain sample 8 (Conference room)}}]{
\centering
\begin{minipage}[t]{0.32\linewidth}
\centering
\includegraphics[width=0.9\linewidth,height=0.45\linewidth]{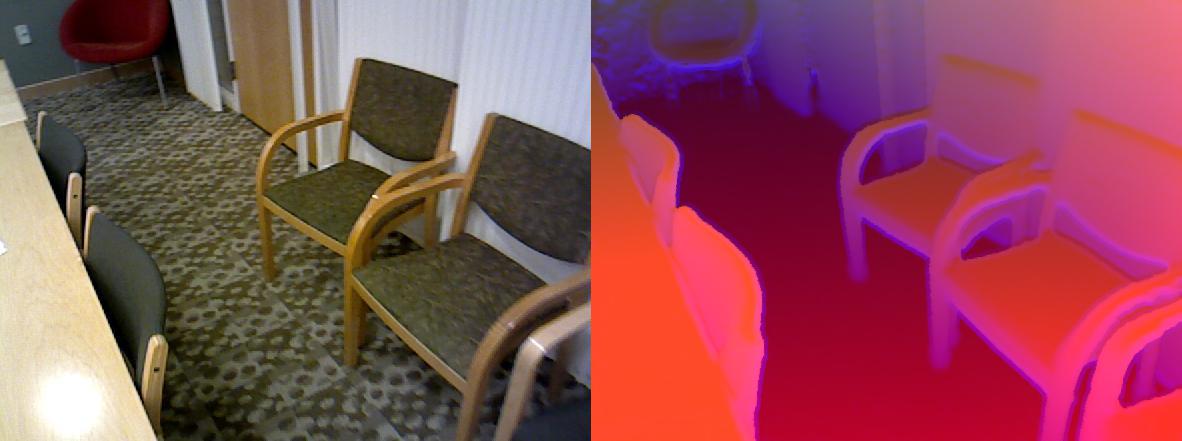}
\centering
\end{minipage}}
\subfigure[\textcolor{red}{\normalsize{Uncertain sample 9 (Classroom)}}]{
\centering
\begin{minipage}[t]{0.32\linewidth}
\centering
\includegraphics[width=0.9\linewidth,height=0.45\linewidth]{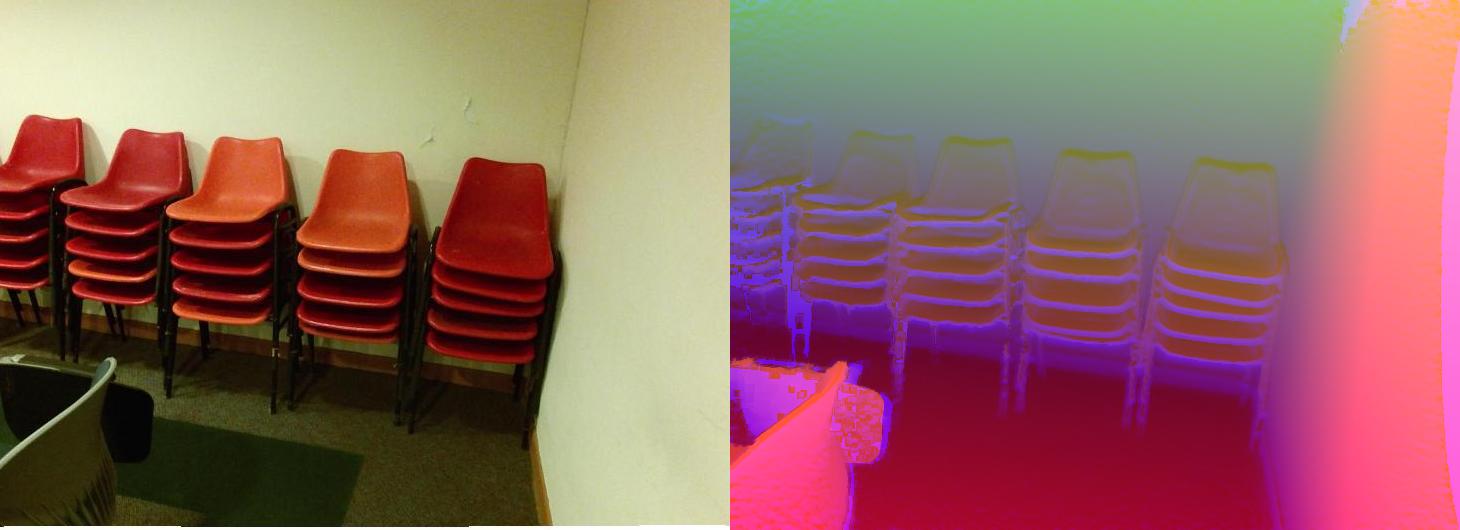}
\centering
\end{minipage}}
\subfigure[\textcolor{red}{\normalsize{Uncertain sample 10 (Conference room)}}]{
\centering
\begin{minipage}[t]{0.32\linewidth}
\centering
\includegraphics[width=0.9\linewidth,height=0.45\linewidth]{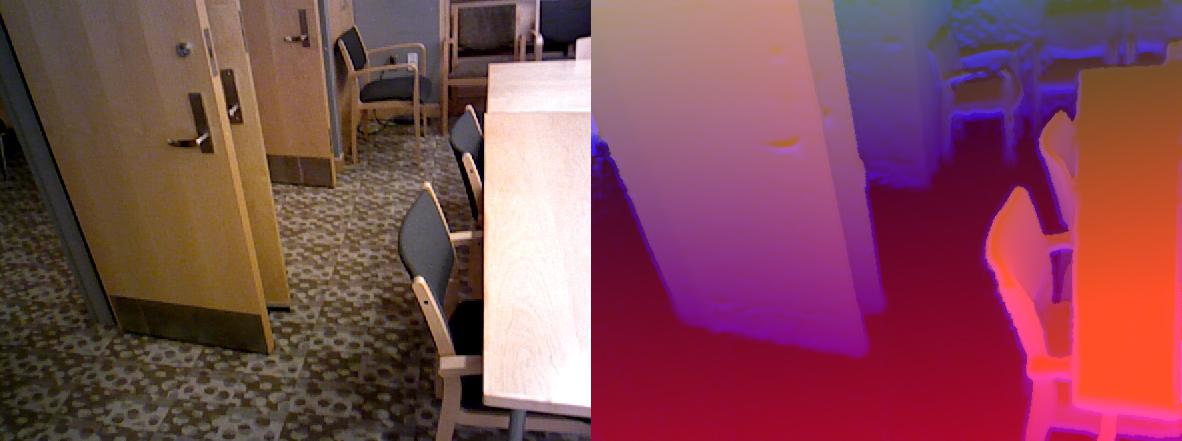}
\centering
\end{minipage}}
\subfigure[\textcolor{red}{\normalsize{Uncertain sample 11 (Rest space)}}]{
\centering
\begin{minipage}[t]{0.32\linewidth}
\centering
\includegraphics[width=0.9\linewidth,height=0.45\linewidth]{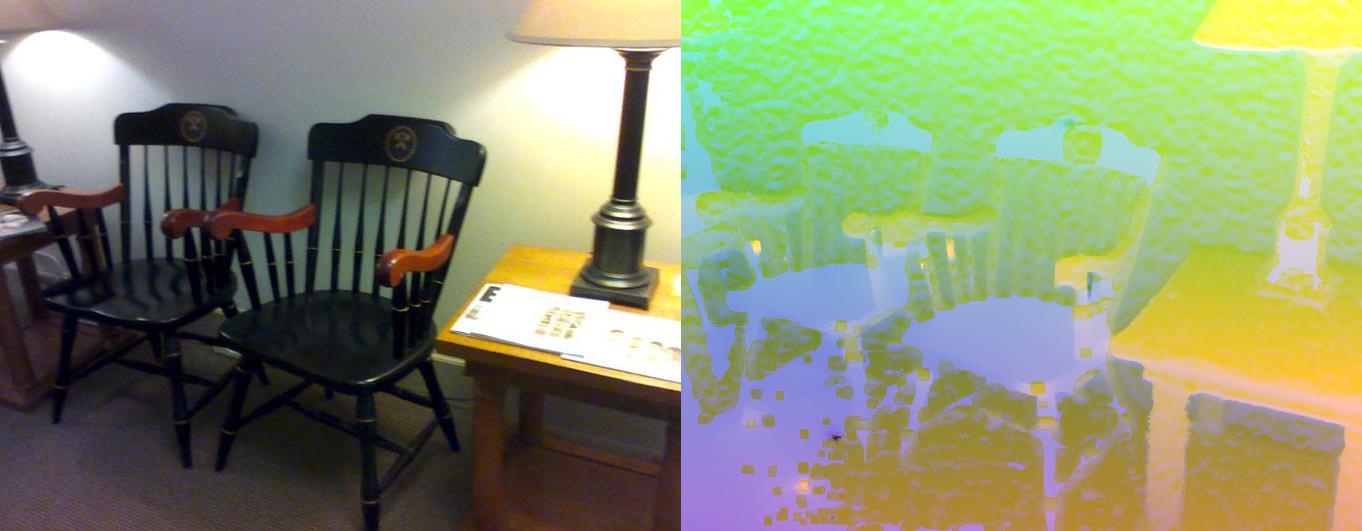}
\centering
\end{minipage}}
\subfigure[\textcolor{red}{\normalsize{Uncertain sample 12 (Lab)}}]{
\centering
\begin{minipage}[t]{0.32\linewidth}
\centering
\includegraphics[width=0.9\linewidth,height=0.45\linewidth]{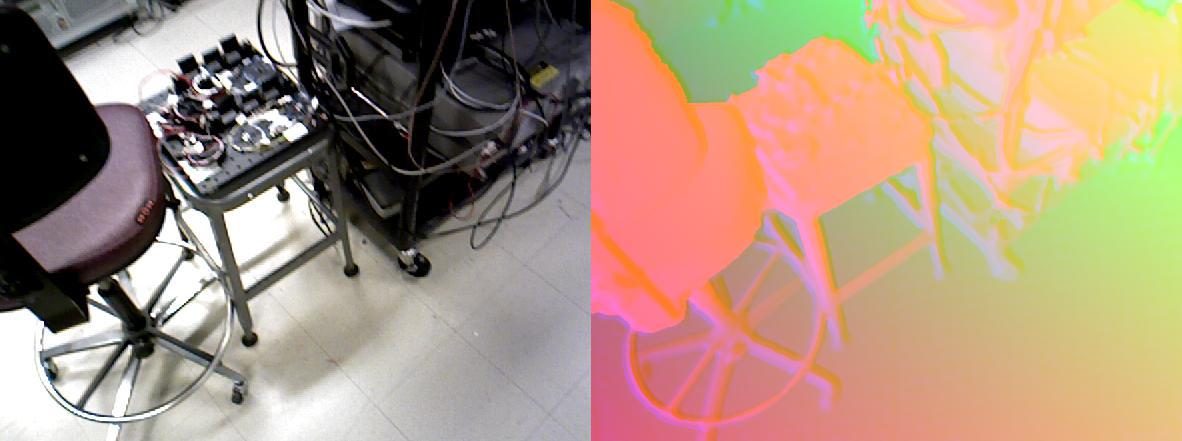}
\centering
\end{minipage}}
\subfigure[\textcolor{red}{\normalsize{Uncertain sample 13 (Study space)}}]{
\centering
\begin{minipage}[t]{0.32\linewidth}
\centering
\includegraphics[width=0.9\linewidth,height=0.45\linewidth]{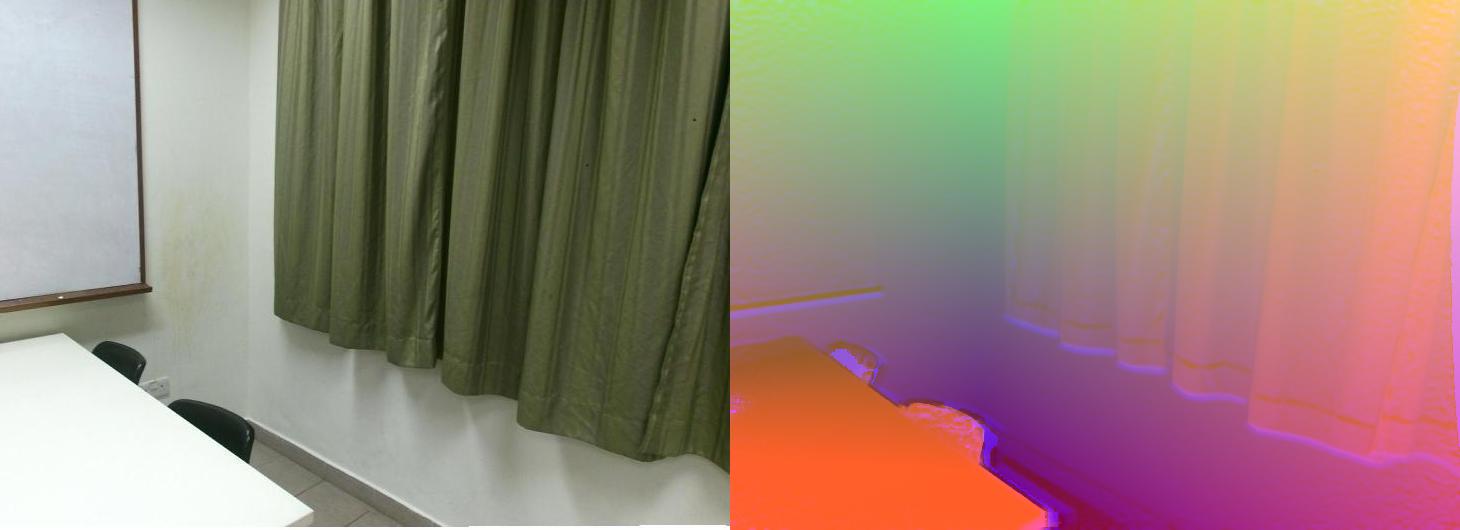}
\centering
\end{minipage}}
\subfigure[\textcolor{red}{\normalsize{Uncertain sample 14 (Rest space)}}]{
\centering
\begin{minipage}[t]{0.32\linewidth}
\centering
\includegraphics[width=0.9\linewidth,height=0.45\linewidth]{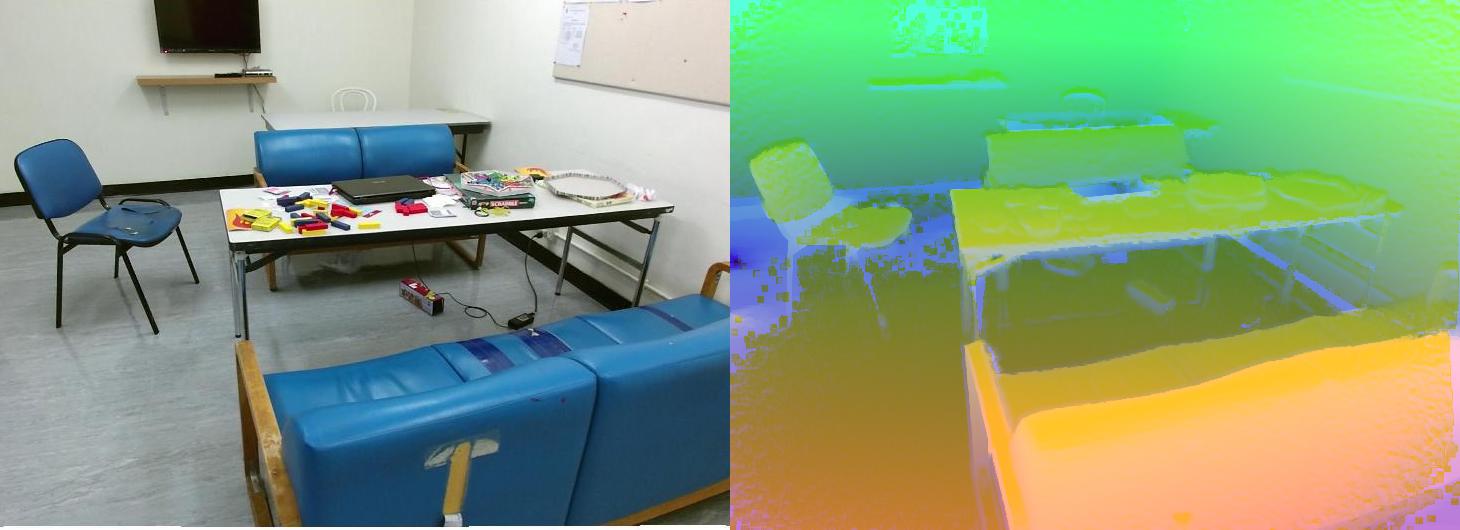}
\centering
\end{minipage}}
\subfigure[\normalsize{Uncertain sample 15 (Office)}]{
\centering
\begin{minipage}[t]{0.32\linewidth}
\centering
\includegraphics[width=0.9\linewidth,height=0.45\linewidth]{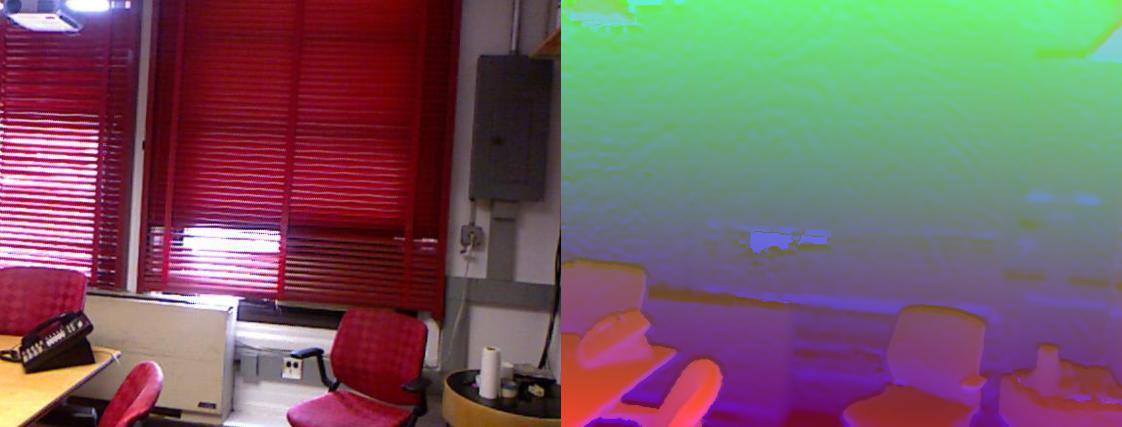}
\centering
\end{minipage}}
\subfigure[\textcolor{red}{\normalsize{Uncertain sample 16 (Lab)}}]{
\centering
\begin{minipage}[t]{0.32\linewidth}
\centering
\includegraphics[width=0.9\linewidth,height=0.45\linewidth]{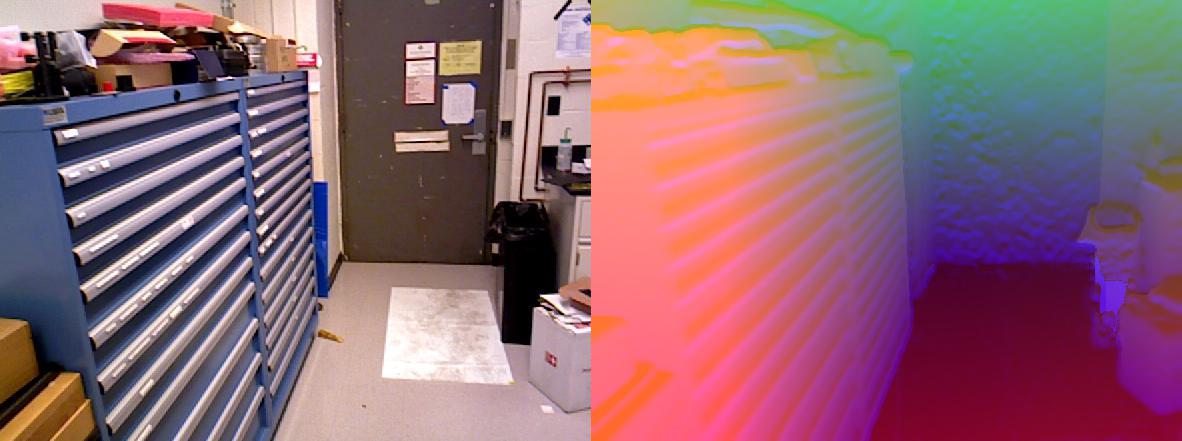}
\centering
\end{minipage}}
\subfigure[\textcolor{red}{\normalsize{Uncertain sample 17 (Classroom)}}]{
\centering
\begin{minipage}[t]{0.32\linewidth}
\centering
\includegraphics[width=0.9\linewidth,height=0.45\linewidth]{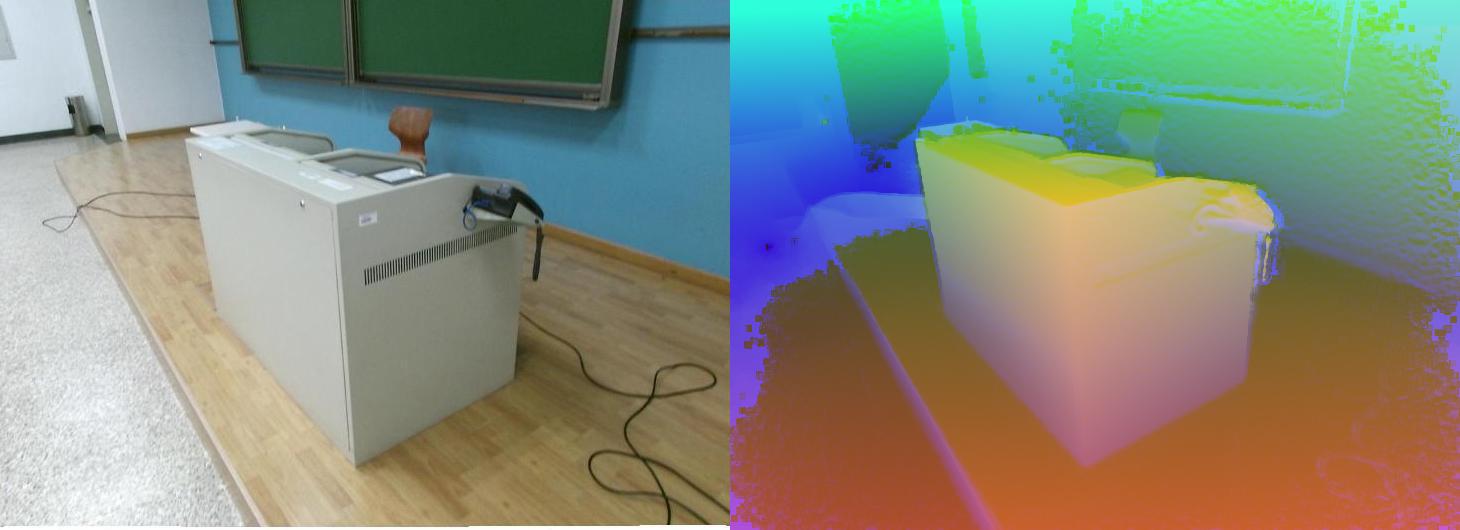}
\centering
\end{minipage}}
\subfigure[\textcolor{red}{\normalsize{Uncertain sample 18 (Study space)}}]{
\centering
\begin{minipage}[t]{0.32\linewidth}
\centering
\includegraphics[width=0.9\linewidth,height=0.45\linewidth]{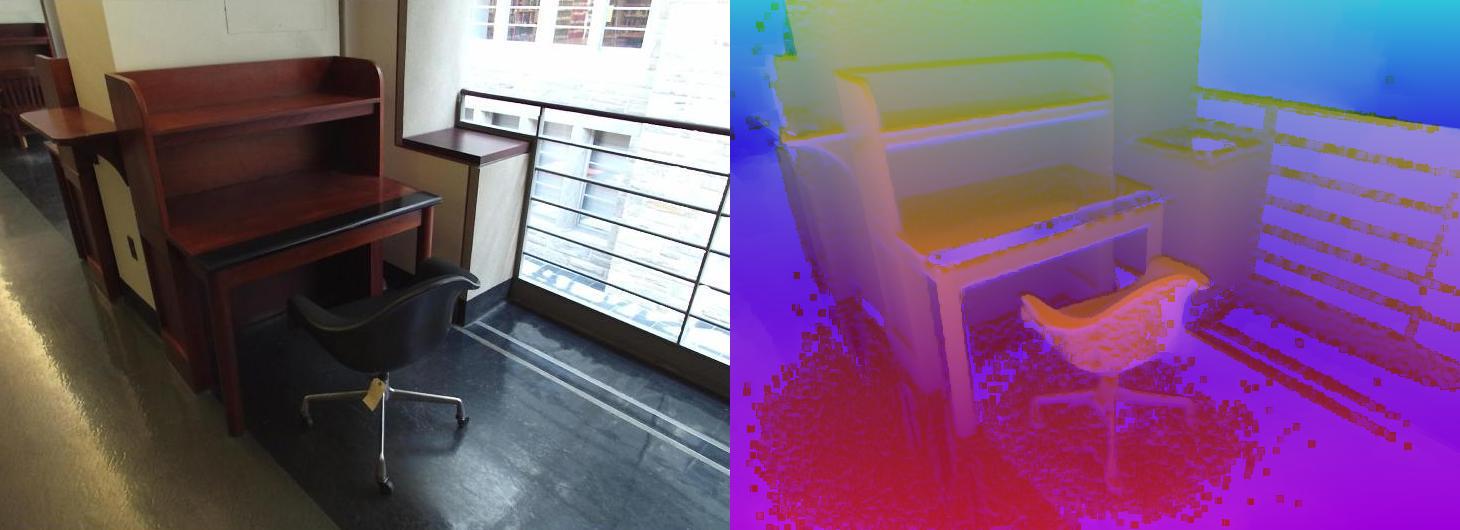}
\centering
\end{minipage}}
\caption{Examples with the highest (top 18) uncertainty using the ETMC algorithm on SUN RGB-D test data. The labels are in brackets. The text in red indicates that the sample are misclassified.}
\label{fig:ETMC_SUN_vis_more_uncertain}
\end{figure*}

\begin{figure*}[!t]
\centering
\subfigure[\normalsize{Confident sample 1 (Bedroom)}]{
\centering
\begin{minipage}[t]{0.32\linewidth}
\centering
\includegraphics[width=0.9\linewidth,height=0.45\linewidth]{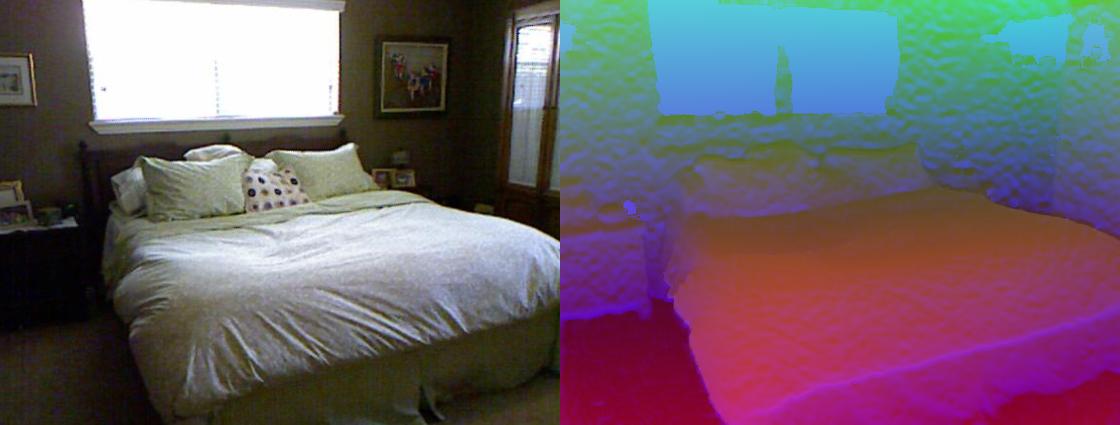}
\centering
\end{minipage}}
\subfigure[\normalsize{Confident sample 2 (Bedroom)}]{
\centering
\begin{minipage}[t]{0.32\linewidth}
\centering
\includegraphics[width=0.9\linewidth,height=0.45\linewidth]{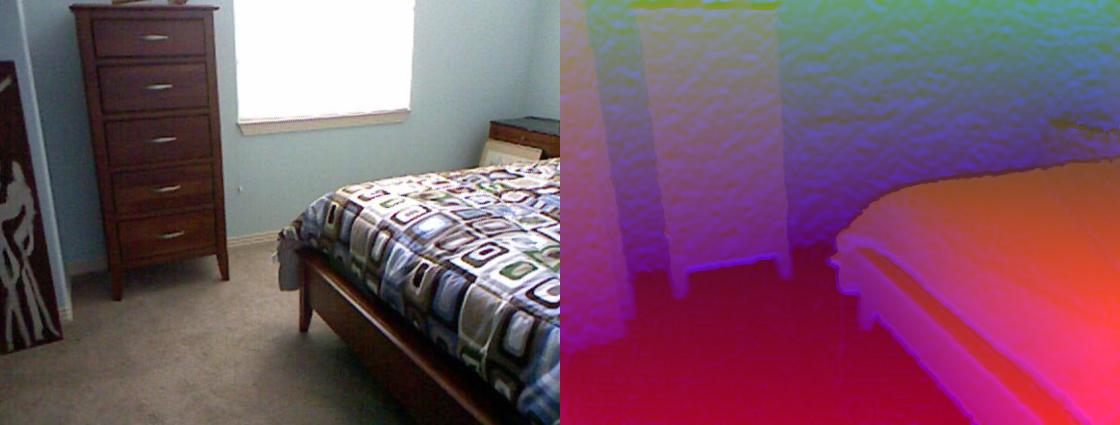}
\centering
\end{minipage}}
\subfigure[\normalsize{Confident sample 3 (Bedroom)}]{
\centering
\begin{minipage}[t]{0.32\linewidth}
\centering
\includegraphics[width=0.9\linewidth,height=0.45\linewidth]{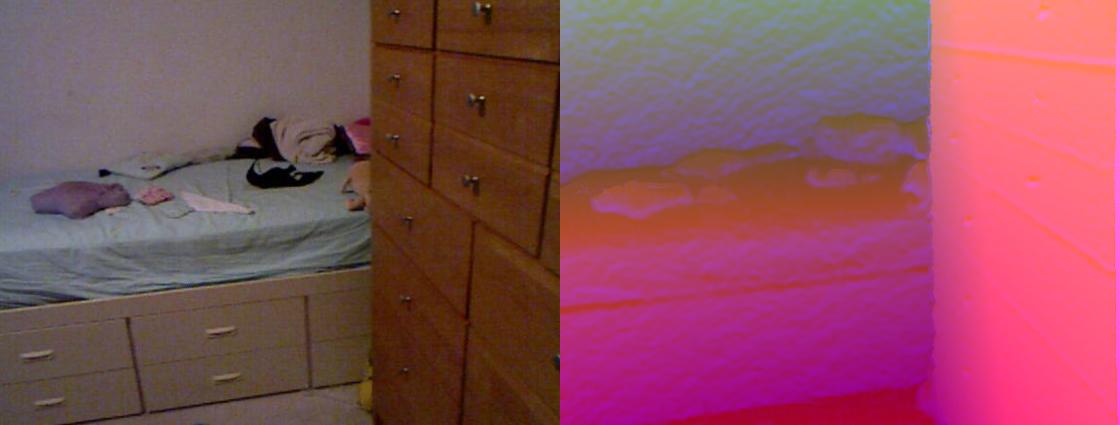}
\centering
\end{minipage}}
\subfigure[\normalsize{Confident sample 4 (Bedroom)}]{
\centering
\begin{minipage}[t]{0.32\linewidth}
\centering
\includegraphics[width=0.9\linewidth,height=0.45\linewidth]{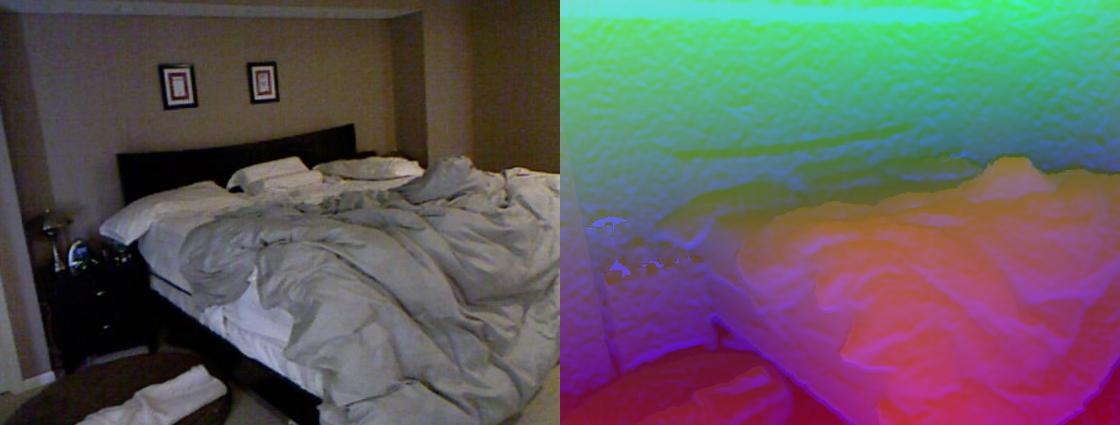}
\centering
\end{minipage}}
\subfigure[\normalsize{Confident sample 5 (Bedroom)}]{
\centering
\begin{minipage}[t]{0.32\linewidth}
\centering
\includegraphics[width=0.9\linewidth,height=0.45\linewidth]{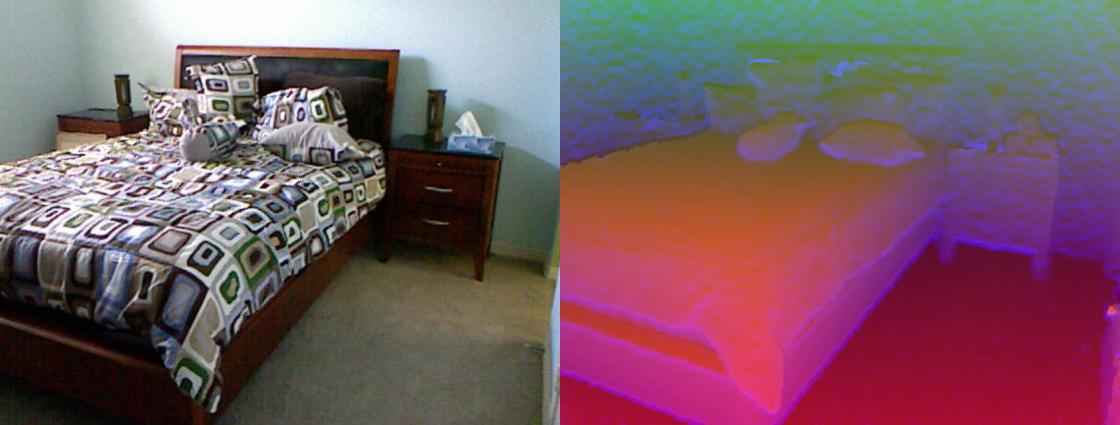}
\centering
\end{minipage}}
\subfigure[\normalsize{Confident sample 6 (Bedroom)}]{
\centering
\begin{minipage}[t]{0.32\linewidth}
\centering
\includegraphics[width=0.9\linewidth,height=0.45\linewidth]{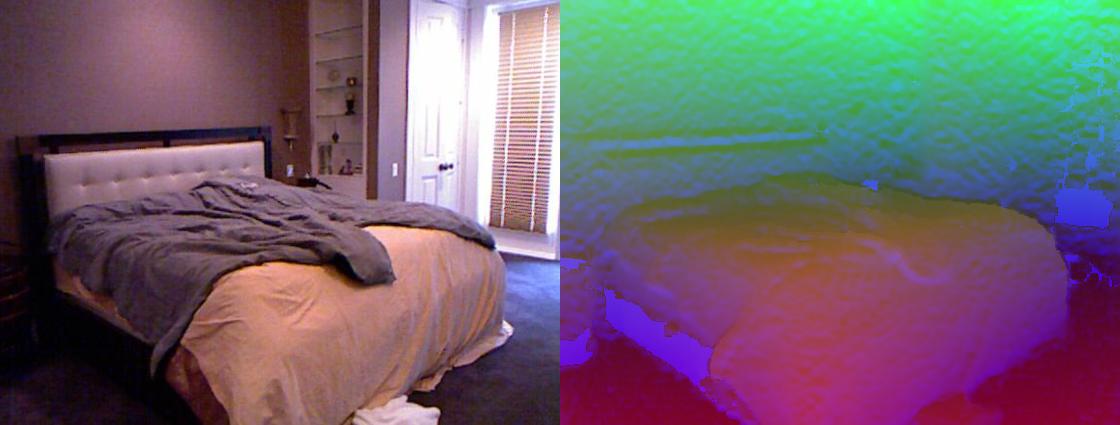}
\centering
\end{minipage}}
\subfigure[\normalsize{Confident sample 7 (Bedroom)}]{
\centering
\begin{minipage}[t]{0.32\linewidth}
\centering
\includegraphics[width=0.9\linewidth,height=0.45\linewidth]{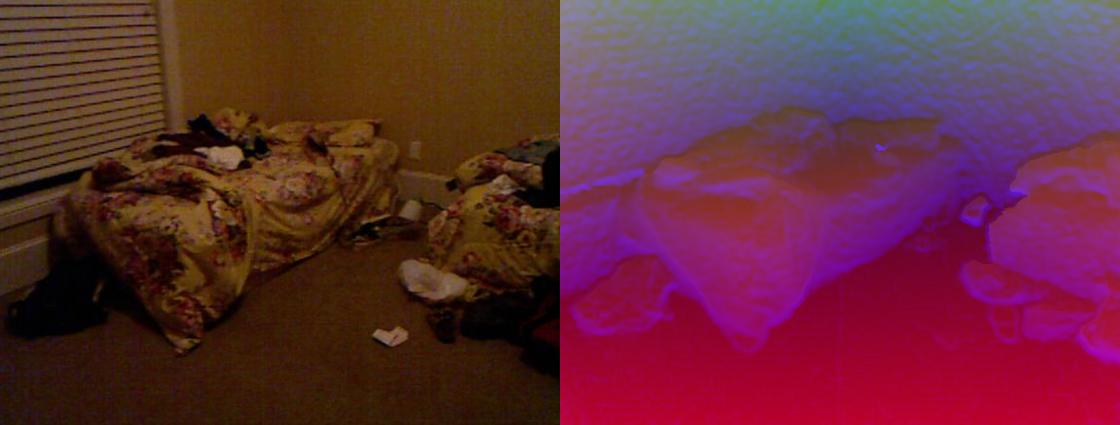}
\centering
\end{minipage}}
\subfigure[\normalsize{Confident sample 8 (Bedroom)}]{
\centering
\begin{minipage}[t]{0.32\linewidth}
\centering
\includegraphics[width=0.9\linewidth,height=0.45\linewidth]{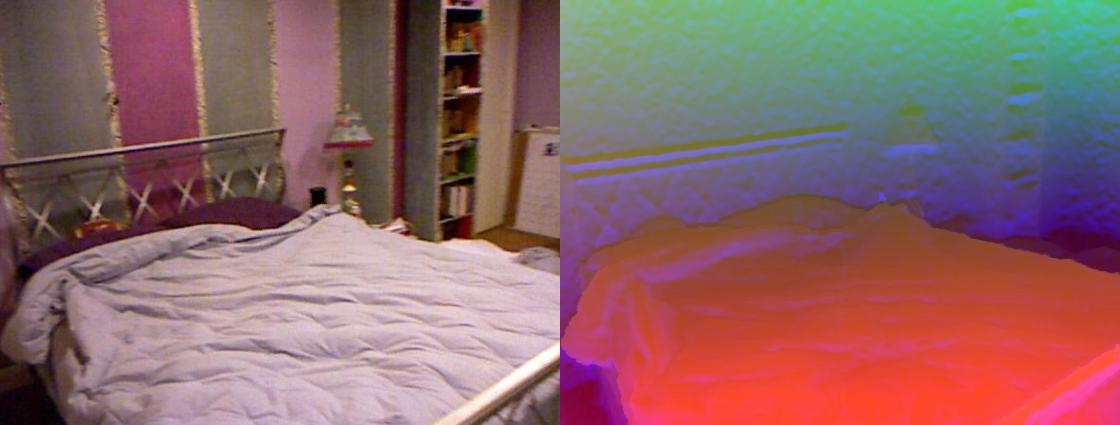}
\centering
\end{minipage}}
\subfigure[\normalsize{Confident sample 9 (Bedroom)}]{
\centering
\begin{minipage}[t]{0.32\linewidth}
\centering
\includegraphics[width=0.9\linewidth,height=0.45\linewidth]{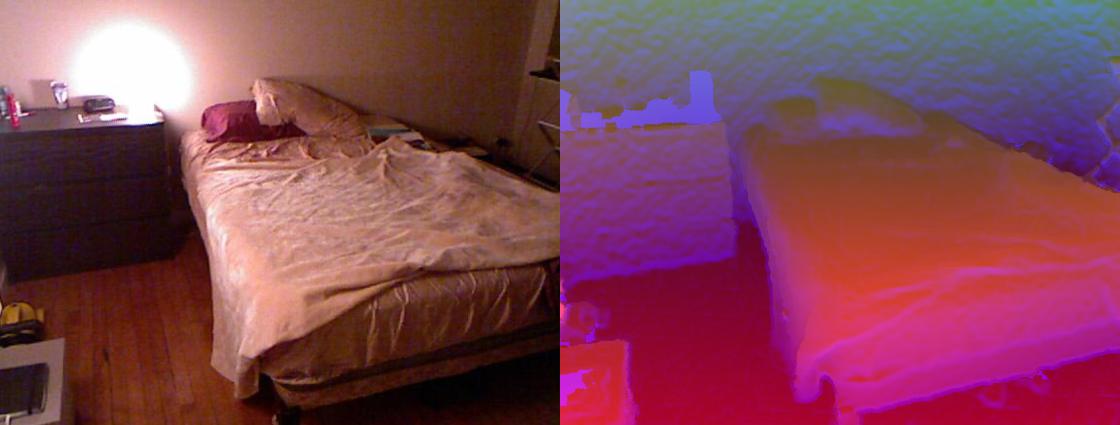}
\centering
\end{minipage}}
\subfigure[\normalsize{Confident sample 10 (Bedroom)}]{
\centering
\begin{minipage}[t]{0.32\linewidth}
\centering
\includegraphics[width=0.9\linewidth,height=0.45\linewidth]{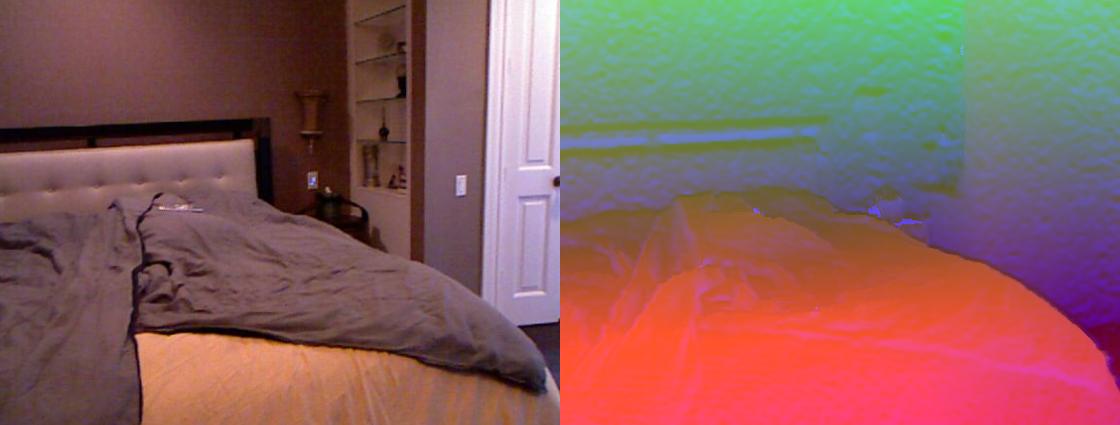}
\centering
\end{minipage}}
\subfigure[\normalsize{Confident sample 11 (Bedroom)}]{
\centering
\begin{minipage}[t]{0.32\linewidth}
\centering
\includegraphics[width=0.9\linewidth,height=0.45\linewidth]{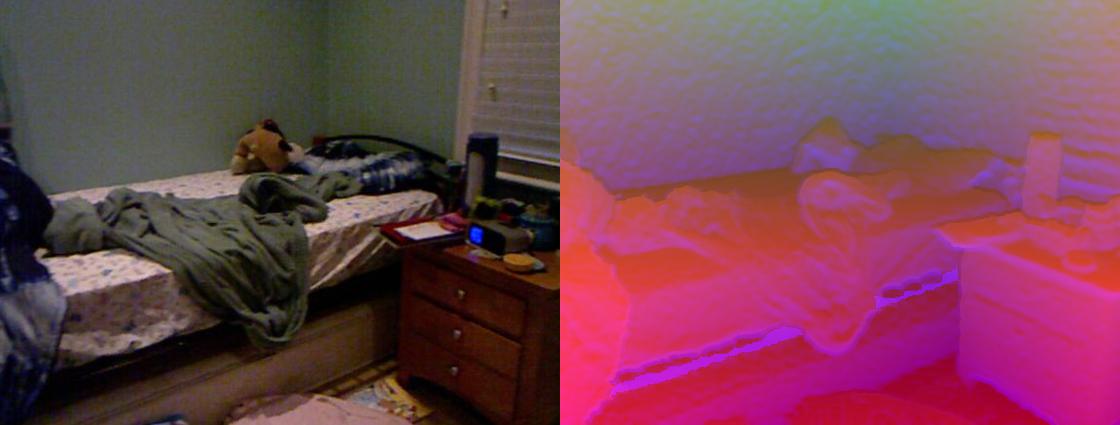}
\centering
\end{minipage}}
\subfigure[\normalsize{Confident sample 12 (Bedroom)}]{
\centering
\begin{minipage}[t]{0.32\linewidth}
\centering
\includegraphics[width=0.9\linewidth,height=0.45\linewidth]{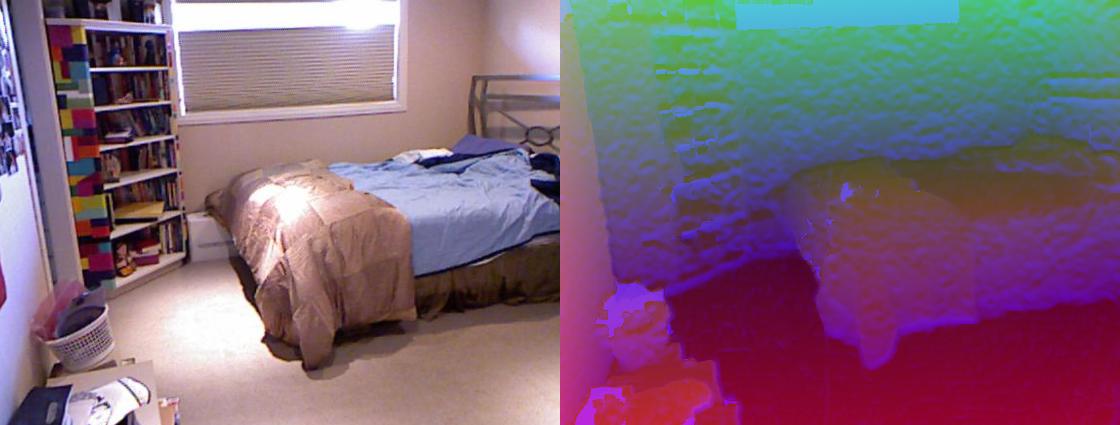}
\centering
\end{minipage}}
\subfigure[\normalsize{Confident sample 13 (Bathroom)}]{
\centering
\begin{minipage}[t]{0.32\linewidth}
\centering
\includegraphics[width=0.9\linewidth,height=0.45\linewidth]{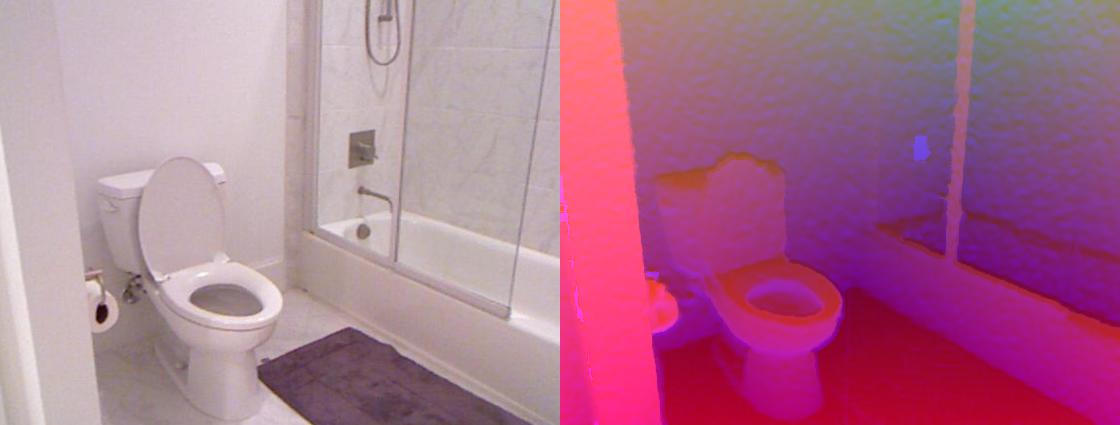}
\centering
\end{minipage}}
\subfigure[\normalsize{Confident sample 14 (Bedroom)}]{
\centering
\begin{minipage}[t]{0.32\linewidth}
\centering
\includegraphics[width=0.9\linewidth,height=0.45\linewidth]{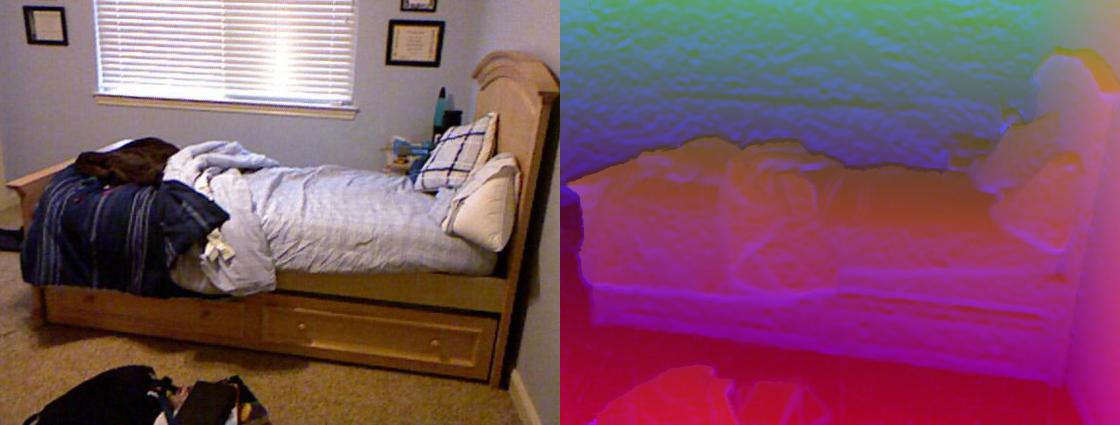}
\centering
\end{minipage}}
\subfigure[\normalsize{Confident sample 15 (Bedroom)}]{
\centering
\begin{minipage}[t]{0.32\linewidth}
\centering
\includegraphics[width=0.9\linewidth,height=0.45\linewidth]{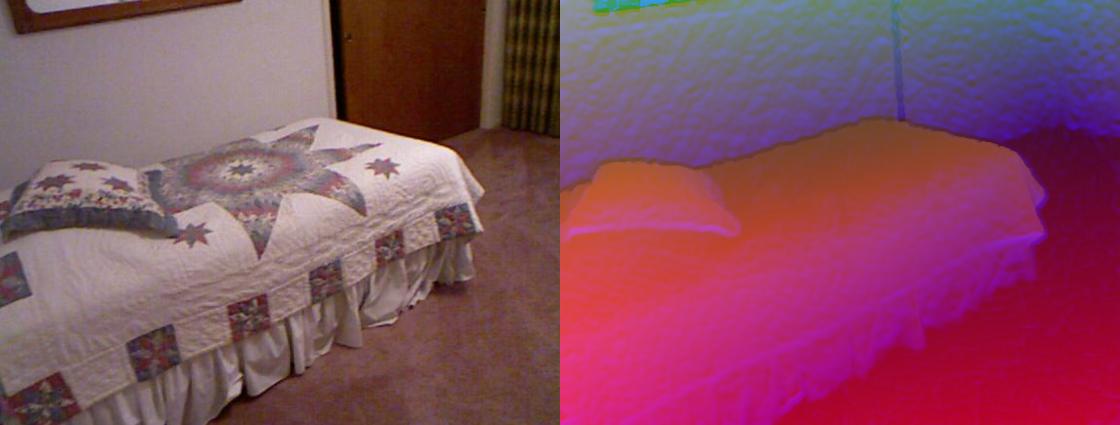}
\centering
\end{minipage}}
\subfigure[\normalsize{Confident sample 16 (Bathroom)}]{
\centering
\begin{minipage}[t]{0.32\linewidth}
\centering
\includegraphics[width=0.9\linewidth,height=0.45\linewidth]{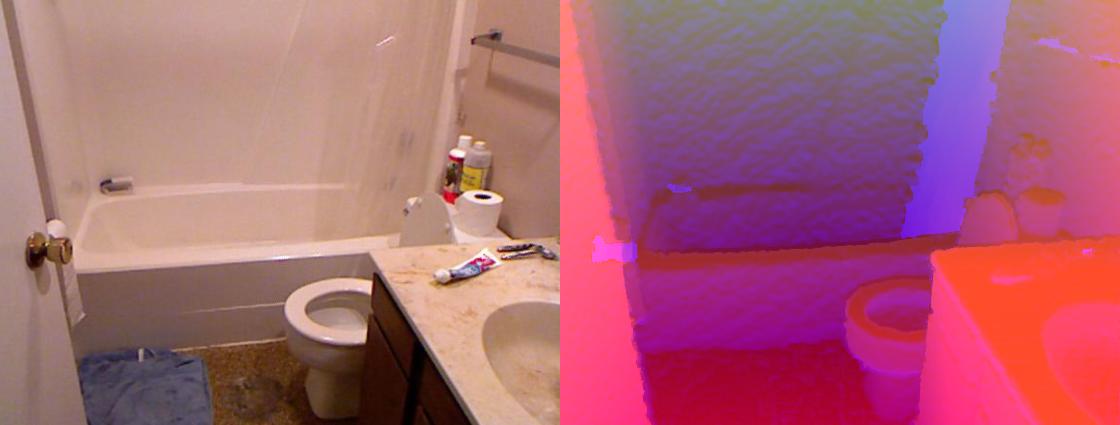}
\centering
\end{minipage}}
\subfigure[\normalsize{Confident sample 17 (Bathroom)}]{
\centering
\begin{minipage}[t]{0.32\linewidth}
\centering
\includegraphics[width=0.9\linewidth,height=0.45\linewidth]{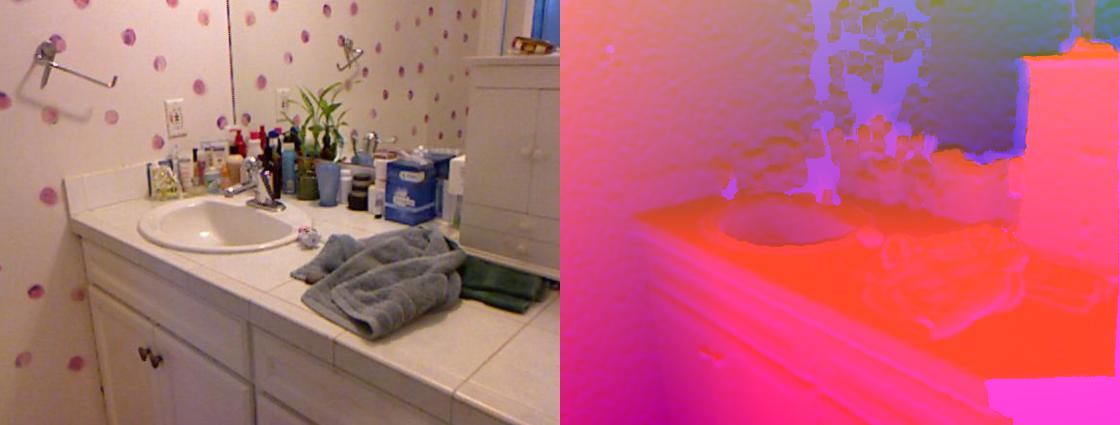}
\centering
\end{minipage}}
\subfigure[\normalsize{Confident sample 18 (Kitchen)}]{
\centering
\begin{minipage}[t]{0.32\linewidth}
\centering
\includegraphics[width=0.9\linewidth,height=0.45\linewidth]{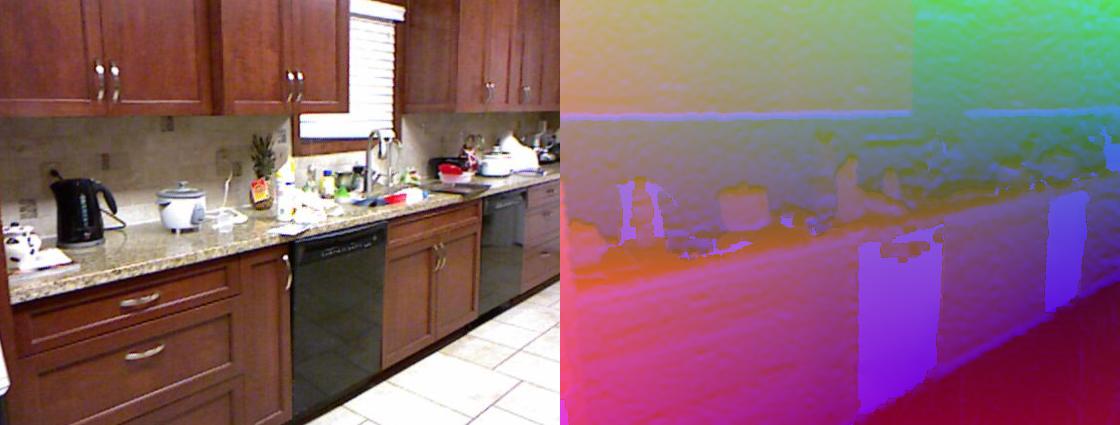}
\centering
\end{minipage}}
\caption{Examples with the lowest (top 18) prediction uncertainty using the ETMC algorithm on NYUD Depth V2 test data. The ground-truth labels are in brackets. All the above examples are correctly classified.}
\label{fig:ETMC_NYUD_vis_more_confident}
\end{figure*}

\begin{figure*}[!t]
\centering
\subfigure[\textcolor{black}{\normalsize{Uncertain sample 1 (Bedroom)}}]{
\centering
\begin{minipage}[t]{0.32\linewidth}
\centering
\includegraphics[width=0.9\linewidth,height=0.45\linewidth]{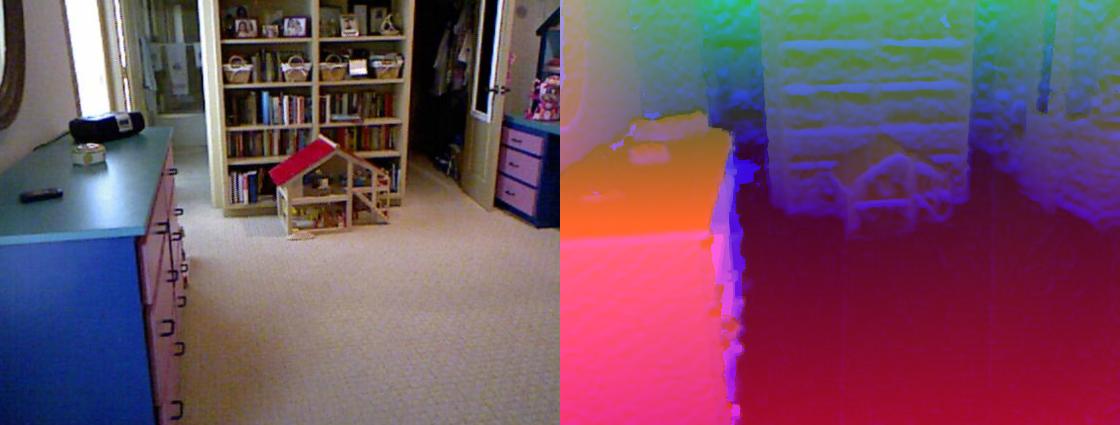}
\centering
\end{minipage}}
\subfigure[\textcolor{red}{\normalsize{Uncertain sample 2 (Others)}}]{
\centering
\begin{minipage}[t]{0.32\linewidth}
\centering
\includegraphics[width=0.9\linewidth,height=0.45\linewidth]{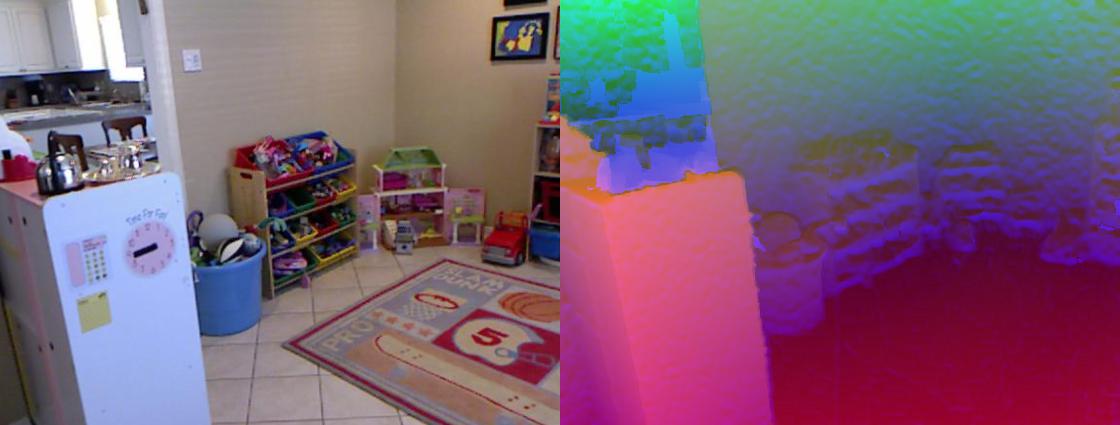}
\centering
\end{minipage}}
\subfigure[\textcolor{red}{\normalsize{Uncertain sample 3 (Kitchen)}}]{
\centering
\begin{minipage}[t]{0.32\linewidth}
\centering
\includegraphics[width=0.9\linewidth,height=0.45\linewidth]{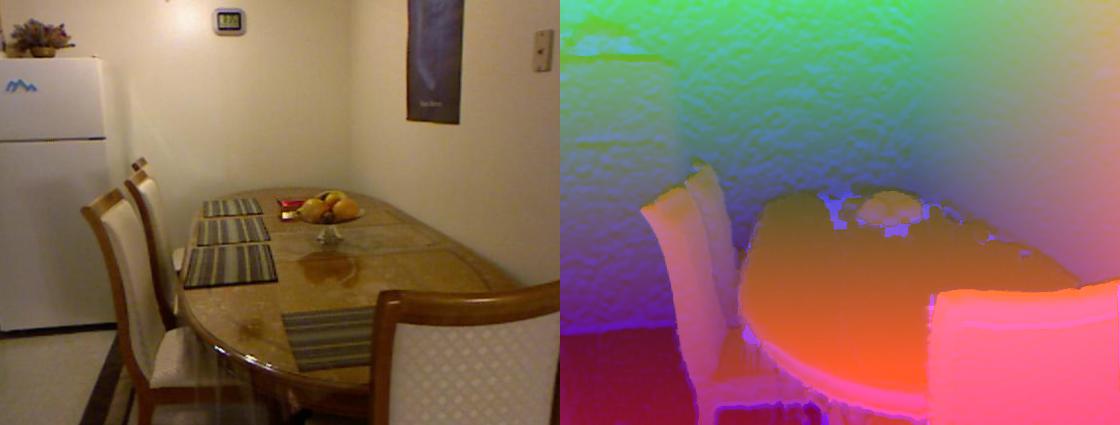}
\centering
\end{minipage}}
\subfigure[\textcolor{red}{\normalsize{Uncertain sample 4 (Kitchen)}}]{
\centering
\begin{minipage}[t]{0.32\linewidth}
\centering
\includegraphics[width=0.9\linewidth,height=0.45\linewidth]{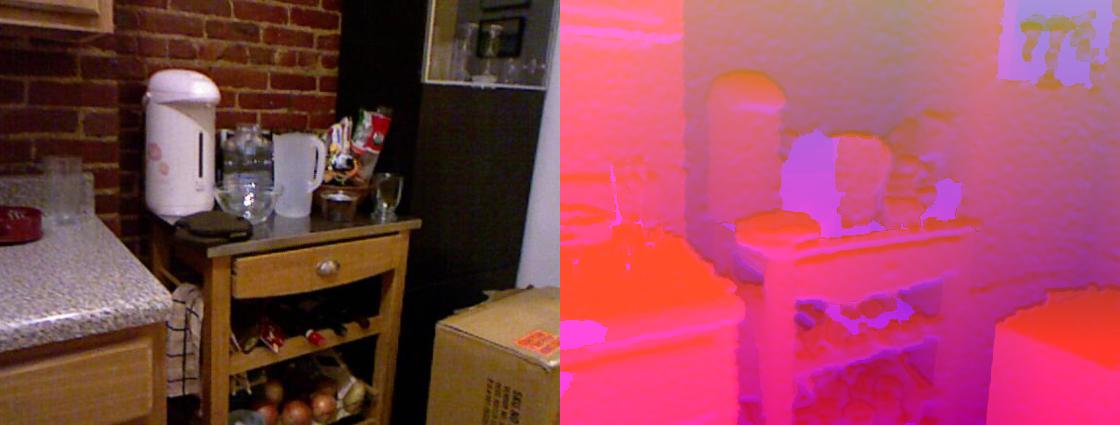}
\centering
\end{minipage}}
\subfigure[\textcolor{red}{\normalsize{Uncertain sample 5 (Bedroom)}}]{
\centering
\begin{minipage}[t]{0.32\linewidth}
\centering
\includegraphics[width=0.9\linewidth,height=0.45\linewidth]{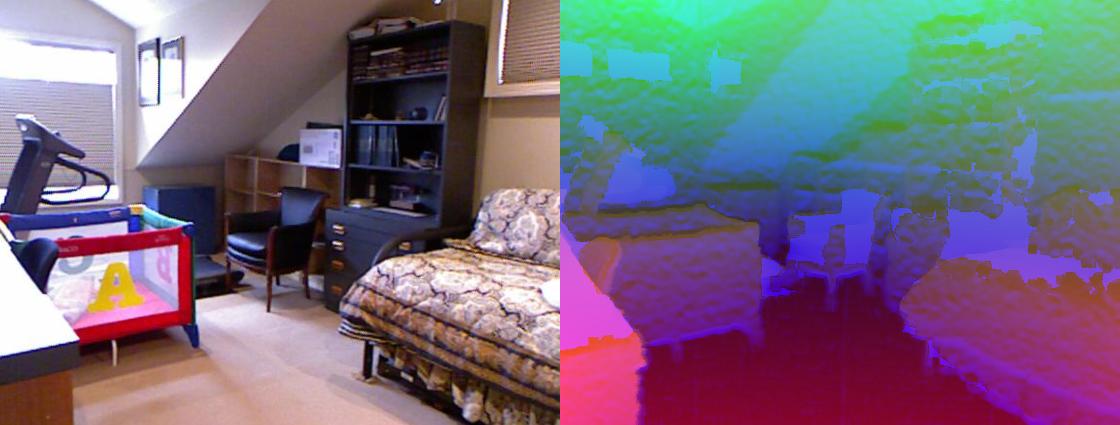}
\centering
\end{minipage}}
\subfigure[\textcolor{red}{\normalsize{Uncertain sample 6 (Bedroom)}}]{
\centering
\begin{minipage}[t]{0.32\linewidth}
\centering
\includegraphics[width=0.9\linewidth,height=0.45\linewidth]{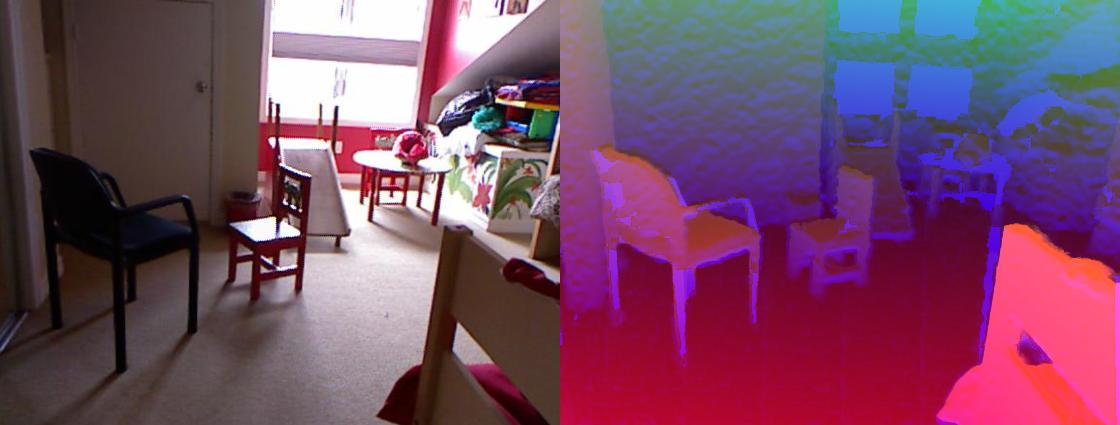}
\centering
\end{minipage}}
\subfigure[\textcolor{red}{\normalsize{Uncertain sample 7 (Dining room)}}]{
\centering
\begin{minipage}[t]{0.32\linewidth}
\centering
\includegraphics[width=0.9\linewidth,height=0.45\linewidth]{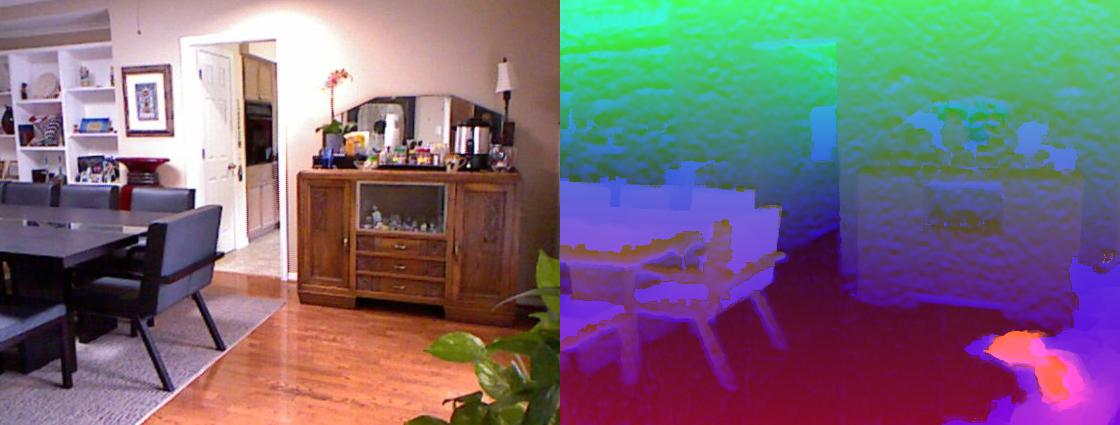}
\centering
\end{minipage}}
\subfigure[\textcolor{red}{\normalsize{Uncertain sample 8 (Home office)}}]{
\centering
\begin{minipage}[t]{0.32\linewidth}
\centering
\includegraphics[width=0.9\linewidth,height=0.45\linewidth]{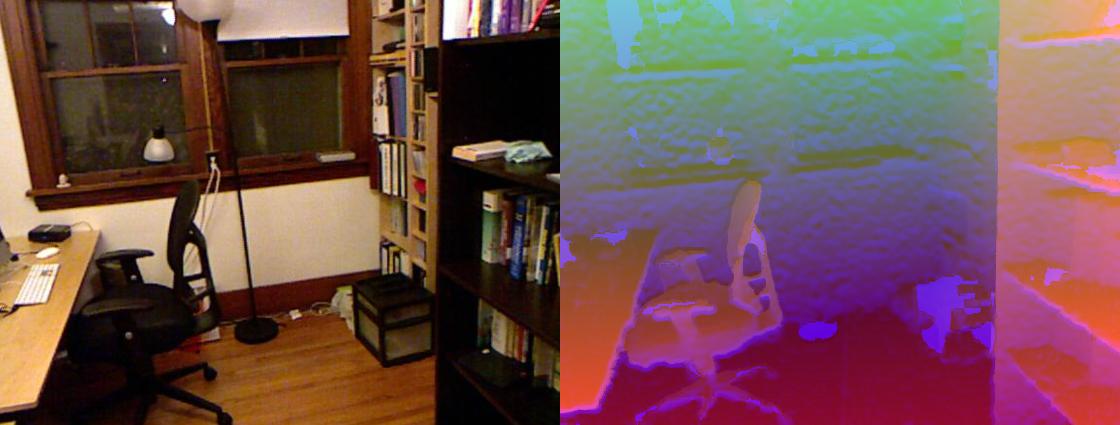}
\centering
\end{minipage}}
\subfigure[\textcolor{black}{\normalsize{Uncertain sample 9 (Living room)}}]{
\centering
\begin{minipage}[t]{0.32\linewidth}
\centering
\includegraphics[width=0.9\linewidth,height=0.45\linewidth]{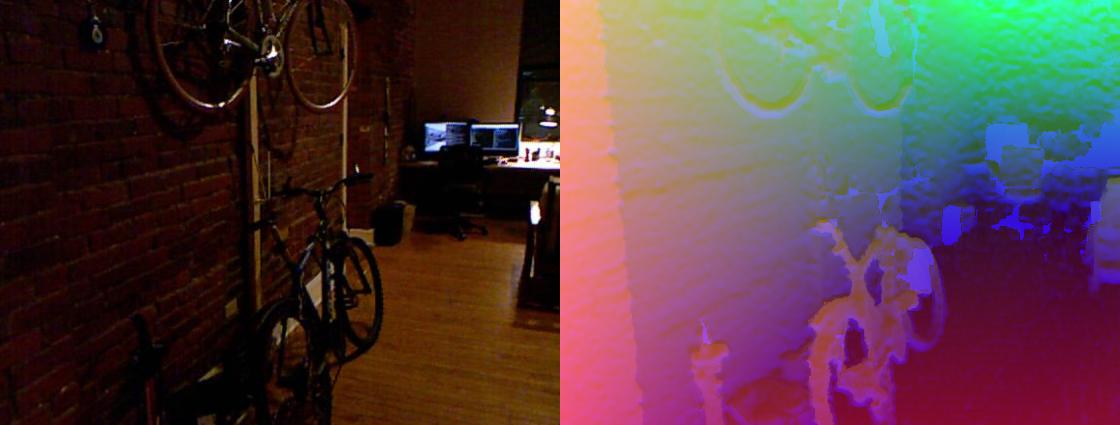}
\centering
\end{minipage}}
\subfigure[\textcolor{red}{\normalsize{Uncertain sample 10 (Bedroom)}}]{
\centering
\begin{minipage}[t]{0.32\linewidth}
\centering
\includegraphics[width=0.9\linewidth,height=0.45\linewidth]{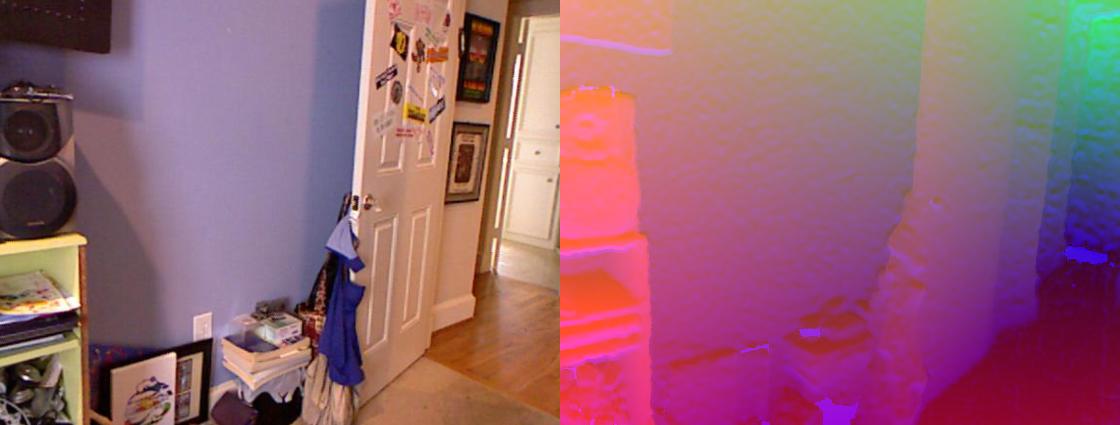}
\centering
\end{minipage}}
\subfigure[\textcolor{red}{\normalsize{Uncertain sample 11 (Living room)}}]{
\centering
\begin{minipage}[t]{0.32\linewidth}
\centering
\includegraphics[width=0.9\linewidth,height=0.45\linewidth]{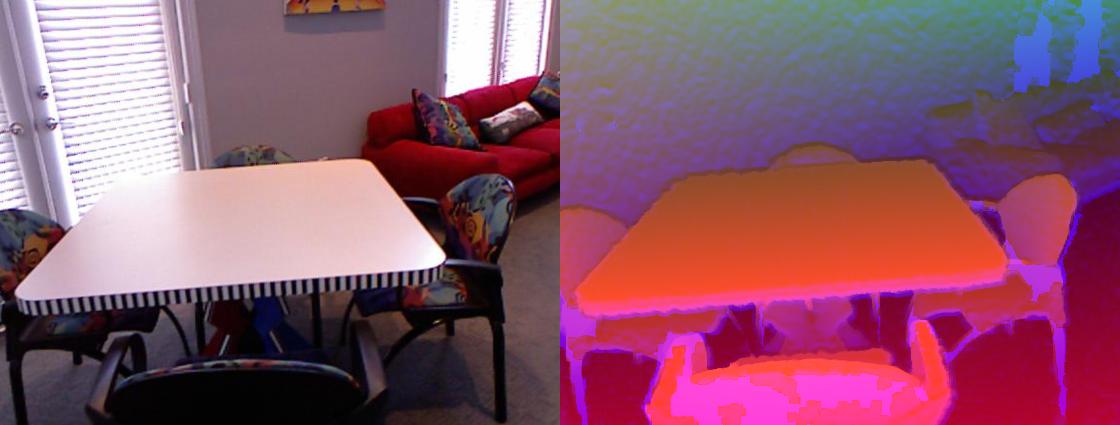}
\centering
\end{minipage}}
\subfigure[\textcolor{red}{\normalsize{Uncertain sample 12 (Living room)}}]{
\centering
\begin{minipage}[t]{0.32\linewidth}
\centering
\includegraphics[width=0.9\linewidth,height=0.45\linewidth]{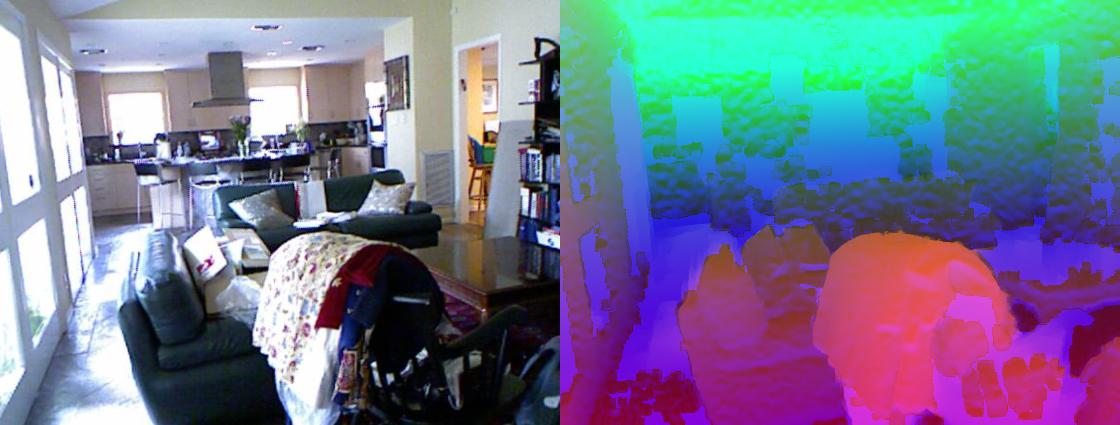}
\centering
\end{minipage}}
\subfigure[\textcolor{red}{\normalsize{Uncertain sample 13 (Home office)}}]{
\centering
\begin{minipage}[t]{0.32\linewidth}
\centering
\includegraphics[width=0.9\linewidth,height=0.45\linewidth]{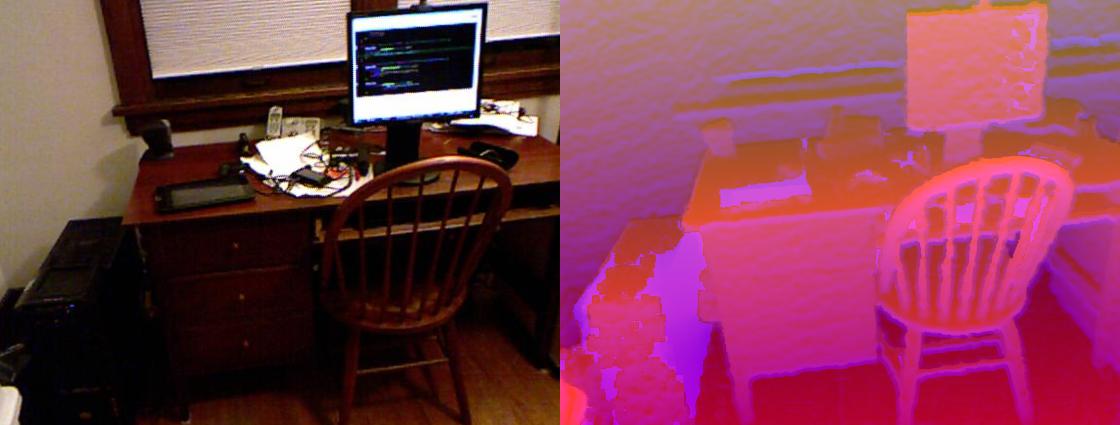}
\centering
\end{minipage}}
\subfigure[\textcolor{red}{\normalsize{Uncertain sample 14 (Kitchen)}}]{
\centering
\begin{minipage}[t]{0.32\linewidth}
\centering
\includegraphics[width=0.9\linewidth,height=0.45\linewidth]{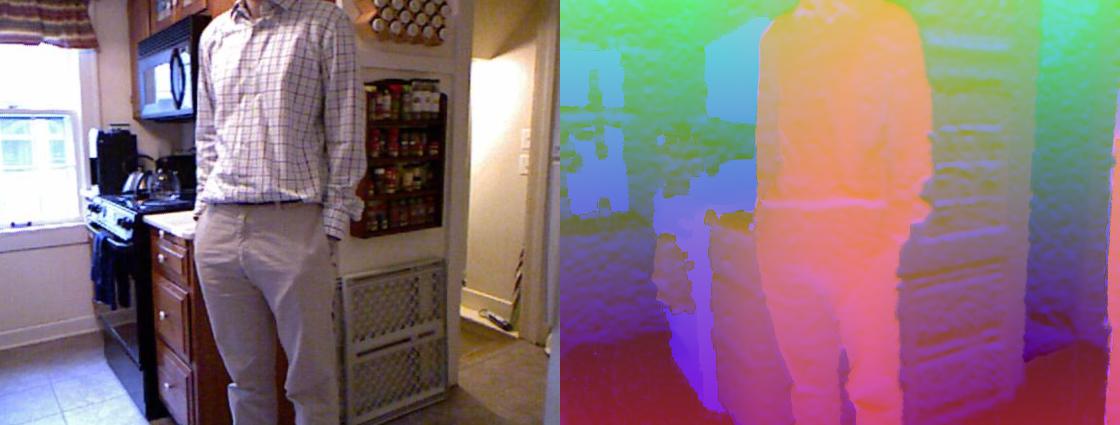}
\centering
\end{minipage}}
\subfigure[\normalsize{Uncertain sample 15 (Bedroom)}]{
\centering
\begin{minipage}[t]{0.32\linewidth}
\centering
\includegraphics[width=0.9\linewidth,height=0.45\linewidth]{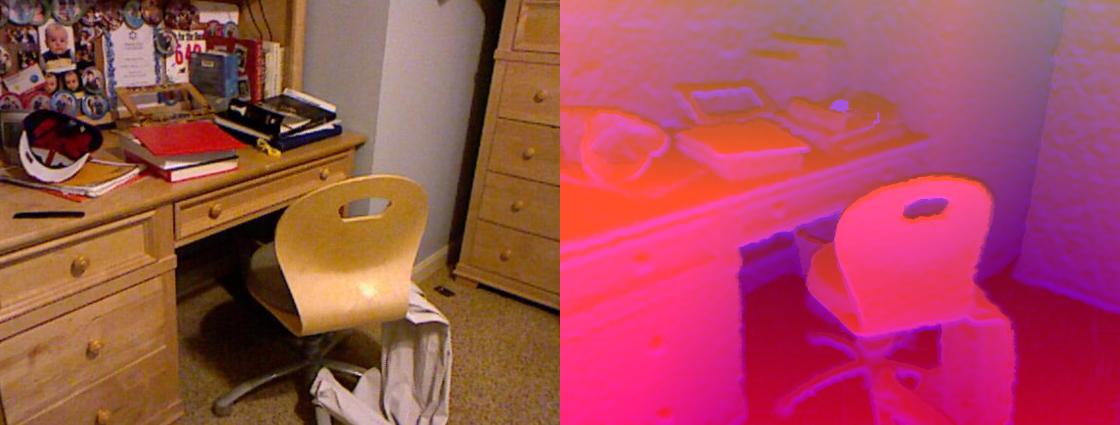}
\centering
\end{minipage}}
\subfigure[\textcolor{red}{\normalsize{Uncertain sample 16 (Others)}}]{
\centering
\begin{minipage}[t]{0.32\linewidth}
\centering
\includegraphics[width=0.9\linewidth,height=0.45\linewidth]{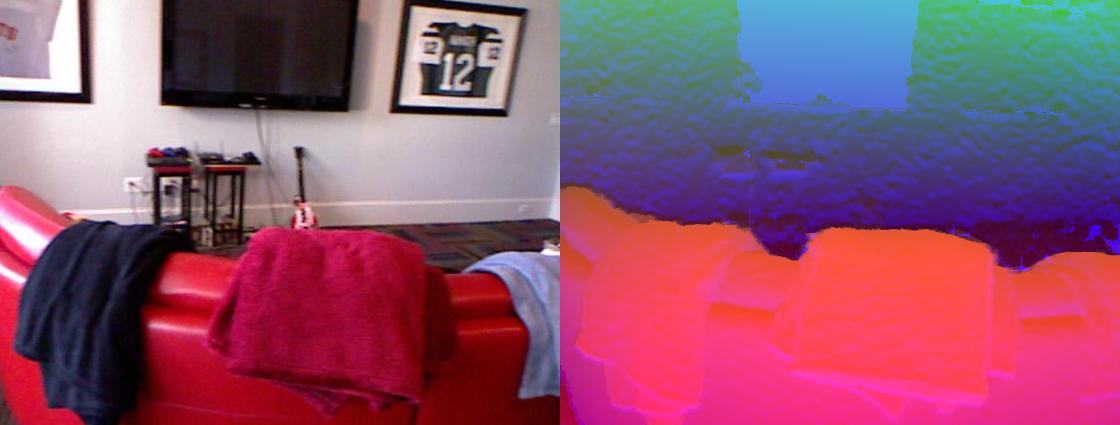}
\centering
\end{minipage}}
\subfigure[\textcolor{red}{\normalsize{Uncertain sample 17 (Kitchen)}}]{
\centering
\begin{minipage}[t]{0.32\linewidth}
\centering
\includegraphics[width=0.9\linewidth,height=0.45\linewidth]{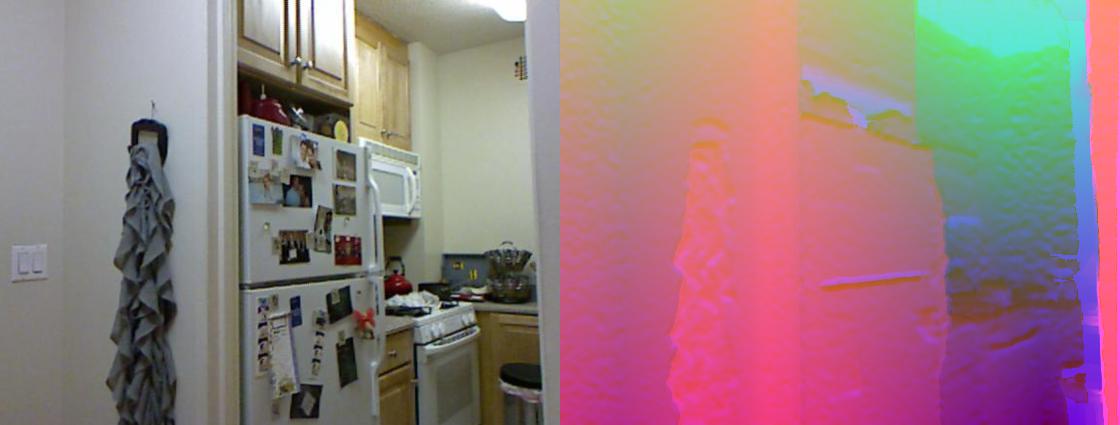}
\centering
\end{minipage}}
\subfigure[\textcolor{red}{\normalsize{Uncertain sample 18 (Living room)}}]{
\centering
\begin{minipage}[t]{0.32\linewidth}
\centering
\includegraphics[width=0.9\linewidth,height=0.45\linewidth]{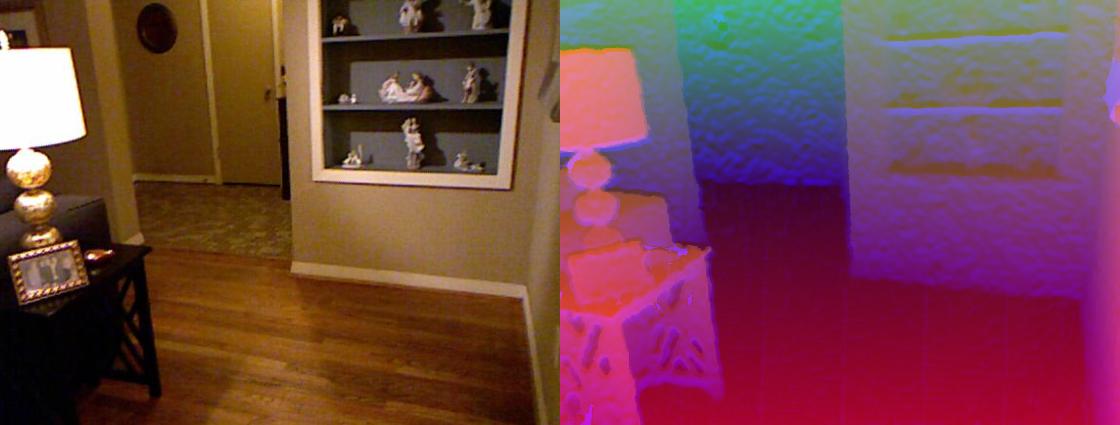}
\centering
\end{minipage}}
\caption{Examples with the highest (top 18) prediction uncertainty using the ETMC algorithm on NYUD Depth V2 test data. The ground-truth labels are in brackets. The text in red indicates that these samples are misclassified.}
\label{fig:ETMC_NYUD_vis_more_uncertain}
\end{figure*}
\begin{table*}[!t]
    \centering
    \small
    \renewcommand\arraystretch{1.1}
    \begin{tabular}{p{0.5cm}p{7cm}p{3cm}p{1.5cm}}
        \toprule
        \textbf{Index}&\textbf{Text} & \textbf{Image} & \textbf{Label}\\
        \midrule
        1&
        \raisebox{-.9\height}{\includegraphics[width=210pt]{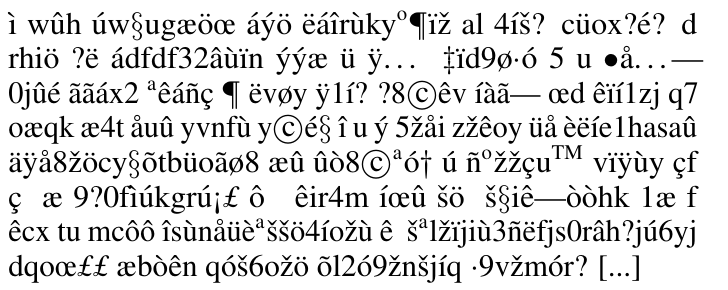}}
        & \raisebox{-.9\height}{\includegraphics[width=80pt]{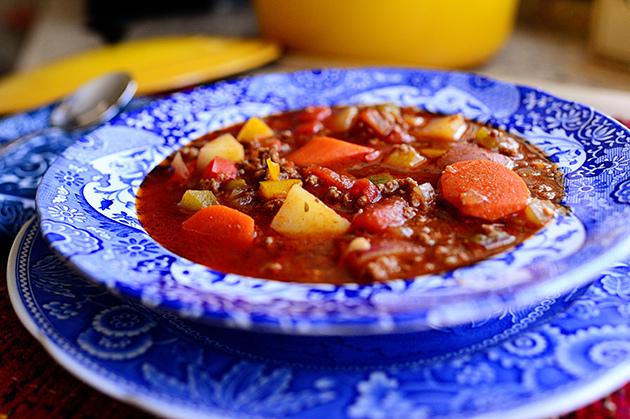}} & Hamburger\\\midrule
        2&log in sign up recipe box shopping list my profile my inbox account settings log out add a recipe looking for your public and private recipes or recipes with notes? they re now in your new recipe box learn more introducing the new food com app getting dinner on the table just got a whole lot easier 31 kitchen tips that ll change your life everything you need to conquer life in the kitchen don t miss secret ingredient burgers grilled chicken 15 ways and fun hot dog recipes looking for more? [...]& \raisebox{-.9\height}{\includegraphics[width=80pt]{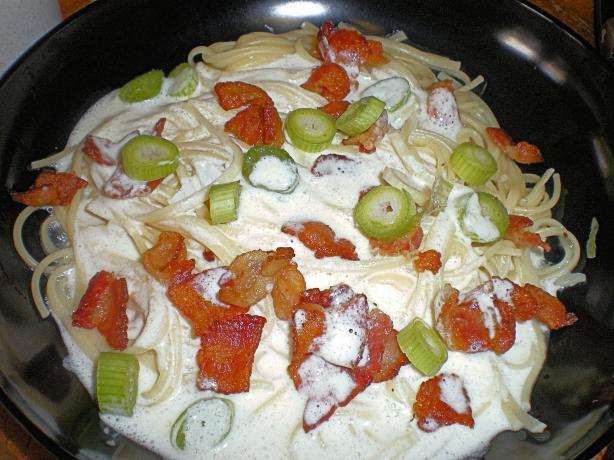}} & Spaghetti carbonara\\\midrule
        3&ack to channel options | sign out welcome sign in options | sign up account info make aol my homepage aol com aol com mail aol mail view all mail compose aol on favorites favorites search for videos playlists aol originals all originals so much more in short park bench hardwired cityballet candidly nicole the future starts here the restart project anthony eats america fatherhood acting disruptive the sartorialist flat out more shows [...] & \raisebox{-.9\height}{\includegraphics[width=80pt]{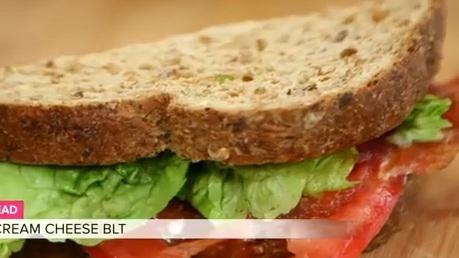}} & Grilled cheese sandwich\\\midrule
        4&shows featured shows barbecue addiction bobby s basics barefoot contessa diners drive ins and dives farmhouse rules giada at home restaurant impossible the kitchen the pioneer woman trisha s southern kitchen all shows a z videos watch food network star salvation full episodes watch food network star watch beat bobby flay watch full episodes chopped diners drive ins and dives the pioneer woman more full episodes more shows tv schedule chefs ree drummond summer party picks cooking for cowboys guy fieri super satisfying burgers dedicated fans of ddd alton brown giada de laurentiis bobby flay ina garten damaris phillips robert irvine more [...] &
        \raisebox{-.9\height}{\includegraphics[width=80pt]{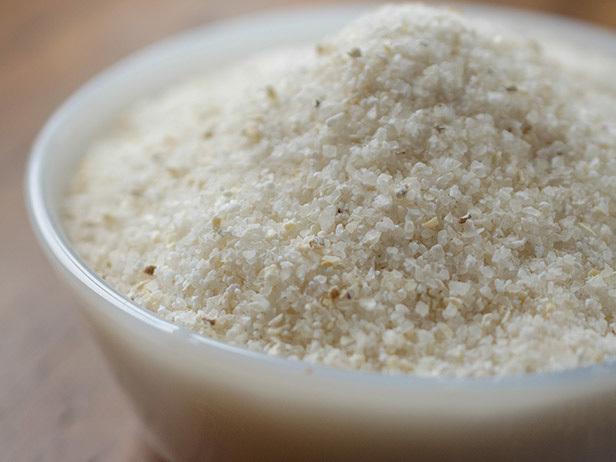}} & Shrimp and grits\\\midrule
        5&the heritage cook ® sharing traditions one recipe at a time by email home about contact disclaimers the heritage cook working together gluten free gf baking tips hints gf flour recipes gf resources websites how to build a gf flour blend from scratch heritage store recipe lists appetizers snacks … baked goods breads … food gifts breakfast brunch … beverages chocolate mondays … desserts main courses sauces marinades dry rubs … sides vegetables soups … salads dressings the bookcase baking shelf cuisine diet ingredient shelf general cooking and reference shelf the pantry chocolate faqs [...] &
        \raisebox{-.9\height}{\includegraphics[width=80pt]{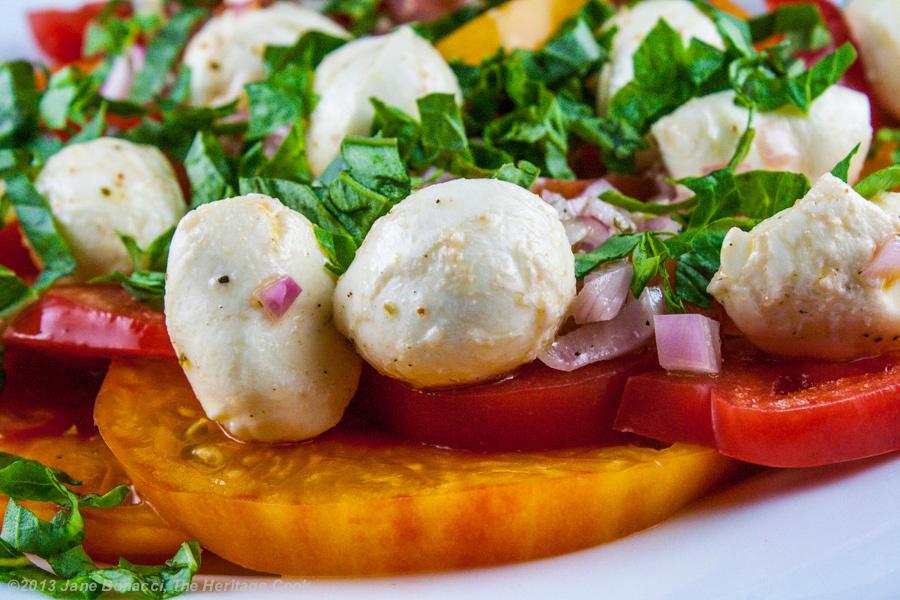}} & Caprese salad
     \\\bottomrule
    \end{tabular}
    \caption{Examples with the highest (top 5) prediction uncertainty on Food101. {These 5 examples are misclassified.}}
    \label{tab:examples1}
\end{table*}
\begin{table*}[!t]
    \centering
    \small
    \renewcommand\arraystretch{1.1}
    \begin{tabular}{p{0.5cm}p{7cm}p{3cm}p{1.5cm}}
        \toprule
        \textbf{Index}&\textbf{Text} & \textbf{Image} & \textbf{Label}\\
        \midrule
        1&gwen s kitchen creations delicious baked goods created in the kitchen for oscar \textbf{macaron} tips and tricks and a recipe by gwen on | i ve made french \textbf{macarons} many many times my first few attempts ended in misery broken shells cracks hollows no feet everything that could go wrong in a bad \textbf{macaron} did however i never stopped trying who knows how much pounds of almond flour meal and powdered sugar i sieved but it was not in vain i learned so much from all these [...] & \raisebox{-.9\height}{\includegraphics[width=80pt]{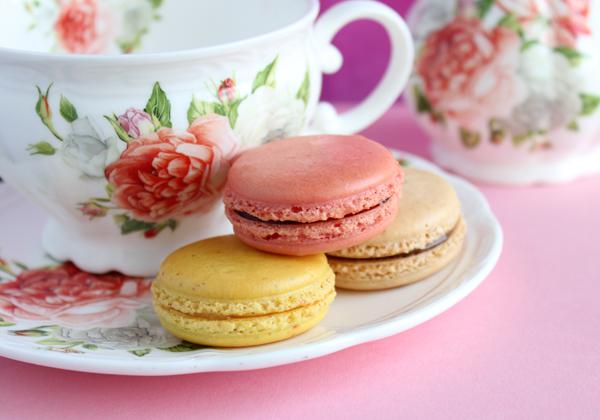}} & Macarons\\\midrule
        2&the pleasure monger telling it as it is menu skip to content home about me advertising collaborative opportunities using my content contact me tag archives \textbf{macaron} recipe sunflower seed \textbf{macarons} with black truffle salted white chocolate ganache 37 replies when the very talented and prolific shulie writer of food wanderings approached me on twitter to do a guest post for her tree nut free \textbf{macaron} series [...]& \raisebox{-.9\height}{\includegraphics[width=80pt]{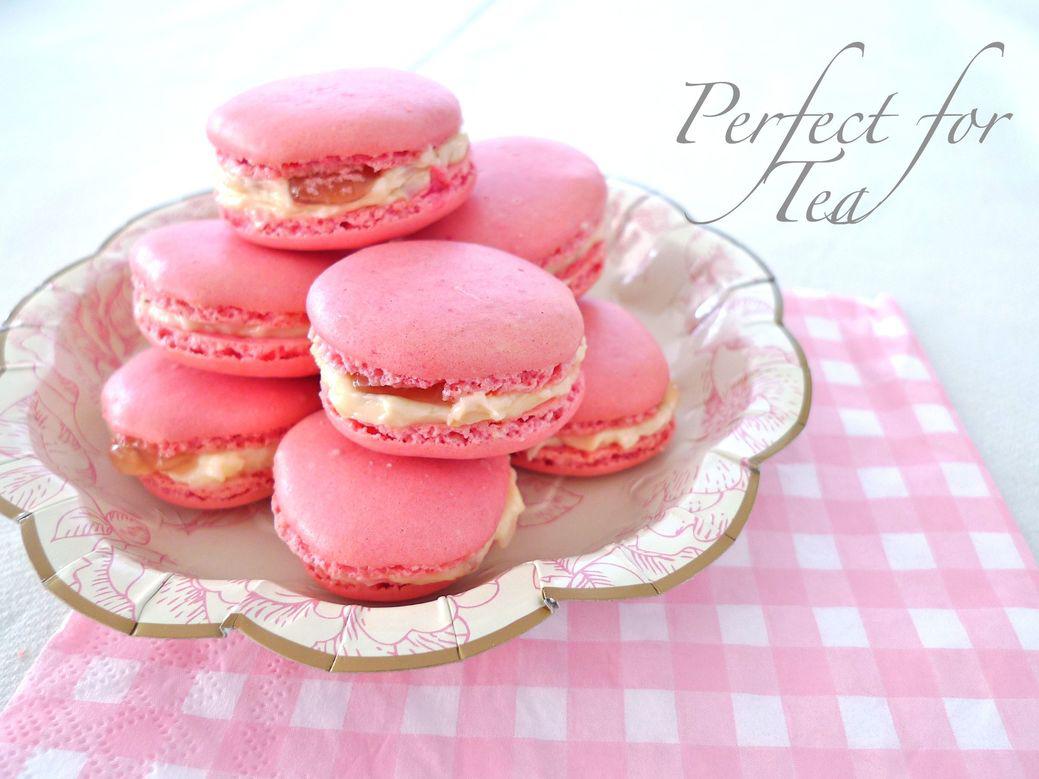}} & Macarons\\\midrule
        3&it s free and you can unsubscribe at any time submit your email address below and we ll send you a confirmation message right away approve with one click and you re done for more information click here get it delivered selected posts best new pastry chef why weight total eclipse of the tart homemade sprinkles about bravetart monday october 24 2011 \textbf{macarons} are for eating tweet when i've posted about \textbf{macarons} before in \textbf{macaron} mythbusters [...] & \raisebox{-.9\height}{\includegraphics[width=80pt]{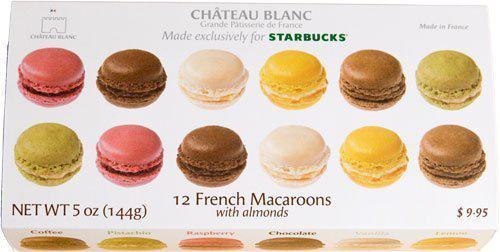}} & Macarons \\\midrule
        4&tastespotting features the delicious life facebook twitter tastespotting want to submit something delicious? login or register first browse by date | category | popularity | random tag \textbf{macarons} daydreamerdesserts blogspot com helping end childhood hunger one \textbf{macaron} at a time s mores banana cream pie apple pie oatmeal raisin piña colada \textbf{macarons} 83305 by tavqueenb save as favorite bakingupforlosttime blogspot com \textbf{macaron} party my first attempt at french \textbf{macarons} [...] &
        \raisebox{-.9\height}{\includegraphics[width=80pt]{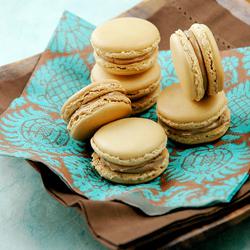}} & Macarons\\\midrule
        5&blue ribbons recipe contests giveaways meet the kitchen crew cookbooks custom cookbooks member cookbooks members choice cookbooks my cookbooks coupons shop knickknacks gift memberships members choice cookbooks my kitchen my profile recipe box cookbooks menu calendar grocery list conversations notifications hide ad trending recipes rice krispy treat \textbf{macarons} rice krispy treat \textbf{macarons} 1 photo pinched 3 times grocery list add this recipe to your grocery list print print this recipe and money saving coupons photo [...] &
        \raisebox{-.9\height}{\includegraphics[width=80pt]{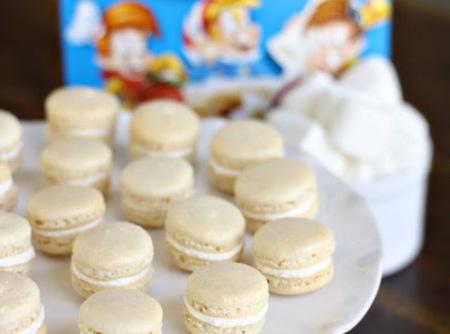}} & Macarons
     \\\bottomrule
    \end{tabular}
    \caption{Examples with the lowest (top 5) prediction uncertainty on Food101. {The above 5 samples are correctly classified.}}
    \label{tab:examples2}
\end{table*}
\end{document}